\pdfoutput=1

\documentclass[10pt,twocolumn,letterpaper]{article}

\usepackage{cvpr}              %

\usepackage[accsupp]{axessibility}

\usepackage{multirow}
\usepackage{xcolor}
\usepackage{colortbl}
\usepackage{makecell}

\usepackage{amsmath}
\usepackage{amsthm} 
\usepackage{thmtools}
\newtheorem{definition}{Definition}

\newtheorem{lemma}{Lemma}
\newtheorem{assumption}{Assumption}
\newtheorem{proposition}{Proposition}

\definecolor{rowblue}{RGB}{232,240,249}
\definecolor{darkblue}{RGB}{38,85,139}
\definecolor{darkred}{RGB}{180,0,0}

\usepackage{svg}

\usepackage{booktabs}      %
\definecolor{LightRedHighlight}{RGB}{244, 231, 232}
\definecolor{DarkRedHighlight}{RGB}{223, 182, 188}

\usepackage{xcolor}
\usepackage{tcolorbox}
\newtcolorbox{summarybox}{
  colback=black!5!white,  %
  colframe=black!60!white, %
  boxrule=0.5pt,           %
  arc=2mm,                 %
  fonttitle=\bfseries,
  left=5mm,
  right=5mm,
  top=3mm,
  bottom=3mm
}

\definecolor{cvprblue}{rgb}{0.21,0.49,0.74}
\usepackage[pagebackref,breaklinks,colorlinks,allcolors=cvprblue]{hyperref}

\def\paperID{4269} %
\def\confName{CVPR}
\def\confYear{2026}

\title{Understanding and Enforcing Weight Disentanglement in Task Arithmetic}

\author{%
  Shangge Liu $^{1}$,
  Yuehan Yin $^{1}$,
  Lei Wang $^{2}$, 
  Qi Fan $^{1}$, 
  Yinghuan Shi $^{1}$, \\
  {Wenbin Li} $^{1}$\thanks{Corresponding Author} ,
  {Yang Gao} $^{1}$, 
  Dacheng Tao $^{3}$
    \\
  $^1$State Key Laboratory for Novel Software Technology, Nanjing University, China \quad \\
$^2$University of Wollongong, Australia \quad 
    $^3$Nanyang Technological University, Singapore \quad 
}

\begin{document}
\maketitle
\begin{abstract}
Task arithmetic provides an efficient, training-free way to edit pre-trained models, yet lacks a fundamental theoretical explanation for its success. The existing concept of ``weight disentanglement" describes the ideal outcome of non-interfering task composition but does not reveal its underlying cause. Crucially, what intrinsic properties of the pre-trained model ($\theta_0$) or the task vectors ($\tau_t$) enable this disentanglement remains underexplored. In this paper, we introduce Task-Feature Specialization (TFS), a model's ability to allocate distinct internal features to different tasks, as the fundamental principle. We first prove that TFS is a sufficient condition for weight disentanglement. More importantly, we find that TFS also gives rise to an observable geometric consequence: weight vector orthogonality. This positions TFS as the common cause for both the desired functional outcome (disentanglement) and a measurable geometric property (orthogonality). This relationship provides the key insight for our method: since the abstract TFS property is intractable to enforce directly, we can instead promote weight disentanglement by shaping its concrete geometric consequence, orthogonality. Therefore, we propose OrthoReg, a simple and effective regularization method that actively enforces an internal orthogonal structure on weight updates ($\Delta W$) that constitute $\tau_t$ during fine-tuning. And we theoretically prove that OrthoReg promotes disentanglement. Extensive experiments demonstrate that OrthoReg consistently and significantly enhances the performance of various task arithmetic methods. Code is available at \href{https://github.com/RL-MIND/OrthoReg}{https://github.com/RL-MIND/OrthoReg}.
\end{abstract}
\section{Introduction}\label{sec:intro}
Large-scale pre-trained models (PTMs)~\cite{clip2021,vit2021,llama2023} have become powerful foundations for various applications~\cite{Financial2024,Medical2025}. However, a critical challenge lies in adapting these powerful yet static models to new requirements, such as acquiring new skills~\cite{Mosbach2021}, personalizing behavior~\cite{Zhang2022_2}, or unlearning harmful capabilities~\cite{Yao2023,Hong2024}. Conventional methods like joint fine-tuning are often impractical due to prohibitive computational costs, the inaccessibility of all training datasets, and the risk of catastrophic forgetting~\cite{Li2024,Luo2023}. 

\begin{figure}[t]
  \centering
   \includegraphics[width=\linewidth]{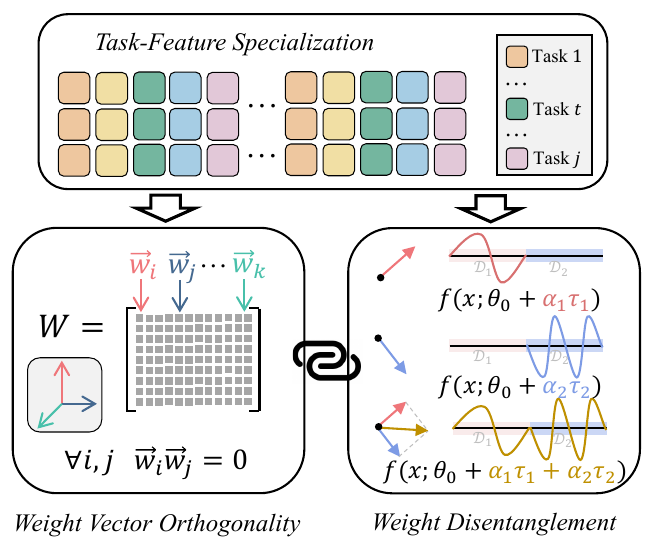}
   \caption{ 
   Conceptual illustration of our central thesis: Task-Feature Specialization (TFS) is proposed and shown as the common cause that connects the geometric property of Weight Vector Orthogonality (WVO) with the functional property of Weight Disentanglement (WD). This paper establishes this connection in two ways: first, by proving that TFS, which gives rise to inherent orthogonality in the pre-trained model $\theta_0$, is a sufficient condition for ideal disentanglement; and second, by proposing a method that actively enforces this structure on weight updates ($\Delta W$) that constitute $\tau_t$ to promote disentanglement in realistic scenarios.}
   \label{fig:concept}
   \vspace{-16pt}
\end{figure}

To address this challenge, model merging~\cite{Yang2024,Stoica2025,Adamerging2024} has recently emerged as an efficient, training-free paradigm. Instead of costly retraining, model merging operates post-hoc combination of the weights of multiple specialized models, each fine-tuned for a specific task, to create a single, multi-talented model. Among these methods, task arithmetic~\cite{ta2023} is particularly elegant. It operates by representing the knowledge for a task $t$ as a task vector, defined by the parameter shift $\tau_t=\theta_t-\theta_0$ from the pre-trained weights $\theta_0$ to the fine-tuned weights $\theta_t$. By simply adding or subtracting these task vectors, one can compose, remove, or even draw analogies between different skills, all without the need for costly joint training.

Despite its empirical success, a fundamental question remains: \textbf{why does task arithmetic work?} Answering this is critical to transforming task arithmetic from an empirical curiosity into a reliable engineering tool and to improve it beyond its current limitations, especially in critical applications where predictability and trustworthiness are significant. The concept of \textbf{Weight Disentanglement}, introduced in \textit{Tangent Task Arithmetic} (TTA)~\cite{tta2023}, offers a partial answer. It posits that in an ideal scenario, the effects of different task vectors are isolated to their respective data domains, thus preventing destructive interference. However, weight disentanglement is more of a phenomenological description of the desired outcome than an explanation of its fundamental cause. The existing literature does not fully specify what properties of the pre-trained model ($\theta_0$) or the task vectors ($\tau_t$) are necessary to achieve this state.

This gap in understanding motivates our work. To move from description to explanation, we must address two questions. \textbf{\textit{i.} What makes a model $\theta_0$ good for task arithmetic? \textit{ii.} How to construct a good $\tau_t$?} The first question asks: what intrinsic properties of a pre-trained model $\theta_0$ make it inherently suitable for achieving weight disentanglement? Without answering this question, we will not be able to select or design foundation models that are naturally amenable to effective model editing, leaving performance to chance. The second question addresses the construction of the task vectors themselves: how can we construct task vectors $\tau_t$ that actively promote weight disentanglement? Without answering this, standard fine-tuning offers no guarantee of producing task vectors that compose well and often leads to suboptimal, interference-prone results.

In this paper, we identify Task-Feature Specialization (TFS), the model's ability to allocate distinct internal features to different tasks, as the key underlying principle. We first prove that TFS is a sufficient condition for Weight Disentanglement (WD) (\Cref{theorem:pretrain}). More importantly, we find that TFS also gives rise to a geometric consequence: Weight Vector Orthogonality (WVO). This positions TFS as the common cause for both the desired functional outcome (WD) and an observable geometric consequence (WVO),  a conceptual relationship shown in \Cref{fig:concept}. We can thus answer Question \textit{i}: models that achieve TFS are effective at disentanglement, and WVO provides a possible indicator for this abstract property.
However, TFS is an ideal property that rarely holds for a pre-trained model $\theta_0$ in practice. This challenge motivates us to resort to $\tau_t$, that is, answering Question \textit{ii}. To address this, we propose a method that enforces an internal orthogonal structure on weight updates ($\Delta W$) that constitute $\tau_t$ and theoretically show its efficacy (\Cref{theorem:ortho_updates}). We also establish a theoretical connection between our approach and TTA~\cite{tta2023}, showing both converge on the same principle: inter-task vector orthogonality.

Our main contributions can be summarized as follows.
\begin{itemize}
    \item We put forward a theory for the success of task arithmetic, identifying task-feature specialization as a sufficient condition for weight disentanglement and then weight vector orthogonality as its geometric consequence.
    \item Based on the theory, we propose OrthoReg, a regularization method that actively promotes disentanglement by enforcing orthogonality on weight updates, for which we provide a rigorous theoretical proof of efficacy.
    \item We establish a theoretical connection between our method and the existing work TTA and reveal that they both succeed by achieving inter-task vector orthogonality.
    \item We experimentally demonstrate that OrthoReg consistently and significantly improves the performance of various task arithmetic methods.
\end{itemize}
\section{Related Work}
\textbf{Model Merging and Task Arithmetic.}
Task Arithmetic~\cite{ta2023} is a model merging technique that combines models by algebraically manipulating their ``task vectors". While this approach avoids costly retraining~\cite{Yang2024}, it often suffers from destructive interference when composing multiple tasks. Existing solutions to this problem can be broadly classified into two categories~\cite{Yang2024}: during-merging and pre-merging methods. During-merging methods design sophisticated algorithms to combine already-trained models~\cite{Tiesmerging2023, Adamerging2024, tsv2025}. In contrast, pre-merging methods aim to create more ``mergeable" models from the outset by modifying the fine-tuning process~\cite{tta2023,linearatt2025,Zhang2025,Tang2024,taujp2025}. Our proposed method, OrthoReg, belongs to the pre-merging category.

Key theoretical work in this area includes Tangent Task Arithmetic (TTA)~\cite{tta2023}, which introduces the crucial concept of weight disentanglement and shows that fine-tuning in linearized tangent space promotes it. Work~\cite{linearatt2025} further demonstrates that fine-tuning only the attention modules also enhances it. 
Concurrent works have provided generalization analyses for nonlinear Transformers based on data-dependent task correlations~\cite{Li2025} and established theoretical bounds that explicitly require task vectors to be nearly orthogonal~\cite{Zeng2025}.
Our work provides a more fundamental explanation through Task-Feature Specialization (TFS) and proposes enforcing its geometric consequence, orthogonality, as a direct mechanism to mitigate interference.

\noindent \textbf{Orthogonality in Neural Networks.}
The geometric properties of weights, particularly orthogonality, have been well-studied for their role in improving training stability, generalization, and efficiency~\cite{Saxe2014, Arjovsky2016, MiyatoKKY2018, He2015, aoft2025}.
It has been successfully applied in RNNs to prevent vanishing/exploding gradients~\cite{Arjovsky2016} and in GANs via Spectral Normalization to stabilize discriminator training~\cite{MiyatoKKY2018}.
Our work repurposes this powerful geometric constraint for a novel application: task arithmetic. We demonstrate that actively shaping the geometry of weight updates to be orthogonal is a direct and effective way to mitigate interference in task arithmetic.
\section{Preliminaries and Problem Formulation}
In this section, we first define the basic setup of task arithmetic, then introduce the concept of weight disentanglement. Finally, we introduce Neural Tangent Kernel (NTK) linearization hypothesis.

\subsection{Basic Setup and Notation}

Let $f(x; \theta)$ be a neural network parameterized by weights $\theta \in \mathbb{R}^P$, with initial pre-trained weights denoted as $\theta_0$. After fine-tuning on a task $t$, the new weights are $\theta_t^*$. Following the literature~\cite{ta2023}, the task vector $\tau_t$ is defined as the parameter shift, 
\begin{equation}\small
\setlength\abovedisplayskip{5pt}
\setlength\belowdisplayskip{5pt}
    \tau_t = \theta_t^* - \theta_0.
\end{equation}
The task vector $\tau_t$ encapsulates the parameter modifications required for the model to adapt to task $t$. Task arithmetic then performs algebraic operations on these task vectors to create a multi-task model. Specially, combining a set of tasks $\mathcal{T} = \{1,\dots, T\}$ via task addition is achieved by
\begin{equation}\small
\setlength\abovedisplayskip{5pt}
\setlength\belowdisplayskip{5pt}
    \theta_{\text{MT}}=\theta_0 +\sum_{t=1}^T\alpha_t\tau_t,
\end{equation}
where $\alpha_t$ is scalar coefficient of each task. Our theoretical analysis focuses on the parameters of linear layers (\textit{e.g.}, FC layers and attention projections). This simplification is well-justified by their central role in modern architectures~\cite{Kaplan2020,Zhang2022} and model merging practices~\cite{ta2023,Modelsoup2022,Tiesmerging2023}, with a detailed rationale provided in Appendix A.

\subsection{The Weight Disentanglement Property}\label{subsec:wd}
The concept of weight disentanglement introduced by the seminal work~\cite{tta2023} is a key property proposed to explain the success of task arithmetic.

\begin{definition}[Weight Disentanglement]\label{defin:weight_disentanglement}
    A model $f$ satisfies weight disentanglement at $\theta_0$ with respect to a set of tasks $\mathcal{T}$ with data domains $\{\mathcal{D}_t\}_{t=1}^T$ if, for any set of scalar coefficients $\{\alpha_t\}_{t=1}^T$, the following relationship holds,
\begin{equation}\small
\setlength\abovedisplayskip{5pt}
\setlength\belowdisplayskip{5pt}
\begin{aligned}
    f(x; \theta_0 + \sum_{t=1}^T \alpha_t \tau_t)=\begin{cases} f(x; \theta_0 + \alpha_i \tau_i), & \text{if } x \in \mathcal{D}_i\\
    f(x; \theta_0). & \text{if } x \notin \bigcup_{t=1}^T \mathcal{D}_t \end{cases}
\end{aligned}
\label{formula:weight_disentanglement}
\end{equation}
\end{definition}
Intuitively, this property requires that a merged model's behavior on a specific task's data depends only on that task's vector, while reverting to pre-trained behavior for out-of-domain data. Disentanglement can stem from the inherent properties of the pre-trained model $\theta_0$ or from the specific construction of the task vectors $\tau_t$. Accordingly, our work investigates both the ideal properties of $\theta_0$ and a method for constructing $\tau_t$ to actively promote disentanglement.

\subsection{The NTK Linearization Hypothesis}\label{subsec:ntk}

Consistent with the literature~\cite{tta2023}, our analysis relies on the Neural Tangent Kernel (NTK)~\cite{ntk2018} linearization hypothesis, which approximates the model's output for a small parameter change $\tau$ with a first-order Taylor expansion around $\theta_0$,
\begin{equation}\small
\setlength\abovedisplayskip{5pt}
\setlength\belowdisplayskip{5pt}
f(x; \theta_0 + \tau) \approx f(x; \theta_0) + \tau^{\top} \nabla_\theta f(x; \theta_0).
\end{equation}
Here, $J(x) := \nabla_\theta f(x; \theta_0)$ is the Jacobian of the model's output with respect to its parameters.

\section{The Proposed Framework}
In this section, we theorize that Task-Feature Specialization (TFS) is the key principle for task arithmetic. In the first part, we demonstrate that TFS is a sufficient condition for weight disentanglement (WD) and also naturally leads to weight vector orthogonality (WVO). This suggests that WD and WVO are correlated effects of the common cause, TFS. While TFS is rare in practice, its geometric consequence, orthogonality, offers a tangible objective. This motivates our method in second part: actively enforcing orthogonality on weight updates ($\Delta W$) to mitigate interference and improve disentanglement.

\subsection{An Equivalent Condition for Disentanglement}\label{subsec:EquivalentCondition_for_Disentanglement}
To clearly reveal the underlying mechanism of weight disentanglement, we first focus our analysis on the interaction between two tasks, $t$ and $j$. The core in-domain component of the weight disentanglement property (see \Cref{defin:weight_disentanglement}) can be simplified to,
\begin{equation}\small
\setlength\abovedisplayskip{5pt}
\setlength\belowdisplayskip{5pt}
    f(x;\theta_0+\tau_t+\tau_j)\;=\;f(x;\theta_0+\tau_t),\qquad \forall\,x\in \mathcal{D}_t.\label{formula:simplified_wd}
\end{equation}
This simplification is sufficient for our analysis  because the linearity of interference ensures that pairwise results extend to the multi-task case. Our analysis concentrates on this in-domain condition, which is the core challenge, while the out-of-domain case follows from similar logic. A detailed justification for this reframing is provided in Appendix B.

We can now reframe the condition of weight disentanglement into a tractable one using the NTK linearization hypothesis, as formalized in the following lemma.
\begin{lemma}
\label{lemma:disentanglement_condition}
Under the NTK linearization hypothesis, weight disentanglement between tasks $t$ and $j$ is equivalent to the interference term from task $j$ being approximately zero on the data domain of task $t$:
\begin{equation}\small
\setlength\abovedisplayskip{5pt}
\setlength\belowdisplayskip{5pt}
\tau_j^{\top} J(x) = 0, \quad \forall x \in \mathcal{D}_t. \label{formula:wd_twotasks}
\end{equation}
\end{lemma}
Detailed proof in Appendix C. This condition forms the basis for our subsequent analysis, also identified as the key to disentanglement in recent literature~\cite{taujp2025}.

\subsection{Our Main Theorem}\label{sec:pretrain_ideal}
We now investigate the ideal conditions that enable perfect task arithmetic. We put forward the Task-Feature Specialization (TFS) property and show that it not only guarantees weight disentanglement but also gives rise to the geometric property of weight vector orthogonality.

\subsubsection{Task-Feature Specialization (TFS)}
To fundamentally explain why task arithmetic can work, we propose a new core concept: Task-Feature Specialization (TFS). It means that an ideal model, when faced with different tasks, intelligently allocates distinct internal features, represented by the column vectors of its weight matrices, to specific tasks. 
For instance, under the ideal TFS assumption, a task for classifying cars and a task for classifying MNIST digits would rely on two disjoint sets of internal features within the same layer of the model. We posit that this functional separation is the root of perfect task arithmetic. We formalize this intuition with the following definitions.

\begin{definition}[Task-Specialized Feature Set]\label{defin:task-spec}
    For a given linear layer with weight matrix $W$, we consider each column vector $\{w_k\}_{k=1}^d$ as extracting a ``base feature" whose activation is $z_k$. For a task $t$ with data domain $\mathcal{D}_t$, we define its specialized feature set $I_t \subseteq \{1, \dots, d\}$ as the set of indices for which the model's final output $f(x; \theta_0)$ is sensitive to the activation $z_k$ for inputs $x \in \mathcal{D}_t$. Formally, for any $k \notin I_t$, we have $\mathbb{E}_{x \sim \mathcal{D}_t} [|\frac{\partial f(x; \theta_0)}{\partial z_k}|] = 0$.
\end{definition}

We formalize our core assumption for the ideal case.

\begin{assumption}[Task-Feature Specialization]\label{assumption:task-specific}
    For two distinct tasks $t$ and $j$, their respective specialized feature sets, $I_t$ and $I_j$, are disjoint, \textit{i.e.}, $I_t \cap I_j = \emptyset$.
\end{assumption}

\subsubsection{TFS as a Sufficient Condition for WD}\label{subsec:ortho_pretrain_disanglement}
We now prove that TFS property is a sufficient condition for weight disentanglement. This provides a direct explanation for the success of task arithmetic: when the model functionally dedicates distinct features to distinct tasks.

\begin{restatable}{theorem}{thmPretrain}
\label{theorem:pretrain}
  Under the NTK linearization hypothesis (\Cref{subsec:ntk}) and the Task-Feature Specialization property, weight disentanglement between tasks $t$ and $j$ holds.
\end{restatable}

The proof is detailed in Appendix D.

\subsubsection{From TFS to Weight Vector Orthogonality}\label{subsec:weight_orthogonality}
More interestingly, we find that under the same TFS condition, a geometric property of the model's parameters can be derived: Weight Vector Orthogonality.

\begin{figure}[t]
  \centering
  \begin{subfigure}{0.49\linewidth}
    \includegraphics[width=\linewidth]{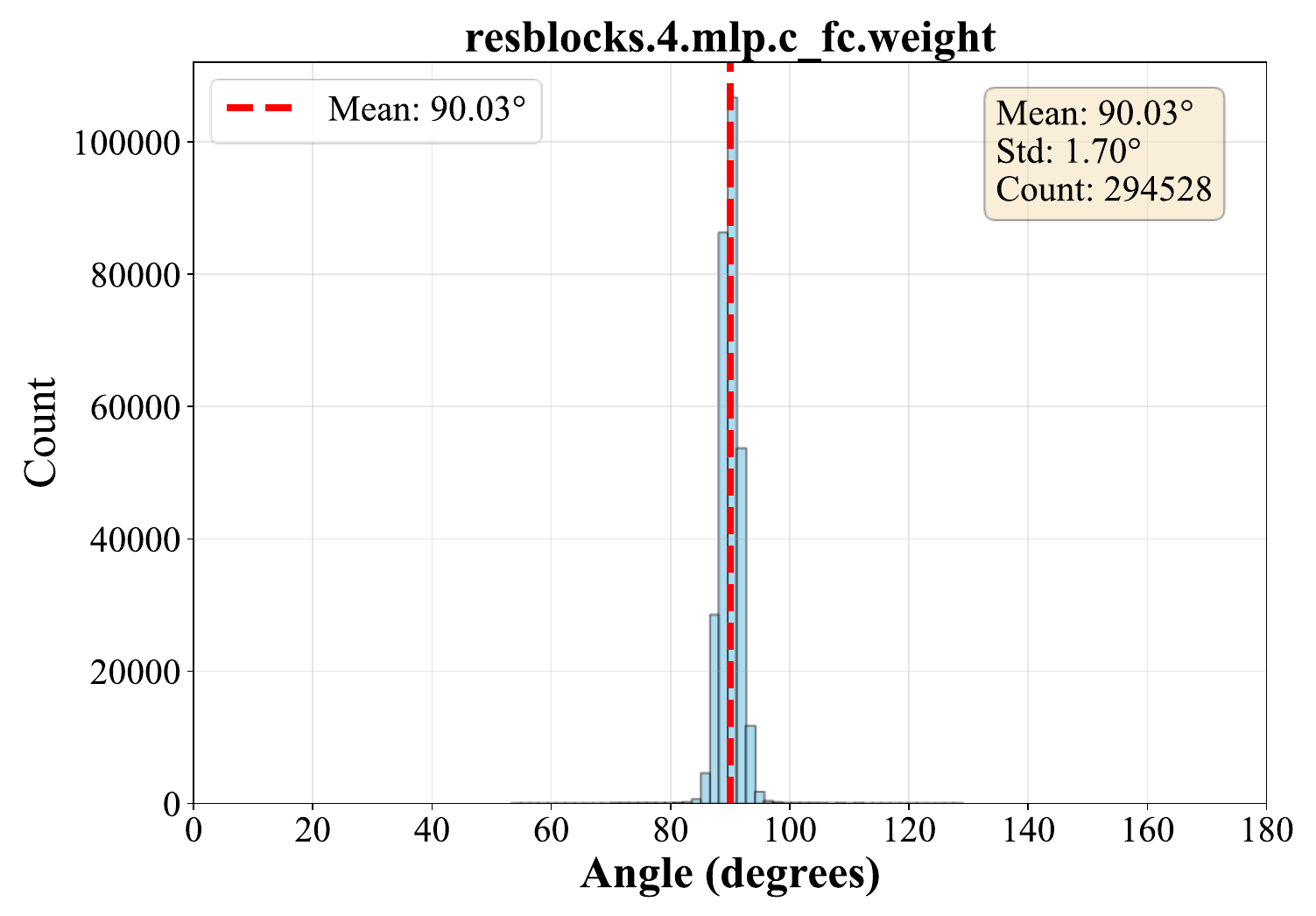}
    \caption{The distribution of angles between column vector pairs in a weight matrix.}
    \label{fig:layer_angle}
  \end{subfigure}
  \hfill
  \begin{subfigure}{0.49\linewidth}
    \includegraphics[width=\linewidth]{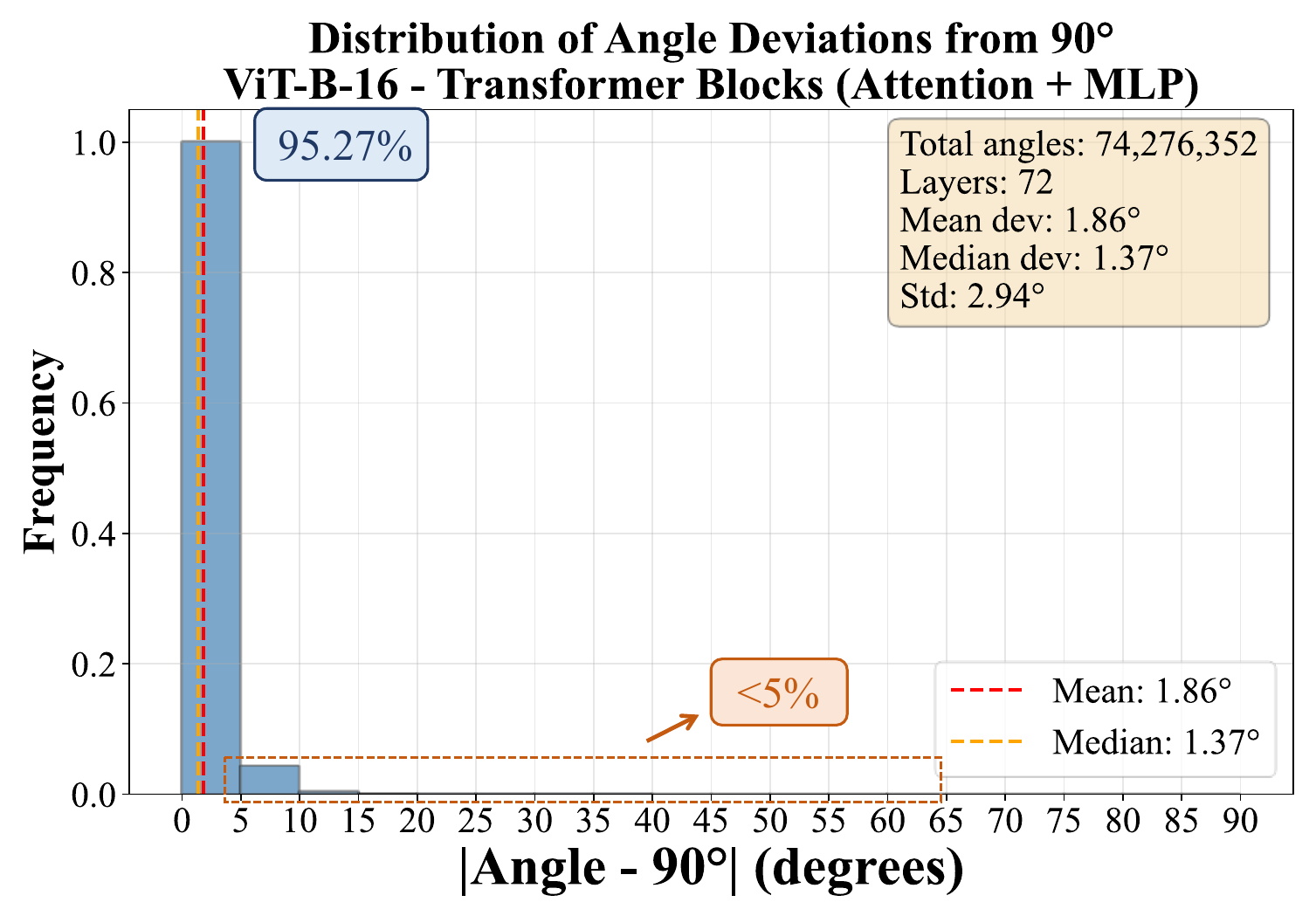}
    \caption{Statistical summary of angular deviations from $90^{\circ}$ across all linear layers of the model.}
    \label{fig:angle_statis}
  \end{subfigure}
  \caption{Empirical evidence of weight vector orthogonality in a pre-trained CLIP ViT-B/16.}
  \label{fig:angle}
  \vspace{-15pt}
\end{figure}

\begin{definition}[Weight Vector Orthogonality]\label{defin:weight_ortho}
A weight matrix $W \in \mathbb{R}^{m \times d}$ with column vectors $\{w_1, \dots, w_d\}$ is said to possess column orthogonality if its column vectors are mutually orthogonal, we distinguish between two key forms.

(a) Block Orthogonality. Given a partition of the column indices into disjoint sets $\{I_1, \dots, I_T\}$ (\textit{e.g.}, corresponding to different tasks), the matrix exhibits block orthogonality if any two vectors $w_k$ and $w_l$ from different sets are orthogonal (\textit{i.e.}, $\langle w_k, w_l \rangle = 0$ for all $k \in I_t, l \in I_j$ with $t \neq j$).

(b) Column-wise Orthogonality. The matrix exhibits column-wise orthogonality if all pairs of distinct column vectors are orthogonal (\textit{i.e.}, $\langle w_k, w_l \rangle = 0$ for all $k \neq l$). This can be seen as a special case of block orthogonality where each block contains only a single vector.
\end{definition}

The TFS property has a direct geometric consequence on the model's parameters. We can show that a model exhibiting TFS will naturally develop an orthogonal structure in its weights, which we formalize as the following corollary.
\begin{restatable}{corollary}{corollaryTFSWVO}
\label{corollary:TFS->WVO}
Given a model that adheres to the Task-Feature Specialization (TFS) property (Assumption~\ref{assumption:task-specific}), its weight matrices will exhibit Block Orthogonality.
\end{restatable}
The proof is detailed in Appendix E.

Empirically, we find that this predicted block orthogonality not only holds, but that the structure is often even stronger, approaching column-wise orthogonality.
In a pre-trained CLIP ViT-B/16 (\Cref{fig:angle}), the angles between all column vector pairs are sharply peaked at $90^{\circ}$. This suggests that pre-training pushes the entire weight matrix towards column-wise orthogonality (WVO) by also decorrelating features within the same task. (Full per-layer distributions are provided in Appendix J.1).

\subsubsection{Orthogonality as a Clue for Disentanglement}\label{subsubsec:clue_analysis}
Our analysis thus far provides an answer to the first question posed in Section~\ref{sec:intro}: What makes a pre-trained model $\theta_0$ good for task arithmetic? Our theory posits that the sufficient condition is task-feature specialization (TFS), which serves as a common cause for both weight disentanglement and block orthogonality. Although this geometric property should not be seen as a direct cause of disentanglement, it is a geometric consequence of the underlying functional separation (TFS) that effective training produces. TFS is an abstract property, but WVO provides a concrete, measurable signature. This relationship enables us to use WVO as a powerful diagnostic clue. As our Bayesian analysis suggests (see Appendix F), observing WVO in a model that has undergone effective training on diverse data strongly increases our belief that it has developed a TFS-like structure, and consequently, will exhibit disentanglement.

\subsection{Our Method OrthoReg}
Section~\ref{sec:pretrain_ideal} shows that task-feature specialization is sufficient for weight disentanglement. However, the TFS property is an idealization that rarely holds in practice. We now address this gap between the theory and realistic scenarios.

\setlength{\extrarowheight}{-1pt}
\begin{table*}[t]
\small
\centering
\caption{Task addition results on CLIP-based models. Performance of adding 8 task vectors on three architectures. Our proposed orthogonal regularization (+OrthoReg) is applied to several baselines, showing consistent improvements in both Absolute Accuracy (Abs.Acc.) and Normalized Accuracy (Norm.Acc.). An asterisk (*) denotes the best absolute accuracy for each model architecture.}
\label{tab:clip_task_addition}
\begin{tabular}{ccccccc}
\toprule
\multirow{2}{*}{\textbf{Method}} &
\multicolumn{2}{c}{\textbf{ViT-B-32, 8 tasks}} &
\multicolumn{2}{c}{\textbf{ViT-B-16, 8 tasks}} &
\multicolumn{2}{c}{\textbf{ViT-L-14, 8 tasks}} \\
\cmidrule(lr){2-3} \cmidrule(lr){4-5} \cmidrule(lr){6-7}
& \textbf{Abs.Acc.}\textbf{(\textuparrow)} & \textbf{Norm.Acc.} \textbf{(\textuparrow)}
& \textbf{Abs.Acc.}\textbf{(\textuparrow)} & \textbf{Norm.Acc.} \textbf{(\textuparrow)}
& \textbf{Abs.Acc.}\textbf{(\textuparrow)} & \textbf{Norm.Acc.} \textbf{(\textuparrow)} \\
\midrule
zero-shot & 47.74 & / & 54.22 & / & 64.54 & / \\
\midrule
Non-linear Finetuning~\cite{ta2023}  & 70.32 & 77.56 & 75.39 & 75.39 & 84.07 & 89.19 \\
Non-lin. FT\textbf{+OrthoReg} (ours) & \textbf{73.41} & \textbf{93.93} & \textbf{77.68} & \textbf{93.62} & \textbf{88.23} & \textbf{100.08} \\
\rowcolor{rowblue}
\quad  $\Delta$ & {\color{darkblue}+3.09} & {\color{darkblue}+16.37} & {\color{darkblue}+2.29} & {\color{darkblue}+18.23} & {\color{darkblue}+4.16} & {\color{darkblue}+10.89}  \\
\addlinespace[-2pt]
\midrule
Tangent Task Arithmetic~\cite{tta2023}  & 74.68 & 85.27 & 78.97 & 87.48 & 86.19 & 93.14 \\
TTA\textbf{+OrthoReg} (ours) & \textbf{76.35} & \textbf{91.81} & \textbf{79.85} & \textbf{88.02} & \textbf{87.52} & \textbf{96.44} \\
\rowcolor{rowblue}
\quad   $\Delta$ & {\color{darkblue}+1.67} & {\color{darkblue}+6.54} & {\color{darkblue}+0.88} & {\color{darkblue}+0.54} & {\color{darkblue}+1.33} & {\color{darkblue}+3.30} \\
\addlinespace[-2pt]
\midrule
Attention-Only Fine-tuning~\cite{linearatt2025} & 78.07 & 86.99 & 80.71 & 87.64 & 87.81 & 93.59 \\
ATT-FT\textbf{+OrthoReg} (ours) & {\textbf{80.87*}} & \textbf{99.76} & {\textbf{83.37*}} & \textbf{98.77} & {\textbf{90.41*}} & \textbf{100.05} \\
\rowcolor{rowblue}
\quad    $\Delta$ & {\color{darkblue}+2.80} & {\color{darkblue}+12.77} & {\color{darkblue}+2.66} & {\color{darkblue}+11.13} & {\color{darkblue}+2.60} & {\color{darkblue}+6.46} \\
\addlinespace[-2pt]
\midrule
LoRA-ATT & 73.84 & 84.29 & 75.51 & 83.17 & 87.02 & 93.33 \\
LoRA-ATT\textbf{+OrthoReg} (ours) & \textbf{76.00} & \textbf{86.10} & \textbf{79.67} & \textbf{87.70} & \textbf{89.16} & \textbf{95.50} \\
\rowcolor{rowblue}
\quad   $\Delta$ & {\color{darkblue}+2.16} & {\color{darkblue}+1.81} & {\color{darkblue}+4.16} & {\color{darkblue}+4.53} & {\color{darkblue}+2.14} & {\color{darkblue}+2.17} \\
\addlinespace[-2pt]
\bottomrule
\end{tabular}\label{tab:task_addition}
\vspace{-12pt}
\end{table*}

\subsubsection{The Challenge of Feature Overlap}
The core assumption underpinning our ideal case is that the specialized feature sets for distinct tasks are disjoint, \textit{i.e.}, $I_t \cap I_j = \emptyset$, that is often violated in practice, as distinct tasks could rely on shared underlying features. When this overlap occurs, \Cref{theorem:pretrain} does not apply anymore. Specifically, for a shared feature $k \in I_t \cap I_j$ and an input $x \in \mathcal{D}_t$, both the gradient term $\nabla_{w_k} f(x; \theta_0)$ (since $k$ is now relevant to task $t$) and the task vector component $(\tau_j)_k$ (since $k$ is now relevant to task $j$) are generally non-zero. Consequently, their inner product $\langle (\tau_j)_k, \nabla_{w_k} f(x; \theta_0) \rangle$ is highly likely to be non-zero, creating non-zero interference, breaking the disentanglement guarantee.

This reveals the limitations of relying solely on static $\theta_0$. In realistic scenarios where the ideal TFS property does not hold, the pre-trained model alone is insufficient to guarantee disentanglement. Consequently, the responsibility shifts towards the second avenue we identified in Section~\ref{subsec:wd}: the dynamic construction of the task vectors $\tau_t$ themselves. This brings us back to our second question: How can we actively construct ``good" task vectors that promote disentanglement, even when the ideal TFS condition is not met?

\subsubsection{Method: Orthogonal Regularizations}\label{subsec:method_reg}
Directly enforcing the abstract property TFS is intractable. Our theory suggests a practical alternative: enforcing its geometric consequence, orthogonality. While \Cref{corollary:TFS->WVO} proves that TFS leads to block orthogonality, it is hard to enforce this property as the number of feature blocks $I_t$ is implicitly determined by the specific task and cannot be known beforehand. To handle this, we propose enforcing a simpler and stronger condition: column-wise orthogonality on the weight updates. This forms the core motivation for our method, OrthoReg (\Cref{fig:orthoreg}). As will be proved in \Cref{theorem:ortho_updates}, enforcing this condition will minimize cross-task interference during model merging. 
In addition, from a representation learning perspective, this condition will encourage decorrelated intra-task features, which is a desirable inductive bias as it promotes a more efficient and less redundant feature basis for the task.

To achieve this, we introduce a novel regularization term to the standard fine-tuning loss function. The total loss for fine-tuning a model on a given task $t$ becomes,
\begin{equation}\small
\setlength\abovedisplayskip{5pt}
\setlength\belowdisplayskip{5pt}
    \mathcal{L}=\mathcal{L}_{\text{task}}(\theta_0+\Delta \theta)+\lambda \cdot \mathcal{L}_{\text{ortho}}(\Delta\theta),
\end{equation}
where $\mathcal{L}_{\text{task}}$ is the original task objective, $\Delta \theta$ represents all parameter updates, \textit{i.e.}, $\tau$. And $\lambda$ is a hyperparameter controlling the regularization strength, and $\mathcal{L}_{\text{ortho}}$ is our proposed orthogonal regularizer.

\begin{figure}[t]
  \centering
   \includegraphics[width=\linewidth]{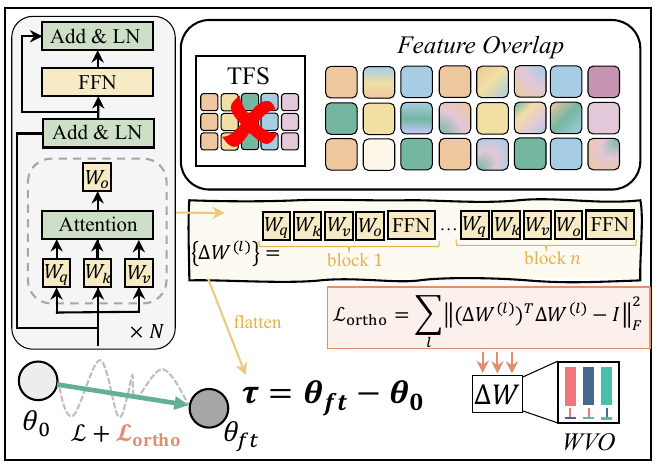}
   \caption{An overview of the OrthoReg method. It mitigates task interference caused by feature overlap by introducing $\mathcal{L}_{\text{ortho}}$. As illustrated for a Transformer block, this loss enforces an orthogonal structure on the weight updates ($\Delta W$) during fine-tuning.}
   \label{fig:orthoreg}
   \vspace{-15pt}
\end{figure}

\begin{definition}\label{defin:ortho_loss}
    The orthogonal regularization term is defined as the sum of penalties over all tuned linear layers,
\begin{equation}\small
\setlength\abovedisplayskip{5pt}
\setlength\belowdisplayskip{5pt}
    \mathcal{L}_{\text{ortho}}(\Delta\theta)=\sum_l \|(\Delta W^{(l)})^{\top}\Delta W^{(l)}-I\|_F^2,
\end{equation}
where the sum is overall linear layers $l$ being updated, $\Delta W^{(l)}$ is weight update matrix for that layer, $I$ is the identity matrix, and $\|\cdot\|_F^2$ denotes the squared Frobenius norm.
\end{definition}

This simple, plug-and-play regularizer penalizes the deviation of each update matrix's Gram matrix from the identity, thereby driving the columns of $\Delta W^{(l)}$ to become mutually orthogonal and have unit norm. 

\subsubsection{Theoretical Justification to OrthoReg}
We now present our second main theoretical result, which formalizes the effectiveness of our proposed method. This theorem shows that enforcing an orthogonal structure on weight updates serves as a key mechanism for disentanglement, even in the realistic scenario of feature overlap.

\begin{restatable}{theorem}{thmTau}\label{theorem:ortho_updates}
    Under the NTK linearization hypothesis (Section~\ref{subsec:ntk}), even if the Task-Feature Specialization property (Assumption 1) does not hold (\textit{i.e.}, $I_t \cap I_j \neq \emptyset$), constraining the task update matrices $\{\Delta W_t^{(l)}\}$ to be approximately internally orthogonal (as encouraged by the regularization in \Cref{defin:ortho_loss}) actively promotes weight disentanglement between tasks $t$ and $j$.
\end{restatable}

\noindent \textbf{Proof Sketch.} Our goal is to show that in this case the interference term $\tau_j^\top J(x)$ is approximately $0$ for $x \in \mathcal{D}_t$. The proof first establishes that for a typical input $x$ from task $t$'s domain, its Jacobian $J(x)$ is directionally aligned with the task vector $\tau_t$. This allows us to reframe the interference by approximating the angle involving the Jacobian, $\angle(\tau_j, J(x))$, with the angle between task vectors, $\angle(\tau_j, \tau_t)$,
\begin{equation}\small
\setlength\abovedisplayskip{5pt}
\setlength\belowdisplayskip{5pt}
    |\tau_j^\top J(x)| \approx ||\tau_j||_2 \cdot ||J(x)||_2 \cdot |\cos\angle(\tau_j, \tau_t)|.
\end{equation}
Then we demonstrate that our regularizer implements a dual control mechanism over the resulting interference expression. (1) Norm Control. It inherently bounds the magnitude of the task vector $\|\tau_j\|_2$. (2) Angle Control. More critically, by enforcing an internal orthogonal structure on each update matrix, it drives the angle between the different task vectors, $\angle(\tau_j, \tau_t)$, statistically towards 90 degrees. By simultaneously bounding the norm and nullifying the angle term, the regularizer ensures the expected interference is negligible, thus establishing weight disentanglement. The full proof is detailed in Appendix G.

\Cref{theorem:ortho_updates} provides a constructive answer to the second question in \Cref{sec:intro}: How can we construct good task vectors $\tau_t$ that promote disentanglement? Our analysis shows that \textbf{actively enforcing an internal orthogonal structure on $\Delta W$ serves as a direct and effective mechanism for achieving weight disentanglement}.

\subsection{Connection between Our Work and TTA}\label{subsec:connection_tta}
As a seminal theoretical analysis in task arithmetic, Tangent Task Arithmetic (TTA)~\cite{tta2023} demonstrates that fine-tuning within the tangent space of the pre-trained model $\theta_0$ effectively promotes weight disentanglement. We now connect our analysis with their findings. Our investigation reveals that \textit{both methods, despite their different implementations, derive their effectiveness from a shared underlying mechanism}: enforcing orthogonality between different task vectors ($\langle\tau_t, \tau_j\rangle \approx 0$), \textit{i.e.}, the ``Angle Control" part of our proof for \Cref{theorem:ortho_updates}.

Specifically, OrthoReg enforces orthogonality explicitly via a regularizer. In contrast, TTA achieves this implicitly through the model's NTK geometry, but at a high computational cost. TTA's reliance on Jacobian calculations can double memory usage and increase training time by 2-3x~\cite{linearatt2025}, posing a significant barrier to adoption. OrthoReg thus offers a more direct, efficient, and scalable alternative. A detailed derivation of this connection and a comparative analysis are provided in Appendix H.

\section{Experiments}
\subsection{Experimental Setup}
\textbf{Datasets and tasks.}
We follow the evaluation protocol established by~\cite{ta2023} and~\cite{tta2023}. The primary benchmark consists of eight diverse image classification datasets: Cars~\cite{Krause2013CollectingAL}, DTD~\cite{Cimpoi_2014_CVPR}, EuroSAT~\cite{Helber2017EuroSATAN}, GTSRB~\cite{Stallkamp2011TheGT}, MNIST~\cite{726791}, RESISC45~\cite{Cheng2017RemoteSI}, SUN397~\cite{5539970}, and SVHN~\cite{37648}. 

\noindent\textbf{Models and training methods.}
In our experiments, we adopt CLIP-pretrained Vision Transformers~\cite{clip2021} as pretrained model, including ViT-B-32, ViT-B-16 and ViT-L-14. During fine-tuning, the text encoder is frozen, and the image encoder can be updated. The regularization strength is selected via validation within the range [0.1, 100]. %

\setlength{\extrarowheight}{-1pt}
\begin{table*}[t]
\small
\centering
\caption{The minimum average Target Accuracy (Tar.Acc.) achievable while maintaining at least 95\% of the zero-shot accuracy on the ImageNet control task (Con.Acc.). Our proposed orthogonal regularization (+OrthoReg) shows a consistent and significant improvement in forgetting the target task. An asterisk (*) denotes the best (lowest) target accuracy for each model architecture.}
\label{tab:clip_task_negation}
\begin{tabular}{ccccccc}
\toprule
\multirow{2}{*}{\textbf{Method}} &
\multicolumn{2}{c}{\textbf{ViT-B-32, 8 tasks}} &
\multicolumn{2}{c}{\textbf{ViT-B-16, 8 tasks}} &
\multicolumn{2}{c}{\textbf{ViT-L-14, 8 tasks}} \\
\cmidrule(lr){2-3} \cmidrule(lr){4-5} \cmidrule(lr){6-7}
& \textbf{Tar.Acc.}\textbf{(\textdownarrow)} & \textbf{Con.Acc.} \textbf{(\textuparrow)}
& \textbf{Tar.Acc.}\textbf{(\textdownarrow)} & \textbf{Con.Acc.} \textbf{(\textuparrow)}
& \textbf{Tar.Acc.}\textbf{(\textdownarrow)} & \textbf{Con.Acc.} \textbf{(\textuparrow)} \\
\midrule
zero-shot & 47.74 & 66.70 & 54.22 & 68.34 & 64.54 & 77.44 \\
\midrule
Non-linear Finetuning~\cite{ta2023}  & 25.05 & 63.91 & 20.29 & 66.38 & 18.09 & 74.39 \\
Non-lin. FT\textbf{+OrthoReg} (ours) & \textbf{18.55} & 64.07 & \textbf{19.51} & 67.42 & \textbf{16.33} & 75.39 \\
\rowcolor{rowblue}
\quad  $\Delta$ & {\color{darkblue}-6.50} & {\color{darkblue}+0.16} & {\color{darkblue}-0.78} & {\color{darkblue}+1.04} & {\color{darkblue}-1.76} & {\color{darkblue}+1.00} \\
\addlinespace[-2pt]
\midrule
Tangent Task Arithmetic~\cite{tta2023}  & 11.47 & 63.99 & 9.33 & 66.82 & 8.36 & 74.39 \\
TTA\textbf{+OrthoReg} (ours) & {\textbf{11.39*}} & 64.07 & {\textbf{7.49*}} & 66.73 & {\textbf{8.36*}} & 74.87 \\
\rowcolor{rowblue}
\quad   $\Delta$ & {\color{darkblue}-0.06} & {\color{darkblue}+0.08} & {\color{darkblue}-1.84} & {\color{darkblue}-0.09} & {\color{darkblue}+0.00} & {\color{darkblue}+0.48} \\
\addlinespace[-2pt]
\midrule
Attention-Only Fine-tuning~\cite{linearatt2025}  & 19.39 & 64.90 & 19.20 & 67.75 & 24.85 & 76.42 \\
ATT-FT\textbf{+OrthoReg} (ours) & \textbf{15.67} & 64.16 & \textbf{14.78} & 66.81 & \textbf{14.67} & 75.40 \\
\rowcolor{rowblue}
\quad    $\Delta$ & {\color{darkblue}-3.72} & {\color{darkblue}-0.74} & {\color{darkblue}-4.42} & {\color{darkblue}-0.94} & {\color{darkblue}-10.18} & {\color{darkblue}-1.02}\\
\addlinespace[-2pt]
\midrule
LoRA-ATT & 20.10 & 64.51 & 19.44 & 67.28 & 22.17 & 75.81 \\
LoRA-ATT\textbf{+OrthoReg} (ours) & \textbf{19.19} & 64.43 & \textbf{17.25} & 67.08 & \textbf{13.94} & 74.45 \\
\rowcolor{rowblue}
\quad   $\Delta$ & {\color{darkblue}-0.91} & {\color{darkblue}-0.08} & {\color{darkblue}-2.19} & {\color{darkblue}-0.20} & {\color{darkblue}-8.23} & {\color{darkblue}-1.36}\\
\addlinespace[-2pt]
\bottomrule
\end{tabular}\label{tab:task_negation}
\vspace{-15pt}
\end{table*}

\noindent \textbf{Baselines.}
We evaluate our proposed orthogonal regularization against several state-of-the-art task arithmetic methods. For each baseline, we report the performance when our regularizer is applied, denoted by the ``+OrthoReg" suffix. The primary baselines are based on full-parameter fine-tuning. (1) Non-linear Fine-tuning (Nonlin.~FT). The standard task arithmetic approach~\cite{ta2023}. (2) Tangent Task Arithmetic (TTA)~\cite{tta2023} that fine-tunes on a linearized version of the model. (3) Attention-Only Fine-tuning (ATT-FT) that fine-tunes only attention modules~\cite{linearatt2025}. In addition, we investigate the effectiveness of OrthoReg on Parameter-Efficient Fine-Tuning (PEFT) approaches, such as LoRA~\cite{lora2022}. Our main result tables include a strong PEFT baseline, LoRA-ATT, where adapters are applied to the query, key, value, and output projections. A detailed analysis of other LoRA configurations is presented in Appendix J.5.

\noindent\textbf{Evaluation metrics.} Consistent with ~\cite{ta2023,tta2023,linearatt2025}, we use two metrics to evaluate performance. (1) Absolute Accuracy (Abs.Acc.), the standard classification accuracy of the merged multi-task model. (2) Normalized Accuracy (Norm.Acc.), which measures the performance of the multi-task model relative to individually fine-tuned single-task models. The definition is in Appendix I.

To ensure a fair comparison, a single, uniform scaling coefficient $\alpha$ is applied to the sum of all task vectors (\textit{i.e.}, $\theta_{MT} = \theta_0 + \alpha \sum \tau_t$). This single coefficient is optimized for each method via a grid search on $\{0.0, 0.05, \dots, 1.0\}$. We emphasize that we do not employ more complex, task-adaptive strategies that would assign a different coefficient $\alpha_t$ to each task vector. This approach, which is consistent with the evaluation protocols in prior work~\cite{tta2023,linearatt2025}, allows for a direct and fair comparison focused on the inherent quality of the task vectors produced by each method.

\subsection{Main Results on Task Addition}\label{subsec:results_addition}

The primary results for the task addition benchmark are summarized in \Cref{tab:task_addition}, where our method is applied to several leading task arithmetic methods across three scales of CLIP-based Vision Transformer models.

\textbf{Overall Performance Comparison.} \Cref{tab:task_addition} shows that our proposed orthogonal regularization (OrthoReg) consistently improves performance across all baselines and model scales. For instance, on the ViT-L-14 model, OrthoReg boosts the absolute accuracy of Non-lin.~FT~\cite{ta2023} by 4.16 points (from 84.07\% to 88.23\%) and enhances the seminal TTA~\cite{tta2023} by 1.33 points (from 86.19\% to 87.52\%). Similar gains are observed for other baselines, highlighting OrthoReg's effectiveness as a versatile, plug-and-play regularizer. These results empirically validate our hypothesis: actively enforcing an orthogonal structure on weight updates is a direct mechanism to mitigate interference. Notably, the ATT-FT+OrthoReg combination achieves the highest absolute accuracy across all tested configurations, establishing a new state-of-the-art on this benchmark. 

\begin{figure}[t]
  \centering
   \includegraphics[width=\linewidth]{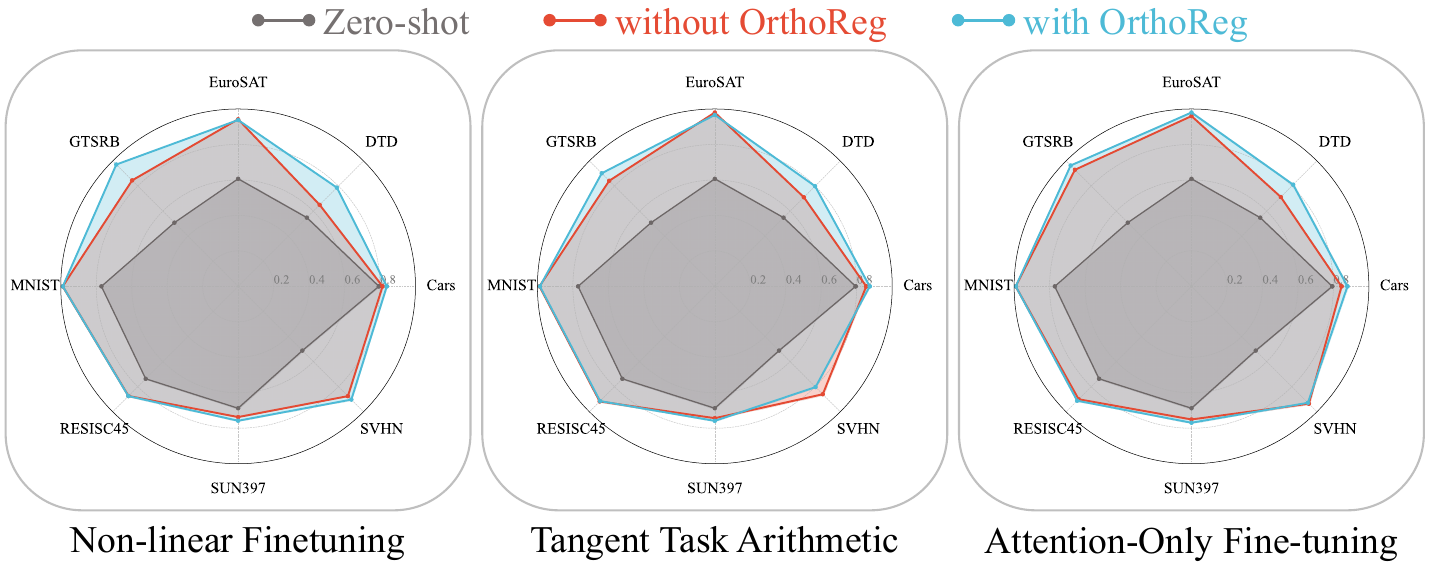}
   \caption{The accuracy of merged models (ViT-L-14) across the eight benchmark tasks. Each subplot shows the performance for a specific baseline method: zero-shot (gray), the baseline's merged model (red), and the baseline enhanced with our OrthoReg (blue). }
   \label{fig:radar-14}
   \vspace{-16pt}
\end{figure}

\textbf{Per-Task Performance Analysis.} \Cref{fig:radar-14} provides a per-task breakdown of these improvements on the ViT-L-14 model. The blue area, representing the performance with OrthoReg, shows a clear expansion compared to the red area (baseline) across the majority of tasks and methods. This demonstrates that the gains from OrthoReg are not merely an average effect but represent a balanced and widespread performance lift across most individual tasks. The corresponding radar charts for ViT-B-32 and ViT-B-16, which show similar trends, are included in Appendix J.2.

\textbf{Normalized Accuracy Analysis.} 
The impact of OrthoReg is particularly striking in the Normalized Accuracy. As shown in \Cref{tab:task_addition}, our method elevates the Norm.Acc. of Non-lin.~FT to 100.08\% and ATT-FT to 100.05\% on ViT-L-14. Achieving a normalized accuracy at or above 100\% is the functional realization of ideal weight disentanglement, as it signifies that the single merged model performs on par with eight individually specialized models, indicating a near-total absence of task interference. This result provides strong empirical validation for \Cref{theorem:ortho_updates}, demonstrating that enforcing an orthogonal geometry on weight updates is an effective mechanism to achieve this state.

\subsection{Main Results on Task Negation}\label{subsec:results_negation}
Beyond combining capabilities, we also evaluate task negation $\theta=\theta_0-\alpha\tau_t$, which aims to make a model forget a specific task. Effective forgetting requires a sharp drop in target task accuracy while preserving performance on a control task (ImageNet)~\cite{ta2023,tta2023}.

As shown in \Cref{tab:task_negation}, our OrthoReg regularizer significantly enhances the ``forgetting" effect across all baseline methods. For instance, when applied to ATT-FT on the ViT-L-14 model, OrthoReg reduces target task accuracy by an additional 10.18 percentage points. This more thorough forgetting is achieved without compromising the model's performance on ImageNet. Further details are provided in Appendix J.3. This result validates our theory that OrthoReg produces cleaner task vectors. Consequently, subtracting such a vector acts as a more precise ``undo" operation, cleanly removing the target capability with minimal side effects on the model's other general abilities.

\subsection{Validation of Inter-Task Orthogonality}\label{subsec:taskvector_sim}

\begin{figure}
  \centering
  \begin{subfigure}{0.49\linewidth}
    \includegraphics[width=\linewidth]{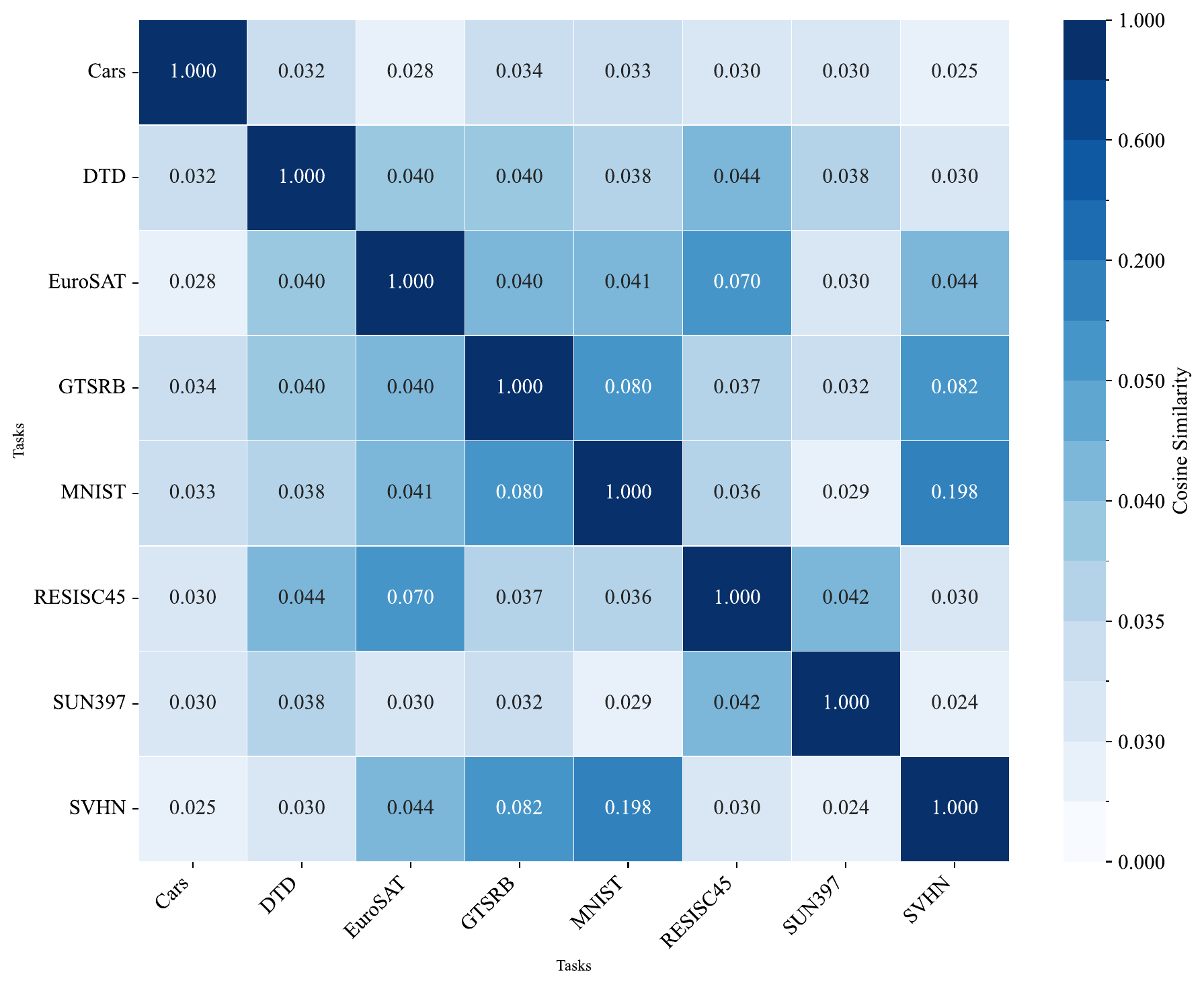}
    \caption{Non-lin. FT}
    \label{fig:sim-vit-b-16-standard}
  \end{subfigure}
  \hfill
  \begin{subfigure}{0.49\linewidth}
    \includegraphics[width=\linewidth]{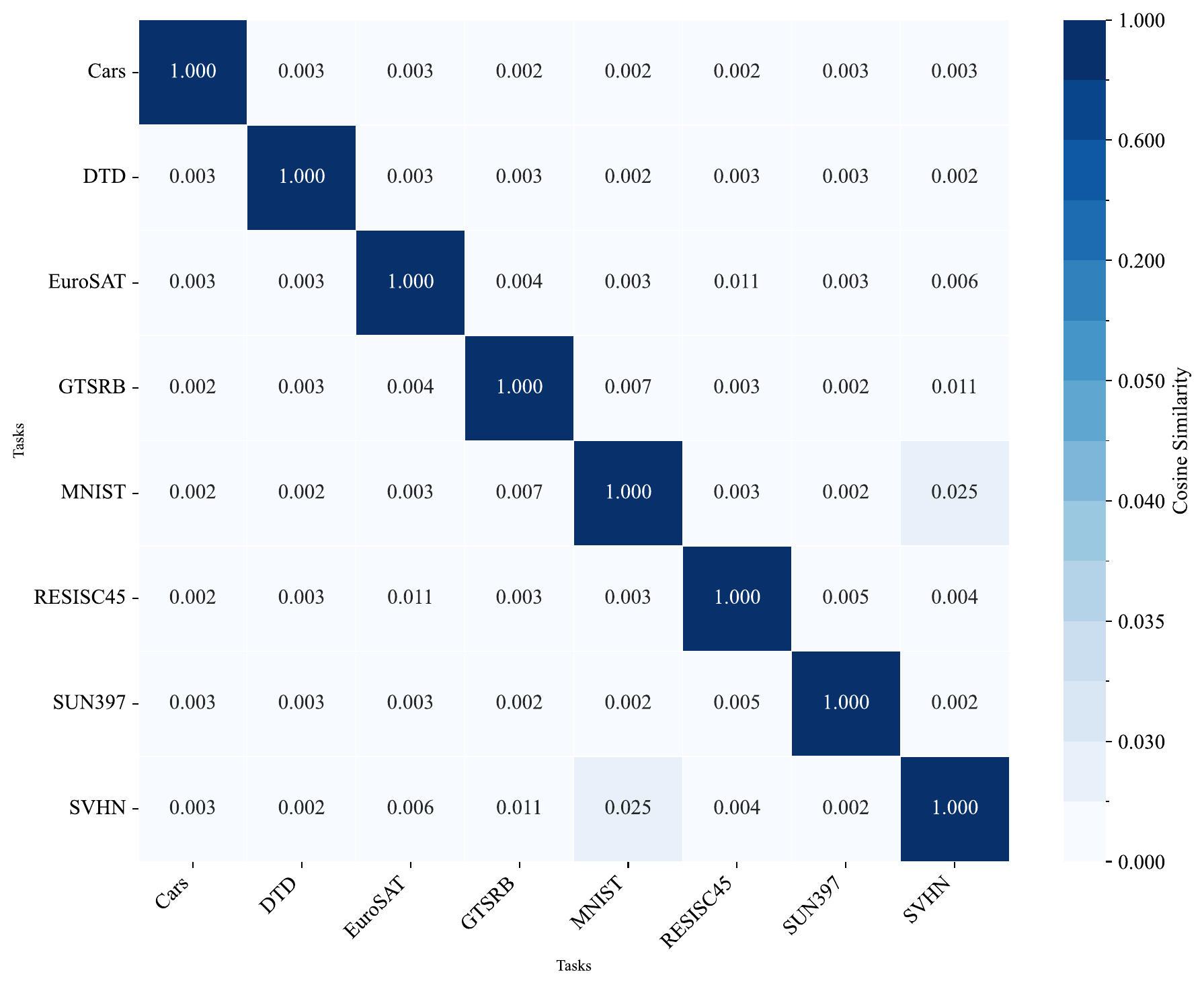}
    \caption{Non-lin. FT+OrthoReg}
    \label{fig:sim-vit-b-16-standard-ortho}
  \end{subfigure}
  \caption{Cosine similarity heatmaps of task vectors for ViT-B-16. (a) Task vectors from Non-lin.~FT show high similarity for several task pairs. (b) Task vectors trained with OrthoReg are significantly more orthogonal.}
  \label{fig:sim-vit-b-16-standard_all}
  \vspace{-17pt}
\end{figure}
Our theory predicts that OrthoReg promotes inter-task orthogonality ($\langle\tau_t, \tau_j\rangle \approx 0$), a mechanism we term ``Angle Control" in the proof of \Cref{theorem:ortho_updates}. To empirically validate this, we visualize the pairwise cosine similarity of task vectors in a heatmap. \Cref{fig:sim-vit-b-16-standard_all} compares the task vectors generated by Non-lin.~FT on ViT-B-16. The baseline heatmap (a) shows significant off-diagonal brightness, indicating high correlation between distinct task vectors. In contrast, after applying OrthoReg (b), the heatmap becomes markedly darker. This result provides direct empirical evidence for our theoretical claims, demonstrating that OrthoReg improves task arithmetic by producing more geometrically disentangled task vectors. Additional heatmaps showing similar trends for other methods are in Appendix~J.4.

\subsection{Parameter Sensitivity Analysis}

We analyze the sensitivity to two key hyperparameters: the regularization strength $\lambda$ and the task vector scaling coefficient $\alpha$. \Cref{fig:lambda} illustrates that the model's accuracy steadily improves as $\lambda$ is increased, demonstrating that the performance gain is a direct and consistent result of the orthogonalization, not sensitive hyperparameter tuning. \Cref{fig:alpha} shows that the model trained with OrthoReg consistently outperforms the baseline across a wide range of $\alpha$ values. This indicates that OrthoReg produces higher-quality task vectors, which not only achieve a higher peak accuracy but also make the task merging process more robust and less sensitive to the choice of the scaling factor.

\begin{figure}
  \centering
  \begin{subfigure}{0.49\linewidth}
    \includegraphics[width=\linewidth]{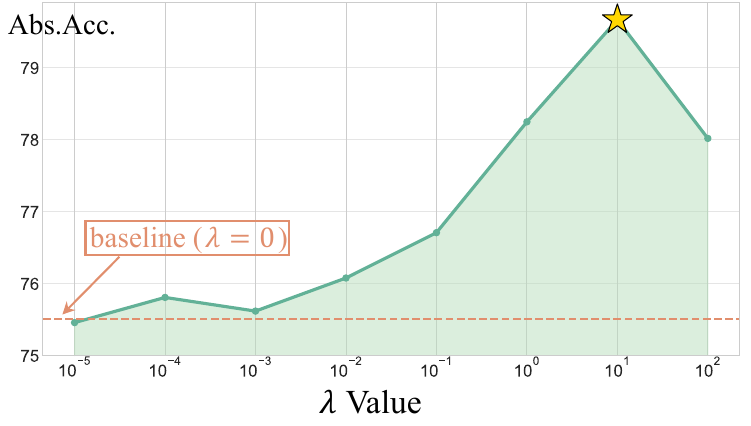}
    \caption{$\lambda$ Sensitivity}
    \label{fig:lambda}
  \end{subfigure}
  \hfill
  \begin{subfigure}{0.49\linewidth}
    \includegraphics[width=\linewidth]{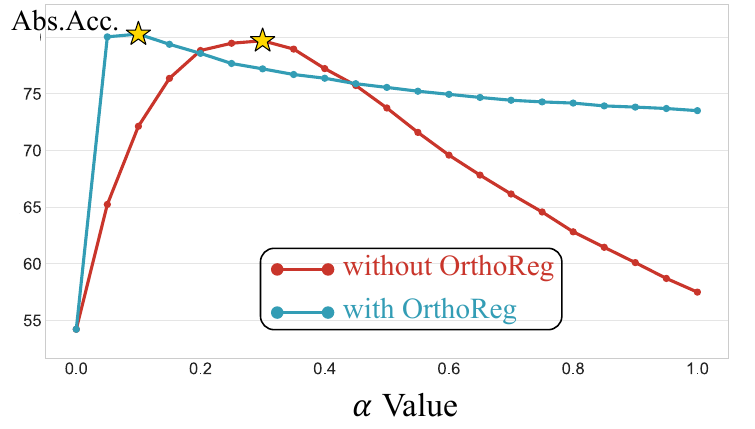}
    \caption{$\alpha$ Sensitivity}
    \label{fig:alpha}
  \end{subfigure}
  \caption{Analysis of hyperparameter sensitivity on ViT-B-16. (a) The impact of the regularization strength $\lambda$ on the performance of LoRA-ATT. (b) The influence of the merging coefficient $\alpha$ on the final accuracy of the merged model. The blue line (TTA+OrthoReg) consistently outperforms the red line (baseline TTA) across a wide range of $\alpha$ values.}
  \label{fig:hyperparameter}
  \vspace{-10pt}
\end{figure}

\section{Conclusion}
Understanding why task arithmetic works is key to making it a reliable engineering tool. In this paper, we advance this understanding by discovering that Task-Feature Specialization ensures weight disentanglement and creates a geometric consequence: weight vector orthogonality. This insight led us to OrthoReg, a method that promotes disentanglement by enforcing orthogonality on weight updates. We found OrthoReg significantly improves performance by creating more orthogonal task vectors. For future work, we plan to explore more diverse forms of orthogonality constraints for more powerful control over model merging.

\section*{Acknowledgement}
This work is supported in part by the National Natural Science Foundation of China (62576160, 62192783), Young Elite Scientists Sponsorship Program by CAST (2023QNRC001), and the Australian Research Council’s Discovery Project(DP220101784).
{
    \small
    \bibliographystyle{ieeenat_fullname}
    \bibliography{main}
}

\maketitlesupplementary
\appendix

\section{Note on the Scope of Analysis: Why Focus on Linear Layers}\label{subsec:linear_layers_seasons}
Throughout our theoretical analysis, we primarily focus on the parameters of linear layers, such as fully-connected (FC) layers and the projection matrices within attention mechanisms. We omit biases and parameters from normalization layers (\textit{e.g.}, LayerNorm). 

This simplification is well-justified, as linear layers constitute the vast majority of parameters~\cite{Kaplan2020,Zhang2022} in modern large-scale models like Transformers~\cite{transformer2017}, and their behavior consequently dictates the model's overall functional transformations and capacity for learning task-specific knowledge. Moreover, this focus aligns with established practices in the model merging literature, where complex strategies are often applied exclusively to linear layers~\cite{ta2023,Tiesmerging2023,tta2023}, suggesting a secondary role for biases and normalization parameters in the task interference phenomena we aim to mitigate. The centrality of these layers is further underscored by the success of parameter-efficient fine-tuning (PEFT) methods like LoRA~\cite{lora2022}, which demonstrate that model adaptation for new tasks primarily occurs within these linear components. 

Given this convergence of evidence, concentrating our geometric analysis on linear layers allows us to build a tractable yet powerful theoretical framework that captures the core mechanisms of task arithmetic.

\section{Justification for Two-Task Simplification}\label{sec:twotask_simplification}
In our main analysis (Section~4.1), we simplify the full definition of weight disentanglement (Definition~1) to a two-task, in-domain scenario as,
\begin{equation}\small
\setlength\abovedisplayskip{5pt}
\setlength\belowdisplayskip{5pt}
    f(x;\theta_0+\tau_t+\tau_j)\;=\;f(x;\theta_0+\tau_t),\qquad \forall\,x\in \mathcal{D}_t.
\end{equation}

This appendix provides a detailed justification for why this simplification is sufficient and does not result in a loss of generality. Our simplification is reasonable for two primary reasons.

First, our subsequent proofs focus on demonstrating that the pairwise interference term $\tau_j^{\top} J(x)$ is approximately zero for any $x$ in the data domain $\mathcal{D}_t$ of a different task $t$. This is the core of the disentanglement mechanism under the NTK linearization hypothesis. Due to the linearity of this interference term with respect to the task vectors, proving the disappearance of pairwise interference is sufficient for the general multi-task case. Specifically, if $\tau_j^{\top} J(x) \approx 0$ for all $j \neq t$, then the total interference from all other tasks in the merged model also vanishes,
\begin{equation}
    \sum_{j \neq t} \alpha_j (\tau_j^\top J(x)) \approx \sum_{j \neq t} \alpha_j \cdot 0 = 0.
\end{equation}
Therefore, focusing on two-task interaction $f(x;\theta_0+\tau_t+\tau_j)$ and omitting the scaling coefficients $\alpha$ during the proof does not compromise the generality of our conclusions.

Second, our analysis concentrates on the ``in-domain disentanglement" condition because it addresses the central challenge of eliminating crosstalk between actively composed tasks. The``out-of-domain preservation" condition, $f(x; \theta_0 + \sum_{t=1}^T \alpha_t \tau_t)=f(x;\theta_0)$ for $x \notin \bigcup_{t=1}^T \mathcal{D}_t$, can be established using the same underlying logic. For an out-of-domain sample $x_{\text{ood}}$, its processing should ideally not rely on the specialized features of any task $t$. This implies that the interference term $\tau_t^\top J(x_{ood})$ should be approximately zero for all task vectors $\tau_t$. This is a direct extension of the principle we prove for the in-domain case. By establishing the core argument for pairwise in-domain disentanglement, we effectively provide the necessary and sufficient reasoning to prove the full weight disentanglement property.

\section{Proof of \Cref{lemma:disentanglement_condition}}
\label{subsec:proof_lemma1}
In this part, we provide the detailed proof for \Cref{lemma:disentanglement_condition}, which establishes the equivalence between the functional property of weight disentanglement and a geometric orthogonality condition under the NTK linearization hypothesis.

\begin{lemma}
\label{lemma:disentanglement_condition}
Under the NTK linearization hypothesis, weight disentanglement between tasks $t$ and $j$ is equivalent to the interference term from task $j$ being approximately zero on the data domain of task $t$:
\begin{equation}\small
\setlength\abovedisplayskip{5pt}
\setlength\belowdisplayskip{5pt}
\tau_j^{\top} J(x) = 0, \quad \forall x \in \mathcal{D}_t. \label{formula:wd_twotasks}
\end{equation}
\end{lemma}

\begin{proof}
Our starting point is the simplified, two-task definition of weight disentanglement, which states that for any input $x$ from the data domain of task $t$, the following approximation should hold:
\begin{equation}
f(x; \theta_0 + \tau_t + \tau_j) = f(x; \theta_0 + \tau_t), \quad \forall x \in \mathcal{D}_t.
\label{eq:proof_start}
\end{equation}
We apply the first-order Taylor approximation from the NTK hypothesis to both sides of this equation.

For the left-hand side (LHS), the total parameter perturbation from the pre-trained state $\theta_0$ is $(\tau_t + \tau_j)$. The linearization is therefore,
\begin{align}
\text{LHS} &\approx f(x; \theta_0) + (\tau_t + \tau_j)^{\top} J(x) \nonumber \\
&= f(x; \theta_0) + \tau_t^{\top} J(x) + \tau_j^{\top} J(x).
\label{eq:proof_lhs}
\end{align}
For the right-hand side (RHS), the perturbation is simply $\tau_t$. The linearization is,
\begin{equation}
\text{RHS} \approx f(x; \theta_0) + \tau_t^{\top} J(x).
\label{eq:proof_rhs}
\end{equation}
By substituting these approximations from \Cref{eq:proof_lhs} and \Cref{eq:proof_rhs} back into the original weight disentanglement condition (\Cref{eq:proof_start}), we obtain,
\begin{equation}
f(x; \theta_0) + \tau_t^{\top} J(x) + \tau_j^{\top} J(x) \approx f(x; \theta_0) + \tau_t^{\top} J(x).
\end{equation}
Canceling the common terms $f(x; \theta_0)$ and $\tau_t^{\top} J(x)$ from both sides of the approximation leaves us with the final, equivalent condition:
\begin{equation}
\tau_j^{\top} J(x) = 0, \quad \forall x \in \mathcal{D}_t.
\end{equation}
This shows that, under NTK linearization, the functional requirement that task $j$ does not interfere with task $t$ is equivalent to the geometric condition that the task vector $\tau_j$ is orthogonal to the model's gradient Jacobian $J(x)$ for all data points $x$ in the domain of task $t$.
\end{proof}

\section{Detailed Proof of \Cref{theorem:pretrain}}\label{sec:proof_theorum1}
\subsection{Proof of Theorem~1}\label{subsec:proof_theorum1}
In this section, we provide the formal proof for Theorem~1.
\begin{restatable}{theorem}{thmPretrain}
\label{theorem:pretrain}
  Under the NTK linearization hypothesis (Section~3.3) and the Task-Feature Specialization property, weight disentanglement between tasks $t$ and $j$ holds.
\end{restatable}

\begin{proof}
    According to Lemma~1, our goal is to prove that the interference term $\tau_j^\top J(x)$ is approximately zero for any $x \in \mathcal{D}_t$. We can decompose this total interference into contributions from each linear layer. For clarity, we analyze the interference arising from a single weight matrix $W \in \mathbb{R}^{m \times d}$ and show it is zero. The conclusion generalizes to the entire model by summation.

    The interference contributed by $W$ is $\langle (\tau_j)_W, J_W(x) \rangle$, where $(\tau_j)_W$ and $J_W(x)$ are the components of the task vector and Jacobian corresponding to $W$. By decomposing this along the column vectors $\{w_1, \dots, w_d\}$ of $W$, we get,
\begin{equation}
    \mathrm{Interference}_W(x) = \sum_{k=1}^d \langle (\tau_j)_k, \nabla_{w_k} f(x; \theta_0) \rangle,
\end{equation}
where $(\tau_j)_k$ is the update applied to column $w_k$. We will show every term in this summation is approximately zero.

\textbf{Analysis of the gradient term ($\nabla_{w_k} f(x; \theta_0)$).} For an input $x \in \mathcal{D}_t$, the gradient of the model output with respect to a weight column $w_k$ can be expressed using the chain rule,
\begin{equation}
    \nabla_{w_k} f(x; \theta_0) = \frac{\partial f(x; \theta_0)}{\partial w_k}=\frac{\partial f(x; \theta_0)}{\partial z_k} \cdot \frac{\partial z_k}{\partial w_k}.
\end{equation}
According to Definition~2, if the feature index $k$ is not in the specialized set for task $t$ (\textit{i.e.}, $k \notin I_t$), the model's output is insensitive to it, meaning $\frac{\partial f(x; \theta_0)}{\partial z_k} \approx 0$. For $x \in \mathcal{D}_t$,
\begin{equation}
    k \notin I_t \implies \nabla_{w_k} f(x; \theta_0) \approx 0.\label{formula:gradient_equal_0}
\end{equation}

\textbf{Analysis of the task Vector term ($(\tau_j)_k$).} The task vector component $(\tau_j)_k$ is the accumulated update to weight $w_k$ from fine-tuning on task $j$. By definition, if feature $k$ is not specialized for task $j$ (\textit{i.e.}, $k \notin I_j$), the loss function for task $j$ is insensitive to it. This means the gradients with respect to $w_k$ computed on the data domain $\mathcal{D}_j$ are consistently negligible. Since $(\tau_j)_k$ is the sum of these negligible gradients, it will be approximately zero. (A detailed proof is provided in Appendix~\ref{subsec:proposition1} as Proposition~\ref{proposition:updated_tau_j_0}).
\begin{equation}
    k \notin I_j \implies (\tau_j)_k \approx 0. \label{formula:taujk_equal_0}
\end{equation}

Now, we examine each term $\langle (\tau_j)_k, \nabla_{w_k} f(x; \theta_0) \rangle$ in the summation for index $k \in \{1, \dots, d\}$. There are two mutually exclusive possibilities.

Case A: $k \in I_j$. By the Task-Feature Specialization property ($I_t \cap I_j = \emptyset$), it must be that $k \notin I_t$. From gradient analysis (\Cref{formula:gradient_equal_0}), this implies $\nabla_{w_k} f(x; \theta_0) \approx 0$. %

Case B: $k \notin I_j$. From task vector analysis (\Cref{formula:taujk_equal_0}), this implies $(\tau_j)_k \approx 0$. %

In both cases, the term $\langle (\tau_j)_k, \nabla_{w_k} f(x; \theta_0) \rangle$ vanishes. Since this holds for all $k$, the interference from this layer, $\mathrm{Interference}_W(x)$, is approximately zero. As this applies to all layers, the total interference $\tau_j^\top J(x) \approx 0$. By \Cref{lemma:disentanglement_condition}, this proves that weight disentanglement holds. 
\end{proof}

\subsection{Supporting Proposition for \Cref{theorem:pretrain}}\label{subsec:proposition1}
In this part, we provide a detailed proof for the proposition referenced in the proof of \Cref{theorem:pretrain}. This proposition formalizes the intuition that if a task does not depend on a specific feature, the fine-tuning process for that task will not significantly alter the weights associated with that feature.

\begin{proposition}\label{proposition:updated_tau_j_0}
Under the NTK Linearization hypothesis (Section~3.3) and the Task-Feature Specialization property, consider the fine-tuning process for task $j$ on its data domain $\mathcal{D}_j$. If a feature index $k$ does not belong to the specialized feature set for task $j$ (\textit{i.e.}, $k \notin I_j$), then the corresponding component of the resulting task vector, $(\tau_j)_k$, is approximately zero.
\begin{equation}
    k \notin I_j \implies (\tau_j)_k \approx 0.
\end{equation}
\end{proposition}

\begin{proof}
The task vector $\tau_j$ is defined as the total change in parameters after fine-tuning on task $j$, starting from the pre-trained weights $\theta_0$,
\begin{equation}
    \tau_j=\theta_j^*-\theta_0,
\end{equation}
where $\theta_j^*$ are the final fine-tuned parameters. The component $(\tau_j)_k$ specifically represents the change in the weight column $w_k$ of a given linear layer.

Let's model the fine-tuning process as a sequence of updates using a gradient-based optimizer, such as Stochastic Gradient Descent (SGD). For a total of $S$ update steps, the weight column $w_k$ is updated iteratively. The update rule for $w_k$ at step $s$ is,
\begin{equation}
    w_k^{(s+1)} = w_k^{(s)} - \eta \cdot \mathbb{E}_{x \sim \mathcal{D}_j} \left[ \nabla_{w_k} \mathcal{L}_j(x; \theta^{(s)}) \right],
\end{equation}
where $\eta$ is the learning rate, and $\theta^{(s)}$ represents the model parameters at step $s$, with the initial state being $\theta^{(0)} = \theta_0$.

Consistent with the perspective from work on Adaptive Weight Disentanglement (AWD)~\cite{awd2024} that views the task vector as the sum of accumulated gradients, the total change in the weight column $w_k$, which is the task vector component $(\tau_j)_k$, is the sum of all single-step updates over the course of training,
\begin{equation}
\begin{aligned}
    (\tau_j)_k &= w_{k}^{(S)} - w_{k}^{(0)} = \sum_{s=0}^{S-1} \left( w_k^{(s+1)} - w_k^{(s)} \right) \\
    & = -\eta \sum_{s=0}^{S-1} \mathbb{E}_{x \sim \mathcal{D}_j} \left[ \nabla_{w_k} \mathcal{L}_j(x; \theta^{(s)}) \right].\label{formula:tau_jk_sum}
\end{aligned}
\end{equation}
To prove that $(\tau_j)_k \approx 0$, we need to show that the expected gradient $\mathbb{E}_{x \sim \mathcal{D}_j} \left[ \nabla_{w_k} \mathcal{L}_j(x; \theta^{(s)}) \right]$ is approximately zero at every step $s$ of the fine-tuning process.

Let's analyze the gradient for a single data point $x \in \mathcal{D}_j$ using the chain rule,
\begin{equation}
    \nabla_{w_k} \mathcal{L}_j(x; \theta^{(s)}) = \frac{\partial \mathcal{L}_j}{\partial f(x; \theta^{(s)})} \cdot \frac{\partial f(x; \theta^{(s)})}{\partial z_k} \cdot \frac{\partial z_k}{\partial w_k}, \label{formula:gradient_3}
\end{equation}
where $z_k$ is the activation of the base feature corresponding to $w_k$. We analyze each term in this product.
\begin{itemize}
\item $\frac{\partial \mathcal{L}_j}{\partial f(x; \theta^{(s)})}$. This is the derivative of the loss with respect to the model's final output. Before the model has fully converged, this term is generally non-zero and bounded.

\item $\frac{\partial z_k}{\partial w_k}$. For a standard linear layer where $z_k = (w_k)^\top \mathrm{In}(x)$, this derivative is simply the input to the layer, $\mathrm{In}(x)$. This term is also non-zero and bounded.

\item $\frac{\partial f(x; \theta^{(s)})}{\partial z_k}$. It measures the sensitivity of the final model output to the intermediate feature activation $z_k$. Our core assumption is that $k \notin I_j$. By Definiton~2 (Task-Specialized Feature Set), this means that at the pre-trained state $\theta_0$, the model's output is insensitive to $z_k$ in expectation over the data domain $\mathcal{D}_j$,
\begin{equation}
    \mathbb{E}_{x \sim \mathcal{D}_j} \left\| \frac{\partial f(x; \theta_0)}{\partial z_k} \right\| \approx 0.
\end{equation}
The fine-tuning process occurs in the neighborhood of $\theta_0$. Under the NTK linearization hypothesis, the parameter changes are small, and the model's Jacobian is assumed to be stable. Therefore, for all steps $s$ in the fine-tuning process, $\theta^{(s)}$ remains close to $\theta_0$, and the sensitivity of the model's output to feature $z_k$ also remains negligible,
\begin{equation}
    \mathbb{E}_{x \sim \mathcal{D}_j} \left\| \frac{\partial f(x; \theta^{(s)})}{\partial z_k} \right\| \approx 0 \quad \text{for } s=0,1,\dots,S-1.
\end{equation}
\end{itemize}

Now, let's take the expectation of the full gradient expression (\Cref{formula:gradient_3}) over the data domain $\mathcal{D}_j$,
\begin{equation}
\begin{aligned}
    &\mathbb{E}_{x \sim \mathcal{D}_j} \left[ \nabla_{w_k} \mathcal{L}_j(x; \theta^{(s)}) \right] \\
    &= \mathbb{E}_{x \sim \mathcal{D}_j} \left[\underbrace{\frac{\partial \mathcal{L}_j}{\partial f(x; \theta^{(s)})}}_{\text{non-zero,bounded}} \cdot \underbrace{\frac{\partial f(x; \theta^{(s)})}{\partial z_k}}_{\text{Expectation}\approx 0} \cdot \underbrace{\frac{\partial z_k}{\partial w_k}}_{\text{non-zero,bounded}} \right].
\end{aligned}
\end{equation}
Since the expectation of the sensitivity term $\frac{\partial f}{\partial z_k}$ is approximately zero, and the other terms are bounded, the expectation of their product will also be approximately zero.

This holds for every step $s$ of the fine-tuning process. Substituting this result back into \Cref{formula:tau_jk_sum}, we find that the total update $(\tau_j)_k$ is a finite sum of near-zero vectors,
\begin{equation}
    (\tau_j)_k = -\eta \sum_{s=0}^{S-1} \underbrace{\mathbb{E}_{x \sim \mathcal{D}_j} \left[ \nabla_{w_k} \mathcal{L}_j(x; \theta^{(s)}) \right]}_{\approx 0} \approx 0.
\end{equation}
This demonstrates that if a feature is not part of a task's specialized set, the corresponding weights will remain virtually unchanged during fine-tuning for that task.

This completes the proof. 

\end{proof}

\section{Proof of \Cref{corollary:TFS->WVO}}\label{sec:proof_tfs_to_wvo}
This part provides the detailed proof for \Cref{corollary:TFS->WVO}, which establishes that the Task-Feature Specialization (TFS) property, a functional characteristic of an ideal pre-trained model, gives rise to a specific geometric structure in its parameters, namely, weight vector block-orthogonality. This formalizes the connection, that Weight Vector Orthogonality (WVO) is presented as a geometric consequence of TFS.

\begin{restatable}{corollary}{corollaryTFSWVO}
\label{corollary:TFS->WVO}
Given a model that adheres to the Task-Feature Specialization (TFS) property, its weight matrices will exhibit Block Orthogonality.
\end{restatable}

\subsection{TFS Implies Cross-Task Feature Decorrelation}
We begin by proving a key statistical consequence of TFS, which will be instrumental in our main proof.
The functional separation defined by TFS has a direct consequence on the statistical properties of the feature activations. We formalize this as the following proposition.

\begin{proposition}\label{proposition:feature_second_order_0}
    Under the Task-Feature Specialization (TFS) property, for any two distinct tasks $t \neq j$, and for any pair of features with indices $k \in I_t$ and $l \in I_j$, their activations $z_k$ and $z_l$ are approximately decorrelated over a mixed data distribution $\mu$. That is,
\begin{equation}
\text{Cov}_\mu(z_k, z_l) \approx 0.\\
\end{equation}
\end{proposition}
\begin{proof}
Let us assume the contrary. Suppose TFS holds, but two features $z_k$ (specialized for task $t$, \textit{i.e.}, $k \in I_t$) and $z_l$ (specialized for task $j$, \textit{i.e.}, $l \in I_j$) are statistically correlated. For simplicity, we can model this correlation with an approximate linear relationship,
\begin{equation}
    z_k \approx a \cdot z_l + b + \xi,
\end{equation}
where $a \neq 0$ is a correlation coefficient, $b$ is a bias, and $\xi$ is uncorrelated noise. This model implies that a change in $z_l$ systematically induces a change in $z_k$.

Now, consider the total derivative of the model's final output $f(x; \theta_0)$ with respect to the activation $z_l$ for an input $x$ from task $t$'s data domain, $\mathcal{D}_t$. Using the chain rule, the change in $f$ with respect to a change in $z_l$ has two paths: a direct path ($z_l \to f$) and an indirect path through the correlated feature $z_k$ ($z_l \to z_k \to f$).
\begin{equation}
    \frac{df(x; \theta_0)}{dz_l} = \frac{\partial f}{\partial z_l} + \frac{\partial f}{\partial z_k} \frac{\partial z_k}{\partial z_l}
\end{equation}

We analyze each term in the context of TFS for an input $x \in \mathcal{D}_t$.
\begin{itemize}
    \item $\frac{\partial f}{\partial z_l}$: Since $x \in \mathcal{D}_t$ and the feature $l$ is specialized for task $j$ ($l \in I_j$), the TFS assumption ($I_t \cap I_j = \emptyset$) implies $l \notin I_t$. By Definition~2, the model's output is insensitive to $z_l$ on this data domain. Thus, $\mathbb{E}_{x \sim \mathcal{D}_t} [|\frac{\partial f}{\partial z_l}|] \approx 0$.
    \item $\frac{\partial f}{\partial z_k}$: Since $x \in \mathcal{D}_t$ and the feature $k$ is specialized for task $t$ ($k \in I_t$), the model's output is sensitive to $z_k$. Thus, $\mathbb{E}_{x \sim \mathcal{D}_t} [|\frac{\partial f}{\partial z_k}|]$ is significantly non-zero.
    \item $\frac{\partial z_k}{\partial z_l}$: From our linear correlation model, this derivative is the correlation coefficient $a$, which we assumed to be non-zero.
\end{itemize}
Substituting these into the chain rule expression and taking the expectation over $\mathcal{D}_t$,
\begin{equation}
\begin{aligned}
    \mathbb{E}_{x \sim \mathcal{D}_t} \left| \frac{df}{dz_l} \right| & \approx \mathbb{E}_{x \sim \mathcal{D}_t} \left| \underbrace{\frac{\partial f}{\partial z_l}}_{\approx 0} + \underbrace{\frac{\partial f}{\partial z_k}}_{\text{non-zero}} \cdot \underbrace{\frac{\partial z_k}{\partial z_l}}_{\text{non-zero, } a} \right| \\
    &\approx |a| \cdot \mathbb{E}_{x \sim \mathcal{D}_t} \left| \frac{\partial f}{\partial z_k} \right|.
\end{aligned}
\end{equation}

Since $|a| \neq 0$ and $\mathbb{E}[|\frac{\partial f}{\partial z_k}|]$ is significantly non-zero, the result is a significantly non-zero value. This means that the model's output $f$ shows a non-negligible total sensitivity to the activation $z_l$ on data from task $t$.

This result, however, directly contradicts the premise of TFS. If a model has truly specialized feature $k$ for task $t$ and feature $l$ for task $j$, its function for task $t$ should not be affected by perturbations in $z_l$. The total effect of $z_l$ on the output, not just the partial derivative, should be negligible.

The contradiction arose from our initial assumption of correlation ($a \neq 0$). Therefore, that assumption must be false. We conclude that for TFS to hold, features specialized for different tasks must be statistically decorrelated.
\end{proof}

\subsection{Detailed proof of \Cref{corollary:TFS->WVO}}
\begin{proof}
The proof proceeds by first relating the geometric property of the weight matrix ($W^\top W$) to a statistical property of the feature activations (the covariance matrix $\Sigma_z$), and then showing that TFS imposes a block-diagonal structure on this covariance matrix.

\textbf{Step 1}: Connecting Weight Geometry to Feature Covariance.

Consider a single linear layer with weight matrix $W = [w_1, \dots, w_d] \in \mathbb{R}^{m \times d}$, input $\mathrm{In}(x) \in \mathbb{R}^m$, and feature activations $z = W^\top \mathrm{In}(x) \in \mathbb{R}^d$. We compute the covariance matrix $\Sigma_z$ of the feature activations under a mixed data distribution $\mu$,
\begin{equation}
    \Sigma_z = \mathbb{E}_{x \sim \mu}[(z - \mu_z)(z - \mu_z)^\top], \quad \text{where } \mu_z = \mathbb{E}_{x \sim \mu}[z].
\end{equation}

In modern deep neural networks, the presence of normalization layers like Layer Normalization (LN)~\cite{Ba2016} or Batch Normalization (BN)~\cite{IoffeS2015} is standard practice. A primary function of these layers is to standardize the activations, dynamically regulating their mean and variance~\cite{Ba2016, IoffeS2015, Huang2018}. This forces the mean of the layer's input, $\mu_{\mathrm{In}} = \mathbb{E}_{x \sim \mu}[\mathrm{In}(x)]$, to be approximately zero.

Consequently, the mean of the output feature activations is also approximately zero,
\begin{equation}
\mu_z = \mathbb{E}_{x \sim \mu}[W^\top \mathrm{In}(x)] = W^\top \mathbb{E}_{x \sim \mu}[\mathrm{In}(x)] = W^\top \mu_{\mathrm{In}} \approx 0.   \label{formula:z_0mean} 
\end{equation}
With this zero-mean property, the covariance matrix $\Sigma_z$ simplifies to the second-moment matrix,
\begin{equation}
    \Sigma_z = \mathbb{E}_{x \sim \mu} [z z^\top] = \mathbb{E}_{x \sim \mu} [W^\top \mathrm{In}(x) \mathrm{In}(x)^\top W].
\end{equation}
Since the weight matrix $W$ is constant with respect to the input $x$, we can move it outside the expectation:
\begin{equation}
    \Sigma_z = W^\top \left( \mathbb{E}_{x \sim \mu} [\mathrm{In}(x) \mathrm{In}(x)^\top] \right) W.
\end{equation}

At this point, we analyze the term $\mathbb{E}_{x \sim \mu}[\mathrm{In}(x)\mathrm{In}(x)^\top]$, which represents the second moment matrix of the layer’s input. As argued before, normalization layers standardize activations. Beyond just enforcing a zero mean, this process also regulates variance, driving the covariance matrix of the layer’s input, $\Sigma_{\mathrm{In}}$, towards a whitened state~\cite{Ba2016, IoffeS2015, Huang2018, Santurkar2018}. The covariance matrix of the input is defined as,
\begin{equation}
\begin{aligned}
    \Sigma_{\mathrm{In}} &= \mathbb{E}_{x \sim \mu}[(\mathrm{In}(x)-\mu_{\mathrm{In}})(\mathrm{In}(x)-\mu_{\mathrm{In}})^\top]\\
    &= \mathbb{E}_{x \sim \mu}[\mathrm{In}(x)\mathrm{In}(x)^\top] - \mu_{\mathrm{In}}\mu_{\mathrm{In}}^\top
\end{aligned}
\end{equation}
Given that the input is whitened, we have $\Sigma_{\mathrm{In}} \approx I_m$ and $\mu_{\mathrm{In}} \approx 0$. Substituting these into the definition gives us the second-moment matrix of the input,
\begin{equation}
    \mathbb{E}_{x \sim \mu}[\mathrm{In}(x)\mathrm{In}(x)^\top] = \Sigma_{\mathrm{In}} + \mu_{\mathrm{In}}\mu_{\mathrm{In}}^\top \approx I_m + 0 \cdot 0^\top = I_m.
\end{equation}

Substituting this result back into the expression for $\Sigma_z$, we arrive at the crucial link between the weights’ geometry and the features’ statistics,

\begin{equation}
    \Sigma_z = W^\top I_m W = W^\top W.
\end{equation}
This equation shows that in this case the Gram matrix of the weights, $W^\top W$, is identical to the covariance matrix of the feature activations, $\Sigma_z$. Proving that $W$ has block-orthogonal columns is now equivalent to proving that its Gram matrix $W^\top W$ is block-diagonal, which in turn is equivalent to proving that $\Sigma_z$ is block-diagonal.

\textbf{Step 2}:Proving the Block-Diagonal Structure of $\Sigma_z$.

An element $(\Sigma_z)_{kl}$ of the covariance matrix is, by definition, the covariance between $z_k$ and $z_l$, \textit{i.e.}, $(\Sigma_z)_{kl} = \text{Cov}_\mu(z_k, z_l)$.

Let's consider two distinct feature indices, $k \neq l$.

Case 1: Features are specialized for different tasks. Suppose $k \in I_t$ and $l \in I_j$ for two tasks $t \neq j$. According to \Cref{proposition:feature_second_order_0}, which we derived from the TFS property, the activations of these features are decorrelated over the mixed distribution $\mu$. Therefore, we directly have,
\begin{equation}
    (\Sigma_z)_{kl} = \text{Cov}_\mu(z_k, z_l) \approx 0.
\end{equation}

Case 2: Features are specialized for the same task.
    Suppose $k, l \in I_t$ for some task $t$, with $k \neq l$. Our theory does not make any assumption about intra-task feature decorrelation. Therefore, the term $(\Sigma_z)_{kl} = \text{Cov}_\mu(z_k, z_l)$ is not guaranteed to be zero and may be non-zero in general.

\textbf{Step 3}: Conclusion of Block-Orthogonality

From Step 2, we have shown that the off-diagonal elements of the covariance matrix $\Sigma_z$ are approximately zero whenever the indices correspond to different tasks. The elements corresponding to pairs of features within the same task may be non-zero. This means $\Sigma_z$ has a block-diagonal structure,
\begin{equation}
    \begin{aligned}
        \Sigma_z = W^\top W \approx \begin{pmatrix}
   \mathbf{B}_1 & \mathbf{0} & \dots & \mathbf{0} \\
   \mathbf{0} & \mathbf{B}_2 & \dots & \mathbf{0} \\
   \vdots & \vdots & \ddots & \vdots \\
   \mathbf{0} & \mathbf{0} & \dots & \mathbf{B}_T
\end{pmatrix}.
    \end{aligned}
\end{equation}
where $\mathbf{B}_t$ is the (generally non-diagonal) covariance sub-matrix for features whose indices are in the set $I_t$, and the $\mathbf{0}$ blocks represent matrices with near-zero entries.

The $(k, l)$-th element of the Gram matrix $W^\top W$ is the inner product of the column vectors $\langle w_k, w_l \rangle$. The block-diagonal structure of $W^\top W$ directly implies that if indices $k$ and $l$ belong to different blocks (\textit{i.e.}, $k \in I_t$ and $l \in I_j$ with $t \neq j$), their corresponding entry in the Gram matrix is approximately zero,
\begin{equation}
    \langle w_k, w_l \rangle = (W^\top W)_{kl} \approx 0. \quad \text{for } k \in I_t, l \in I_j, t \neq j
\end{equation}
This is precisely the definition of block-orthogonality for the columns of the weight matrix $W$. The set of column vectors $\{w_k\}_{k \in I_t}$ forms a subspace that is orthogonal to the subspace spanned by $\{w_l\}_{l \in I_j}$ for any $j \neq t$.

This completes the proof.
\end{proof}

\section{Bayesian Analysis of the Relationship between TFS, WVO, and WD}\label{sec:bayes}
This part provides a formal Bayesian analysis to justify the claim made in Section~4.2.4, that observing Weight Vector Orthogonality (WVO) in a pre-trained model strongly increases our belief that it will exhibit Weight Disentanglement (WD). This analysis formalizes the intuition that WVO acts as a powerful diagnostic clue for the desirable, yet abstract, property of Task-Feature Specialization (TFS).

Let us define three distinct events.
\begin{itemize}
    \item Event A: The model has achieved ideal Task-Feature Specialization (TFS). This represents the underlying, unobservable abstract property where the model allocates disjoint sets of internal features to different tasks.
    \item Event B: The model exhibits Weight Disentanglement (WD). This is the desired functional outcome where task vectors can be composed without destructive interference.
    \item Event C: The model's parameters possess Weight Vector Orthogonality (WVO). This is a concrete, measurable geometric property of the model's weight matrices.
\end{itemize}

Our core theory, as established in Section~4.2, posits that TFS is a sufficient condition for both WD (\Cref{theorem:pretrain}) and WVO (\Cref{corollary:TFS->WVO}). We can formalize this as a logical implication,
\begin{equation}
    A \implies (B \land C).
\end{equation}
This means that if Event A is true, then both Event B and Event C must also be true. Consequently, we have the conditional probabilities,
\begin{equation}
    P(B|A) = 1 \quad \text{and} \quad P(C|A) = 1.
\end{equation}

Our goal is to demonstrate that observing WVO (Event C) provides evidence for WD (Event B). In probabilistic terms, we aim to show that the posterior probability of WD given WVO is greater than the prior probability of WD,
\begin{equation}
    P(B|C) > P(B).
\end{equation}

First, we can expand the conditional probability $P(B|C)$ by conditioning on whether TFS (Event A) has occurred,
\begin{equation}
    P(B|C) = P(B|A, C)P(A|C) + P(B|\neg A, C)P(\neg A|C).
\end{equation}

Let's analyze the terms in this expression.

1. $P(B|A, C)$: Since Event A (TFS) is a sufficient condition for Event B (WD), if A is true, B must be true, regardless of C. Therefore, $P(B|A, C) = 1$.

2. $P(B|\neg A, C)$: This is the probability of observing WD when TFS is not present, even though WVO is. Without the foundational structure of TFS, WD is not guaranteed. It might occur due to other unknown reasons or by chance, but we can reasonably assume this probability is significantly less than 1. Let's denote this probability as $q$, where $0 \le q < 1$.

Substituting these into the equation, we get,
\begin{equation}
    P(B|C) = 1 \cdot P(A|C) + q \cdot P(\neg A|C).
\end{equation}
Rearranging this gives,
\begin{equation}
    P(B|C) = q + (1-q)P(A|C).
\end{equation}

Now, we examine the crucial term $P(A|C)$, which represents our updated belief in TFS after having observed WVO. Using Bayes' theorem,
\begin{equation}
    P(A|C) = \frac{P(C|A)P(A)}{P(C)}.
\end{equation}

As established earlier, $P(C|A) = 1$. This simplifies the expression to,
\begin{equation}
    P(A|C) = \frac{P(A)}{P(C)}.
\end{equation}
Here, $P(A)$ is our prior belief that a model has achieved TFS, and $P(C)$ is the prior probability of observing WVO. WVO is a specific geometric structure that is not guaranteed to occur in any arbitrary neural network; its emergence is non-trivial. Therefore, it is safe to assume that $P(C) < 1$.

This leads to a key inequality,
\begin{equation}
    P(A|C) = \frac{P(A)}{P(C)} > P(A).
\end{equation}

This inequality formally captures our intuition: observing the geometric signature of WVO (Event C) strengthens our belief that the model has developed the underlying functional structure of TFS (Event A).

To complete the proof, we compare the expression for $P(B|C)$ with the unconditional prior probability of WD, $P(B)$. Using the law of total probability again,
\begin{equation}
    P(B) = P(B|A)P(A) + P(B|\neg A)P(\neg A).
\end{equation}

We know $P(B|A) = 1$. For the term $P(B|\neg A)$, we introduce a reasonable assumption: in the absence of the common cause (TFS), its consequences (WD and WVO) are approximately conditionally independent.

\begin{equation}
    P(B|\neg A, C) \approx P(B|\neg A).
\end{equation}
This assumption is justified because if the fundamental mechanism (TFS) that links WD and WVO is absent, the correlation between them should vanish or be significantly diminished. Any residual correlation would be a minor influence. Under this assumption, $P(B|\neg A) \approx P(B|\neg A, C) = q$.

Substituting this into the expression for $P(B)$:
\begin{equation}
    P(B) \approx 1 \cdot P(A) + q \cdot (1 - P(A)) = q + (1-q)P(A).
\end{equation}

We now have two expressions to compare:

1.  $P(B|C) = q + (1-q)P(A|C)$;

2.  $P(B) \approx q + (1-q)P(A)$.

We have proved that $P(A|C) > P(A)$. Since $q < 1$, the term $(1-q)$ is positive. It therefore follows directly that,
\begin{equation}
    P(B|C) > P(B)
\end{equation}
This result provides a rigorous probabilistic foundation for our central thesis. It demonstrates that observing the measurable geometric property of Weight Vector Orthogonality is a strong piece of evidence that increases the likelihood that the model also possesses the desired functional property of Weight Disentanglement. This justifies using WVO as a diagnostic tool to assess a model's suitability for task arithmetic.

\section{Detailed Proof of \Cref{theorem:ortho_updates}}\label{sec:proof_theorem2}
\subsection{Proof of \Cref{theorem:ortho_updates}}
\begin{restatable}{theorem}{thmTau}\label{theorem:ortho_updates}
    Under the NTK linearization hypothesis (Section~3.3), even if the Task-Feature Specialization property does not hold (\textit{i.e.}, $I_t \cap I_j \neq \emptyset$), constraining the task update matrices $\{\Delta W_t^{(l)}\}$ to be approximately internally orthogonal (as encouraged by the regularization in Definition~4) actively promotes weight disentanglement between tasks $t$ and $j$.
\end{restatable}

\begin{proof}

According to \Cref{lemma:disentanglement_condition}, our goal is to demonstrate that the interference from task $j$ on the data domain of task $t$ is approximately zero, \textit{i.e.}, $\tau_j^\top J(x) \approx 0$. The interference term's magnitude can be expressed as,
\begin{equation}
    |\tau_j^\top J(x)| = ||\tau_j||_2 \cdot ||J(x)||_2 \cdot |\cos\angle(\tau_j, J(x))|.
\end{equation}

The proof proceeds in four steps. We first reframe the angle term, then demonstrate how our regularizer controls both the norm and angle terms, and finally synthesize the results.

\textbf{Step 1:} Directional Alignment. 

First, we establish that for a typical input $x \in \mathcal{D}_t$, its Jacobian $J(x)$ is directionally aligned with the task vector $\tau_t$. The direction of $\tau_t$ is determined by the average Jacobian over the task's data domain, $\mu_J := \mathbb{E}_{x \in \mathcal{D}_t}[J(x)]$. Under a reasonable data consistency assumption, the gradients of different samples are statistically consistent rather than random, the direction of a typical $J(x)$ aligns with that of $\mu_J$ and, by extension, with $\tau_t$. This alignment, rigorously proven in Appendix~\ref{subsec:proof_direction}, allows us to reframe the term's angle using the angle between the two task vectors,
\begin{equation}
    |\tau_j^\top J(x)| \approx ||\tau_j||_2 \cdot ||J(x)||_2 \cdot |\cos\angle(\tau_j, \tau_t)|.
\end{equation}

\textbf{Step 2:} Norm Control. 

Our second step is to show that the orthogonal regularization term $\mathcal{L}_{\text{ortho}}$ effectively bounds the norm of the task vectors. The regularizer penalizes the deviation of each update matrix $\Delta W$ from the identity. By solving a constrained optimization problem, we can prove that the Frobenius norm of an update matrix $\Delta W$ is strictly bounded by its deviation from orthogonality. As formalized in Proposition \Cref{proposition:norm_control} (see Appendix~\ref{subsec:propostion2}), if $\|\Delta W^\top \Delta W - I\|_F^2 \le \xi$, then the norm is bounded by,
\begin{equation}
    \|\Delta W\|_F^2 \le d + \sqrt{d\xi},
\end{equation}
where $d$ is the number of columns. As the task vector's total norm is determined by the norms of its constituent update matrices, $\|\tau_j\|_2^2 = \sum_l \|\Delta W_j^{(l)}\|_F^2$, our regularizer effectively constrains the overall magnitude of $\tau_j$.

\textbf{Step 3}: Angle Control. 

Our third and most critical step is to demonstrate that the regularization statistically promotes orthogonality between different task vectors, \textit{i.e.}, $\mathbb{E}[|\cos\angle(\tau_j, \tau_t)|] \approx 0$.

The core mechanism is that the internal orthogonal structure imposed on each update matrix $\Delta W$ induces inter-task statistical orthogonality between the resulting task vectors $\tau_t$ and $\tau_j$. This can be understood through the lens of Polar Decomposition~\cite{Polar1990}, which allows us to express any approximately orthogonal update matrix $\Delta W$ as $\Delta W = QP$, where $Q$ is a strictly orthonormal matrix (an element of the Stiefel manifold $V_d(\mathbb{R}^m)$) and $P$ is a symmetric positive semi-definite matrix that is very close to the identity (as formalized in \Cref{proposotion:P_matrix_property} in Appendix~\ref{subsec:proposition3}).

Consequently, the inner product of two task vectors, $\langle \tau_t, \tau_j \rangle$, which is a sum of layer-wise inner products $\sum_l \langle \text{vec}(\Delta W_t^{(l)}), \text{vec}(\Delta W_j^{(l)}) \rangle$, is dominated by the sum of inner products of their orthonormal components, $\sum_l \langle \text{vec}(Q_t^{(l)}), \text{vec}(Q_j^{(l)}) \rangle$ (see Appendix~\ref{subsec:proof_angle_control} for a detailed derivation). As established in \Cref{lemma:stiefel} (Appendix~\ref{subsec:Lemma2-stiefel}), two matrices independently and uniformly drawn from the Stiefel manifold are, when vectorized, statistically orthogonal. Their inner product has an expected value of zero and its probability distribution is sharply peaked at zero. This strong statistical tendency towards orthogonality at each layer propagates to the entire task vectors, ensuring that $\tau_t$ and $\tau_j$ are highly likely to be nearly orthogonal. The detailed proof can be seen in Appendix~\ref{subsubsec:proof_angle}

\textbf{Step 4}: Completing the Proof. 
We now synthesize the results. The magnitude of the interference term is given by,
\begin{equation}
    |\tau_j^\top J(x)| \approx \underbrace{||\tau_j||_2}_{\text{Bounded}} \cdot \underbrace{||J(x)||_2}_{\text{Inherently Bounded}} \cdot \underbrace{|\cos\angle(\tau_j, \tau_t)|}_{\text{Statistically near zero}}
\end{equation}
The dual control mechanism of our regularization ensures that this product is approximately zero in expectation. The norm $\|\tau_j\|$ is bounded (Step 2), $\|J(x)\|$ is bounded for any given input and model, and the cosine of the angle between task vectors is statistically driven towards zero (Step 3). Consequently, the expected interference is negligible,
\begin{equation}
    \mathbb{E}[|\tau_j^\top J(x)|] \approx 0.
\end{equation}
By \Cref{lemma:disentanglement_condition}, this establishes that weight disentanglement is approximately achieved. This completes the proof.
\end{proof}

\subsection{Proof of Directional Alignment (Step 1)}\label{subsec:proof_direction}
In this section, we provide a rigorous proof for the claim that for a typical input $x \in \mathcal{D}_t$, its Jacobian vector $J(x)$ is directionally aligned with the task vector $\tau_t$. 

\begin{proof}
The proof proceeds in two parts: first, relating the direction of $\tau_t$ to the average Jacobian $\mu_J$, and second, relating the direction of an individual $J(x)$ to $\mu_J$.

\textbf{Part 1}: Direction of the Task Vector $\tau_t$.

As clarified in \Cref{formula:tau_jk_sum}, the task vector $\tau_t$ is the result of accumulated gradients during fine-tuning. In the initial phase of fine-tuning, where the parameters $\theta$ are close to $\theta_0$, the direction of $\tau_t$ is dominated by the average gradient of the task loss $\mathcal{L}_t$ over the data domain $\mathcal{D}_t$, evaluated at $\theta_0$.
\begin{equation}
    \tau_t \propto - \mathbb{E}_{(x, y) \sim \mathcal{D}_t} \left[ \nabla_\theta \mathcal{L}_t(f(x; \theta_0), y) \right].
\end{equation}
Using the chain rule, $\nabla_\theta \mathcal{L}_t = \frac{\partial \mathcal{L}_t}{\partial f} \cdot \nabla_\theta f = \frac{\partial \mathcal{L}_t}{\partial f} \cdot J(x)$. The expression becomes,
\begin{equation}
    \tau_t \propto - \mathbb{E}_{x \sim \mathcal{D}_t} \left[ \frac{\partial\mathcal{L}_t}{\partial f}\cdot J(x) \right].
\end{equation}
For a well-posed learning task, the loss derivative $\frac{\partial\mathcal{L}_t}{\partial f}$ (which indicates how the loss changes with respect to the model's output) can be assumed to be an approximately constant scalar $k_t$ across the dataset. This yields,
\begin{equation}
    \tau_t \propto - k_t \cdot \mathbb{E}_{x \in \mathcal{D}_t} [J(x)].
\end{equation}
Let $\mu_J := \mathbb{E}_{x \in \mathcal{D}_t} [J(x)]$ be the average Jacobian vector over the data domain of task $t$. We thus establish the first directional link,
\begin{equation}
    \text{Direction}(\tau_t) = \text{Direction}(\mu_{J}).
\end{equation}

\textbf{Part 2}: Direction of an Individual Jacobian $J(x)$.

Next, we formalize an intuitive hypothesis. For a well-defined, non-random machine learning task, the loss function's gradient directions for different samples within its data domain should exhibit statistical consistency, rather than pointing randomly in all directions throughout the parameter space. This consistency is fundamental to the model's ability to learn generalizable patterns from data. Applied to our scenario, this implies that the distribution of the Jacobian vectors $J(x)$ should not be overly dispersed.

We formalize this as the Data Consistency Assumption.
\begin{assumption}[Data Consistency Assumption]\label{assumption:data_consistency}
For a well-defined task, the Jacobian vectors of individual samples are statistically concentrated around their mean. This means the variance of the Jacobians, $\sigma_J^2 := \mathbb{E}_{x \in \mathcal{D}_t} \left[\| J(x) - \mu_J\|_2^2 \right]$, is significantly smaller than the squared norm of their mean,
\begin{equation}
    \sigma_J^2 \ll \|\mu_J\|_2^2.
\end{equation}
\end{assumption}

By Chebyshev's inequality~\cite{chebyshev1974}, for any constant $C > 1$, we have,
\begin{equation}
\begin{aligned}
     \mathbb{P}\left( \|J(x) - \mu_J\|_2^2 \ge C^2 \sigma_J^2 \right) &\le \frac{\mathbb{E}\left[ \|J(x) - \mu_J\|_2^2 \right]}{C^2 \sigma_J^2} \\
     &= \frac{\sigma_J^2}{C^2 \sigma_J^2} = \frac{1}{C^2}.
\end{aligned}
\end{equation}
This implies that the squared Euclidean distance between the random vector $J(x)$ and its mean $\mu_J$ is bounded by $C^2 \sigma_J^2$ with a probability of at least $1 - 1/C^2$. In other words, for a ``typical" (\textit{i.e.}, high-probability) sample $x'$, its Jacobian vector $J(x)$ satisfies,
\begin{equation}
    \|J(x') - \mu_J\|_2 < C \sigma_J.
\end{equation}

Now, we bound the angle $\theta_{x'} = \angle(J({x'}), \mu_J)$ for such a typical sample. Consider the triangle formed by the origin and the endpoints of the vectors $J({x'})$ and $\mu_J$. By the properties of vector geometry (related to the Law of Sines~\cite{Sines1967}), the sine of the angle $\theta_{x'}$ is bounded by the ratio of the length of the opposing side to the length of the adjacent side,
\begin{equation}
    \sin(\theta_{x'}) \le \frac{\|J(x') - \mu_J\|_2}{\|J(x')\|_2}.
\end{equation}
We have an upper bound for the numerator, $\|J(x) - \mu_J\|_2 < C \sigma_J$. For the denominator, we use the reverse triangle inequality~\cite{Triangle1976} to find a lower bound,
\begin{equation}
\begin{aligned}
    \|J(x')\|_2 &= \|\mu_J + (J(x') - \mu_J)\|_2 \\
    &\ge \|\mu_J\|_2 - \|J(x') - \mu_J\|_2 \\
    &> \|\mu_J\|_2 - C \sigma_J.
\end{aligned}
\end{equation}
Substituting these bounds, we get,
\begin{equation}
    \sin(\theta_{x'}) < \frac{C \sigma_J}{\|\mu_J\|_2 - C \sigma_J} = \frac{C (\sigma_J / \|\mu_J\|_2)}{1 - C (\sigma_J / \|\mu_J\|_2)}.
\end{equation}

Given Assumption~\ref{assumption:data_consistency}, the ratio $\sigma_J / \|\mu_J\|_2$ is a value much smaller than $1$. Therefore, the right-hand side of the inequality is a very small positive number. Since $\sin(\theta_{x'})$ is very small, the angle $\theta_{x'}$ must also be very close to zero. This establishes our second directional link,
\begin{equation}
    \text{Direction}(J(x)) \approx \text{Direction}(\mu_J), \quad \text{for a typical } x \in \mathcal{D}_t.
\end{equation}

Combining the two parts, we have shown that for a typical sample $x \in \mathcal{D}_t$,
\begin{equation}
    \text{Direction}(J(x)) \approx \text{Direction}(\mu_J) \approx \text{Direction}(\tau_t).
\end{equation}
This directional alignment justifies the approximation used in the main proof, allowing the angle between $\tau_j$ and $J(x)$ to be replaced by the angle between $\tau_j$ and $\tau_t$. This completes the proof.
\end{proof}

\subsection{\Cref{proposition:norm_control} and Proof (Norm Control)}\label{subsec:propostion2}
\begin{proposition}\label{proposition:norm_control}
The Frobenius norm of a matrix is bounded by its deviation from orthonormality. Specifically, for a matrix $W \in \mathbb{R}^{m \times d}$, if its deviation from being identity is bounded by $\|W^\top W - I_d\|_F^2 \le \xi$ for some constant $\xi \ge 0$, then its squared Frobenius norm is bounded by,
\begin{equation}
    \|W\|_F^2 \le d + \sqrt{d\xi}.
\end{equation}   
\end{proposition}

Several prior works have implicitly or explicitly leveraged the norm-controlling property of orthogonality~\cite{aoft2025,MiyatoKKY2018,Arjovsky2016}. Here, we provide a formal and rigorous proof to establish this principle.

\begin{proof}
We aim to find the maximum possible value of $\|W\|_F^2$ under the given constraint. This can be formulated as a constrained optimization problem,
\begin{equation}
    \begin{aligned} 
    \max_W \quad & \|W\|_F^2. \\
    \text{s.t.} \quad & \|W^\top W - I_d\|_F^2 \le \xi 
    \end{aligned}
\end{equation}

To solve this, we use the Singular Value Decomposition (SVD)~\cite{SVD1936} of $W$. Let $W = U \Sigma V^\top$, where $U \in \mathbb{R}^{m \times m}$ and $V \in \mathbb{R}^{d \times d}$ are orthogonal matrices, and $\Sigma \in \mathbb{R}^{m \times d}$ is a rectangular diagonal matrix with non-negative singular values $\{\sigma_1, \sigma_2, \dots, \sigma_d\}$ on its diagonal.

First, we rewrite the objective function in terms of the singular values. Because Frobenius norm is invariant under orthogonal transformations, we can get,
\begin{equation}
    \|W\|_F^2 = \|U \Sigma V^\top\|_F^2 = \|\Sigma\|_F^2 = \sum_{i=1}^d \sigma_i^2.
\end{equation}
Next, we rewrite the constraint. We have $W^\top W = (U\Sigma V^\top)^\top (U\Sigma V^\top) = V\Sigma^\top U^\top U \Sigma V^\top = V(\Sigma^\top \Sigma)V^\top$. Let $D = \Sigma^\top \Sigma$, which is a $d \times d$ diagonal matrix with diagonal elements $D_{ii} = \sigma_i^2$. Again, using the orthogonal invariance of the Frobenius norm,
\begin{equation}
\begin{aligned}
    \|W^\top W - I_d\|_F^2 &= \|V D V^\top - V I_d V^\top\|_F^2 \\
    &= \|V(D - I_d)V^\top\|_F^2 = \|D - I_d\|_F^2.
\end{aligned}
\end{equation}
Since $D - I_d$ is a diagonal matrix, its squared Frobenius norm is the sum of the squares of its diagonal elements,
\begin{equation}
    \|D - I_d\|_F^2 = \sum_{i=1}^d (\sigma_i^2 - 1)^2.
\end{equation}

The original problem is now equivalent to a simpler optimization problem over the squared singular values. Let $x_i = \sigma_i^2 \ge 0$,
\begin{equation}
    \begin{aligned} 
    \max \quad & \sum_{i=1}^d x_i. \\ 
    \text{s.t.} \quad & \sum_{i=1}^d (x_i - 1)^2 \le \xi 
    \end{aligned}
\end{equation}
To find the maximum, the constraint must be active, \textit{i.e.}, $\sum_{i=1}^d (x_i - 1)^2 = \xi$. We use the method of Lagrange multipliers~\cite{Lagrange2008}. The Lagrangian is,
\begin{equation}
    \mathcal{L}(\mathbf{x}, \lambda) = \sum_{i=1}^d x_i - \lambda \left( \sum_{i=1}^d (x_i - 1)^2 - \xi \right).
\end{equation}
Taking the partial derivative with respect to $x_j$ and setting it to zero,
\begin{equation}
    \frac{\partial \mathcal{L}}{\partial x_j} = 1 - \lambda \cdot 2(x_j - 1) = 0.
\end{equation}
\begin{equation}
    x_j - 1 = \frac{1}{2\lambda} \implies x_j = 1 + \frac{1}{2\lambda}.
\end{equation}
This shows that at the optimal point, all $x_j$ must be equal. Let $x_1 = x_2 = \dots = x_d = x^*$.

Substituting $x_i = x^*$ into the active constraint,
\begin{equation}
    \sum_{i=1}^d (x^* - 1)^2 = d(x^* - 1)^2 = \xi.
\end{equation}
Solving for $x^*$, we get,
\begin{equation}
    (x^* - 1)^2 = \frac{\xi}{d} \implies x^* - 1 = \pm\sqrt{\frac{\xi}{d}},
\end{equation}
\begin{equation}
    x^* = 1 \pm \sqrt{\frac{\xi}{d}}.
\end{equation}
To maximize the objective function $\sum x_i = d \cdot x^*$, we must choose the positive root,
\begin{equation}
    x^*_{\text{max}} = 1 + \sqrt{\frac{\xi}{d}}.
\end{equation}

Finally, the maximum value of the objective function is,
\begin{equation}
\begin{aligned}
    \max \|W\|_F^2 &= \sum_{i=1}^d x^*_{\text{max}} = d \cdot x^*_{\text{max}} \\
    &= d \left( 1 + \sqrt{\frac{\xi}{d}} \right) = d + \sqrt{d\xi}.
\end{aligned}
\end{equation}
This establishes the upper bound and completes the proof.
\end{proof}

\subsection{Detailed Proof of Angle Control Mechanism}\label{subsec:proof_angle_control}
This section provides the full proof for Step 3 of \Cref{theorem:ortho_updates}, showing that our orthogonal regularization statistically promotes orthogonality between different task vectors.

\subsubsection{\Cref{proposotion:P_matrix_property} and Detailed Proof}\label{subsec:proposition3}
\begin{proposition}\label{proposotion:P_matrix_property}
     Let $P \in \mathbb{R}^{d \times d}$ be a symmetric positive semi-definite matrix. If $\|P^2 - I_d\|_F \le \sqrt{\xi}$, then $\|P - I_d\|_F$ is also bounded, and specifically satisfies,
     \begin{equation}
         \|P - I_d\|_F \le \|P^2 - I_d\|_F.
     \end{equation}
\end{proposition}

\begin{proof}
Since $P$ is symmetric, it has an eigenvalue decomposition $P = U \Lambda U^\top$, where $U$ is an orthogonal matrix and $\Lambda$ is a diagonal matrix of non-negative eigenvalues $\lambda_1, \dots, \lambda_d \ge 0$. The Frobenius norm is invariant under orthogonal transformations. Thus, we can express the norms in terms of the eigenvalues,
\begin{equation}
\begin{aligned}
    \|P - I_d\|_F^2 &= \|U \Lambda U^\top - U I_d U^\top\|_F^2 \\
    &= \|U(\Lambda - I_d)U^\top\|_F^2\\
    &= \|\Lambda - I_d\|_F^2\\
    &= \sum_{i=1}^d (\lambda_i - 1)^2.
\end{aligned}
\end{equation}

Similarly, since $P^2 = (U \Lambda U^\top)(U \Lambda U^\top) = U \Lambda^2 U^\top$,
\begin{equation}
\begin{aligned}
    \|P^2 - I_d\|_F^2 &= \|U(\Lambda^2 - I_d)U^\top\|_F^2 \\
    &= \|\Lambda^2 - I_d\|_F^2 \\
    &=\sum_{i=1}^d (\lambda_i^2 - 1)^2
\end{aligned}
\end{equation}

Now we compare the terms for each eigenvalue,
\begin{equation}
    (\lambda_i^2 - 1)^2 = ((\lambda_i - 1)(\lambda_i + 1))^2 = (\lambda_i - 1)^2 \cdot (\lambda_i + 1)^2
\end{equation}
Since $P$ is positive semi-definite, $\lambda_i \ge 0$. This implies $\lambda_i + 1 \ge 1$, and therefore $(\lambda_i + 1)^2 \ge 1$. Multiplying both sides by the non-negative quantity $(\lambda_i - 1)^2$, we get
\begin{equation}
    (\lambda_i - 1)^2 \cdot (\lambda_i + 1)^2 \ge (\lambda_i - 1)^2 \cdot 1.
\end{equation}
This means $(\lambda_i^2 - 1)^2 \ge (\lambda_i - 1)^2$ for all $i$. Summing over all $i$,
\begin{equation}
    \sum_{i=1}^d (\lambda_i^2 - 1)^2 \ge \sum_{i=1}^d (\lambda_i - 1)^2.
\end{equation}
Substituting the norm expressions back, we have,
\begin{equation}
    \|P^2 - I_d\|_F^2 \ge \|P - I_d\|_F^2.
\end{equation}
Taking the square root of both sides yields the desired result,
\begin{equation}
    \|P - I_d\|_F \le \|P^2 - I_d\|_F.
\end{equation}
\end{proof}

\subsubsection{Proof of Angle Control}\label{subsubsec:proof_angle}
\begin{proof}
Our goal is to show that enforcing an internal orthogonal structure on the update matrices $\Delta W_t$ and $\Delta W_j$ statistically drives their corresponding task vectors $\tau_t$ and $\tau_j$ towards orthogonality. That is, $\mathbb{E}[|\cos\angle(\tau_t, \tau_j)|] \approx 0$. This is equivalent to showing that the inner product $\langle \tau_t, \tau_j \rangle$ is statistically concentrated around zero.

The total inner product is the sum of layer-wise inner products,
\begin{equation}
    \langle \tau_t, \tau_j \rangle = \sum_{l \in \text{Layers}} \langle \text{vec}(\Delta W_t^{(l)}), \text{vec}(\Delta W_j^{(l)}) \rangle.
\end{equation}
We analyze the inner product for a single layer, dropping the superscript $(l)$ for clarity: $\langle \text{vec}(\Delta W_t), \text{vec}(\Delta W_j) \rangle$.

Our $\mathcal{L}_{\text{ortho}} = \|\Delta W^\top \Delta W - I\|_F^2$ encourages the resulting update matrix $\Delta W^*$ to be approximately orthogonal, satisfying $\|(\Delta {W^*})^\top \Delta W^* - I\|_F^2 \le \xi$ for a small $\xi$.

Using Polar Decomposition~\cite{Polar1990}, any such matrix $\Delta W^*$ can be uniquely decomposed into $\Delta W^* = QP$, where $Q \in V_d(\mathbb{R}^m)$ is a matrix with orthonormal columns (an element of the Stiefel manifold) and $P = \sqrt{({\Delta W^*})^\top \Delta W^*}$ is a symmetric positive semi-definite matrix.

Substituting this relation into our regularization constraint, $\|(\Delta W^*)^\top \Delta W^* - I\|_F^2 \le \xi$, we have $\|P^2 - I\|_F^2 \le \xi$. By \Cref{proposotion:P_matrix_property}, this implies that $P$ is close to the identity matrix, \textit{i.e.}, $\|P - I\|_F$ is also small. We can thus write $P = I + E$, where $E = P - I$ is an ``error" matrix with a small Frobenius norm $\|E\|_F$.

Therefore, the update matrices for tasks $t$ and $j$ can be written as,
\begin{equation}
    \Delta W_t = Q_t + Q_t E_t,
\end{equation}
\begin{equation}
    \Delta W_j = Q_j + Q_j E_j,
\end{equation}
where $Q_t, Q_j$ are matrices on Stiefel manifold, and $E_t, E_j$ are error matrices with small norms controlled by $\xi$.

Now, we analyze the inner product of their vectorized forms,
\begin{equation}
\begin{aligned}
    &\langle \text{vec}(\Delta W_t), \text{vec}(\Delta W_j) \rangle \\
    &= \langle \text{vec}(Q_t + Q_t E_t), \text{vec}(Q_j + Q_j E_j) \rangle.
\end{aligned}
\end{equation}
Expanding this expression yields four terms,
\begin{equation}
\begin{aligned} = &\underbrace{\langle \text{vec}(Q_t), \text{vec}(Q_j) \rangle}_{\text{Main Term}} + \underbrace{\langle \text{vec}(Q_t), \text{vec}(Q_j E_j) \rangle}_{\text{Error Term 1}} \\ &+ \underbrace{\langle \text{vec}(Q_t E_t), \text{vec}(Q_j) \rangle}_{\text{Error Term 2}} + \underbrace{\langle \text{vec}(Q_t E_t), \text{vec}(Q_j E_j) \rangle}_{\text{Error Term 3}}. \end{aligned}
\end{equation}
We analyze the expectation of each term, assuming that the fine-tuning processes for distinct tasks $t$ and $j$ result in independently sampled matrices from the space of approximately orthogonal matrices.

\noindent\textbf{Main Term.} $Q_t$ and $Q_j$ are independent, random matrices from the Stiefel manifold $V_d(\mathbb{R}^m)$. According to \Cref{lemma:stiefel} (proven in Appendix~\ref{subsec:Lemma2-stiefel}), the expected value of their inner product is zero,
\begin{equation}
    \mathbb{E}[\langle \text{vec}(Q_t), \text{vec}(Q_j) \rangle] = 0.
\end{equation}
And, \Cref{lemma:stiefel} states that the probability distribution of this inner product is sharply concentrated around zero.

\noindent\textbf{Error Terms.} We bound the magnitude of the error terms using the Cauchy-Schwarz inequality. %

For Error Term 1,
\begin{equation}
    |\langle \text{vec}(Q_t), \text{vec}(Q_j E_j) \rangle| \le \|\text{vec}(Q_t)\|_2 \cdot \|\text{vec}(Q_j E_j)\|_2.
\end{equation}
Since $Q_t$ has $d$ orthonormal columns, $\|\text{vec}(Q_t)\|_2^2 = \|Q_t\|_F^2 = d$. Since $Q_j$ is an orthogonal transformation, $\|\text{vec}(Q_j E_j)\|_2 = \|Q_j E_j\|_F = \|E_j\|_F$. Thus, the term is bounded by $\sqrt{d} \cdot \|E_j\|_F$. As $\|E_j\|_F$ is a small value controlled by the regularizer, this error term is negligible.

Error Term 2 is similarly bounded by $\sqrt{d} \cdot \|E_t\|_F$ and is also negligible.

Error Term 3 is bounded by $\|\text{vec}(Q_t E_t)\|_2 \cdot \|\text{vec}(Q_j E_j)\|_2 = \|E_t\|_F \cdot \|E_j\|_F$, which is a second-order small term and even more negligible.

Since the main term has an expected value of zero and the error terms are negligible, the expected inner product for a single layer is approximately zero.
\begin{equation}
    \mathbb{E}[\langle \text{vec}(\Delta W_t), \text{vec}(\Delta W_j) \rangle] \approx 0.
\end{equation}

By linearity of expectation, the expected inner product of the full task vectors is also approximately zero,
\begin{equation}
    \mathbb{E}[\langle \tau_t, \tau_j \rangle] = \sum_l \mathbb{E}[\langle \text{vec}(\Delta W_t^{(l)}), \text{vec}(\Delta W_j^{(l)}) \rangle] \approx 0.
\end{equation}

Because the distribution of the main term at each layer is sharply peaked at zero, the distribution of the sum (the total inner product) will also be sharply peaked at zero. This implies that $\tau_t$ and $\tau_j$ are statistically very likely to be orthogonal, and thus $\mathbb{E}[|\cos\angle(\tau_t, \tau_j)|] \approx 0$. This completes the proof of the angle control mechanism.
\end{proof}

\subsubsection{Lemma 2 and Detailed Proof: Stiefel Manifold Inner Product}\label{subsec:Lemma2-stiefel}
\begin{lemma}\label{lemma:stiefel}
Let $A$ and $B$ be two matrices independently and uniformly sampled from the Stiefel manifold $V_d(\mathbb{R}^m)$~\cite{Stiefel2009} (the set of $m \times d$ matrices with orthonormal columns). Let $Z = \langle \text{vec}(A), \text{vec}(B) \rangle$. Then,

(1) The expected value of the inner product is zero: $\mathbb{E}[Z] = 0$.

(2)The probability distribution of $Z$ is sharply concentrated around 0.
\end{lemma}
\begin{proof}
\textbf{Part 1: Proof of Zero Expectation.}

The inner product can be written as the trace of the matrix product: $Z = \text{Tr}(A^\top B)$. Due to the independence of $A$ and $B$, the expectation of the product is the product of expectations,
\begin{equation}
    \mathbb{E}[Z] = \mathbb{E}_A[\mathbb{E}_B[\text{Tr}(A^\top B) | A]] = \mathbb{E}_A[\text{Tr}(A^\top \mathbb{E}_B[B])].
\end{equation}
Let's compute $\mathbb{E}[B]$. The distribution of $B$ is the uniform (Haar) measure on $V_d(\mathbb{R}^m)$. This distribution is invariant under left-multiplication by any orthogonal matrix $Q \in O(m)$, where $O(m)$ is the group of $m \times m$ orthogonal matrices. This means that for any $Q \in O(m)$, the random matrix $QB$ has the same distribution as $B$. Therefore,
\begin{equation}
    \mathbb{E}[B] = \mathbb{E}[QB] = Q\mathbb{E}[B].
\end{equation}
This equality must hold for all $Q \in O(m)$. Let's consider a specific reflection matrix $Q$ that negates the first coordinate, \textit{e.g.}, $Q = \text{diag}(-1, 1, \dots, 1)$. If the first row of $\mathbb{E}[B]$ were a non-zero vector $\mathbf{r}$, then the first row of $Q\mathbb{E}[B]$ would be $-\mathbf{r}$. The equality $\mathbb{E}[B] = Q\mathbb{E}[B]$ would imply $\mathbf{r} = -\mathbf{r}$, which is only possible if $\mathbf{r} = \mathbf{0}$. This logic applies to every row by choosing appropriate reflection matrices. Therefore, the only matrix that satisfies this condition for all $Q \in O(m)$ is the zero matrix.
\begin{equation}
    \mathbb{E}[B] = \mathbf{0}.
\end{equation}
Substituting this back into the expectation for $Z$, we get,
\begin{equation}
    \mathbb{E}[Z] = \text{Tr}(\mathbb{E}[A^\top] \cdot \mathbf{0}) = 0.
\end{equation}
This proves the first part of the lemma.

\textbf{Part 2: Proof of Concentration around Zero.}

This is a geometric argument. The vectors $\text{vec}(A)$ and $\text{vec}(B)$ are not arbitrary vectors in $\mathbb{R}^{m\times d}$. They are constrained to lie on the submanifold $\text{vec}(V_d(\mathbb{R}^m))$. The condition $A^\top A = I_d$ imposes $\frac{d(d+1)}{2}$ independent constraints on the elements of $A$. This means the dimension of the Stiefel manifold $V_d(\mathbb{R}^m)$ is $\dim(V) = md - \frac{d(d+1)}{2}$.

The co-dimension of this submanifold within the ambient space $\mathbb{R}^{m \times d}$ is $\frac{d(d+1)}{2}$, which is positive for $d \ge 1$. The condition for orthogonality, $\langle\text{vec}(A), \text{vec}(B)\rangle = 0$, defines a hyperplane in the product space. The probability density of $Z$ at a value $z_0$ is proportional to the ``volume" of the surface defined by $\langle\text{vec}(A), \text{vec}(B)\rangle = z_0$ on the product manifold $V_d(\mathbb{R}^m) \times V_d(\mathbb{R}^m)$.

Intuitively, because the vectors are already living in a lower-dimensional space due to the internal orthogonality constraints, the additional constraint of being orthogonal to another such vector is ``easier" to satisfy. The intersection of the hyperplane $\langle\mathbf{a}, \mathbf{b}\rangle=0$ with the product manifold is larger than its intersection with the product of two spheres of the same dimension. This geometric fact leads to a higher probability density at $Z=0$, creating a sharp peak in the distribution. This indicates that two random matrices from Stiefel manifold are much more likely to be nearly orthogonal than two completely random unit vectors in $\mathbb{R}^{m\times d}$.
\end{proof}

\section{Comparative Analysis with TTA}\label{sec:tta}
\subsection{Theoretical Connection}\label{subsec:theoretical_tta}
In this section, we establish a theoretical connection between our proposed method (OrthoReg) and Tangent Task Arithmetic (TTA)~\cite{tta2023}. We demonstrate that both methods, despite their different implementations, derive their effectiveness from a shared underlying mechanism: promoting orthogonality between different task vectors (\textit{i.e.}, $\langle \tau_t, \tau_j \rangle \approx 0$ for $t \neq j$). This inter-task vector orthogonality is a key driver for achieving weight disentanglement.

As proven in \Cref{theorem:ortho_updates} (specifically, the Angle Control mechanism in Appendix~\ref{subsubsec:proof_angle}), our OrthoReg achieves this goal explicitly. By enforcing an internal orthogonal structure on each update matrix $\Delta W$, it statistically drives the resulting full task vectors towards orthogonality.

In contrast, TTA achieves this goal implicitly by leveraging the geometric properties of the pre-trained model's Neural Tangent Kernel (NTK). We now provide a detailed derivation to formalize this connection.

\begin{table*}[t]
\centering
\caption{
    Computational cost comparison on the Cars dataset using a ViT-L-14 model. The table highlights the efficiency of OrthoReg. The final column shows the Absolute Accuracy from the task addition benchmark (as seen in Table 1 of the main paper). While applying OrthoReg to Non-linear Fine-tuning (Non-lin.~FT) achieves performance that is superior to Tangent Task Arithmetic (TTA) and significantly better than the baseline Non-lin.~FT, this table further demonstrates its superior computational efficiency. As seen, TTA incurs substantial overhead in both training time and memory, whereas OrthoReg adds only a modest cost to the baseline. The colored cells visually emphasize the significant difference in computational cost between TTA and our proposed method.
}
\label{tab:computational_cost}
\begin{tabular}{lccccc}
\toprule
\textbf{Fine-tuning Method} & {\textbf{Total}} & {\textbf{Trainable}} & {\textbf{Training}} & {\textbf{Peak GPU}} & {\textbf{Abs. Acc.}} \\
& {\textbf{Params (M)}} & {\textbf{Params (M)}} & {\textbf{Time (Min)}} & {\textbf{Mem (MB)}} & {\textbf{(\%)}} \\
\midrule
\multicolumn{6}{l}{\textit{Full Fine-tuning Methods}} \\
\addlinespace[0.3em]
Non-lin.~FT~\cite{ta2023} (Baseline)     & 342.56 & 342.56 & 158.21 & 42589.22 & 84.07 \\
TTA~\cite{tta2023} (Linearized)           & \cellcolor{DarkRedHighlight}685.12 & 342.56 & \cellcolor{DarkRedHighlight}280.86 & \cellcolor{DarkRedHighlight}68031.34 & 86.19 \\
Non-lin.~FT + OrthoReg (ours)   & \cellcolor{LightRedHighlight}342.56 & 342.56 & \cellcolor{LightRedHighlight}177.04 & \cellcolor{LightRedHighlight}44500.27 & 88.23  \\
\midrule
\multicolumn{6}{l}{\textit{Parameter-Efficient Fine-tuning (Attention-Only)}} \\
\addlinespace[0.3em]
ATT-FT~\cite{linearatt2025}                     & 342.56 & 100.66  & 126.28 & 36591.06 & 87.81 \\
ATT-FT + OrthoReg (ours)        & 342.56 & 100.66  & 132.96 & 36976.50 & 90.41 \\
\bottomrule
\end{tabular}
\end{table*}

TTA operates by performing fine-tuning in the tangent space of the pre-trained model $\theta_0$. The model's output is approximated by its first-order Taylor expansion,
\begin{equation}
    f(x; \theta_0 + \tau) \approx f(x; \theta_0) + \tau^\top J(x),
\end{equation}
where $J(x) = \nabla_\theta f(x; \theta_0)$ is the Jacobian. The optimization is performed over the task vector $\tau$ directly. For a task $t$ with data $\{(x_i, y_i)\}_{i=1}^{N_t}$ from domain $\mathcal{D}_t$, the TTA objective can be formulated as a regularized empirical risk minimization problem, for instance, using a mean-squared error loss:
$$
\min_{\tau_t} \frac{1}{N_t} \sum_{i=1}^{N_t} \left\| (f(x_i; \theta_0) + \tau_t^\top J(x_i)) - y_i \right\|_2^2 + \lambda \|\tau_t\|_2^2.
$$
This is a linear ridge regression problem in the variable $\tau_t$. According to the Representer Theorem, the optimal solution $\tau_t^*$ must lie in the subspace spanned by the Jacobians of the training data points. Therefore, $\tau_t^*$ can be expressed as a linear combination of these Jacobians,
\begin{equation}
    \tau_t^* = \sum_{i=1}^{N_t} \alpha_i J(x_i),
\end{equation}
where $\{\alpha_i\}$ are scalar coefficients determined by the optimization.

\begin{figure*}
  \centering
   \includegraphics[width=0.98\linewidth]{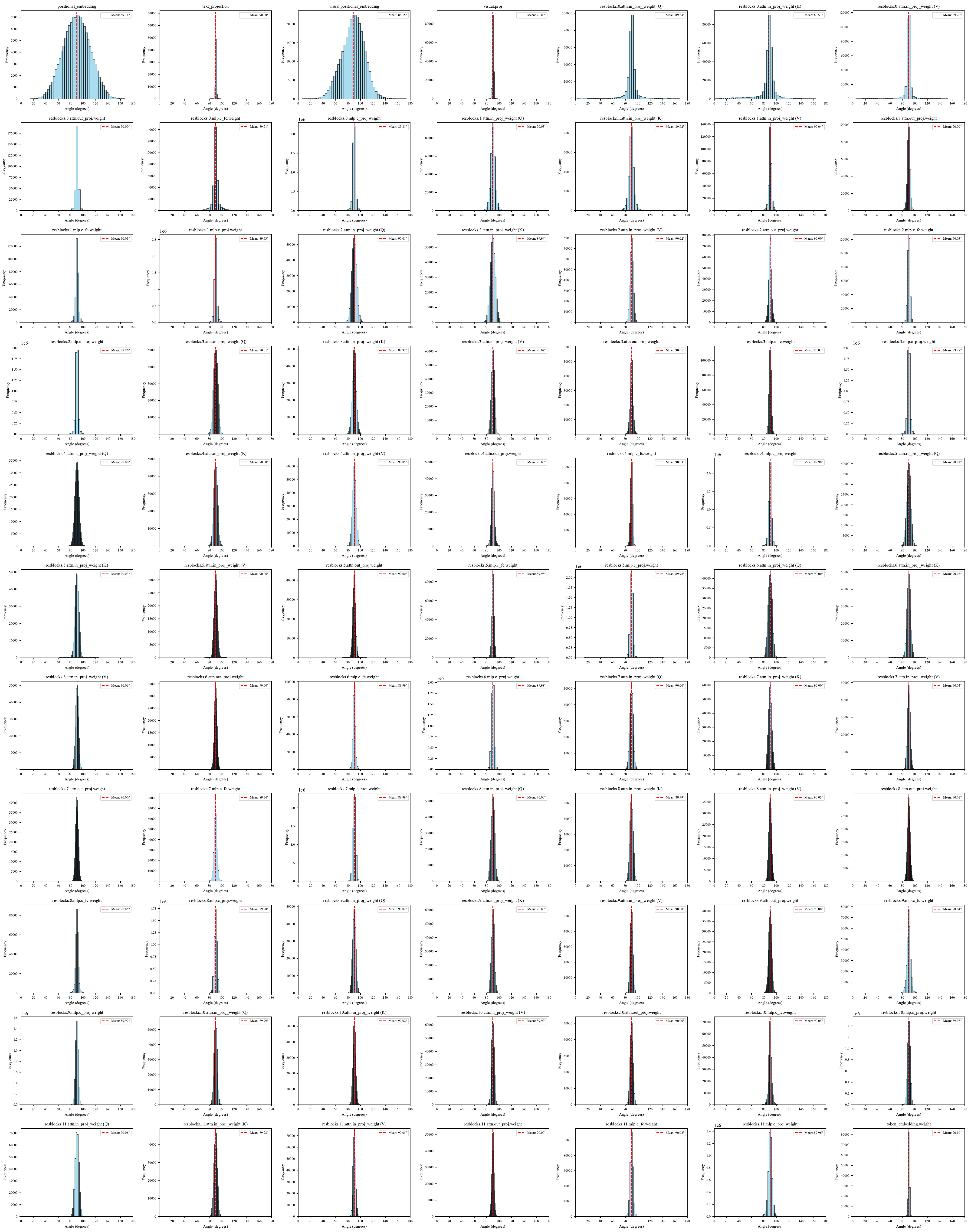}
   \caption{Angle distributions of weight matrix columns for all layers in ViT-B/16. Each subplot displays a histogram of the angles (in degrees) between all pairs of column vectors for a specific weight matrix. The red dashed line indicates the $90^{\circ}$ point of perfect orthogonality. The plots are ordered sequentially, starting with the embedding layers, followed by the 12 transformer blocks.}
   \label{fig:angle_all}
\end{figure*}

Now, consider the inner product between two task vectors, $\tau_t^*$ and $\tau_j^*$, obtained by applying TTA to two different tasks, $t$ and $j$,
\begin{equation}
    \langle \tau_t^*, \tau_j^* \rangle = \left\langle \sum_{i=1}^{N_t} \alpha_i J(x_i), \sum_{k=1}^{N_j} \beta_k J(x_k) \right\rangle,
\end{equation}
where $\{x_i\} \subset \mathcal{D}_t$ and $\{x_k\} \subset \mathcal{D}_j$. By linearity of the inner product, this becomes,
\begin{equation}
    \langle \tau_t^*, \tau_j^* \rangle = \sum_{i=1}^{N_t} \sum_{k=1}^{N_j} \alpha_i \beta_k \langle J(x_i), J(x_k) \rangle.
\end{equation}

The term $\langle J(x_i), J(x_k) \rangle$ is precisely the definition of the Neural Tangent Kernel (NTK) evaluated at the pair of inputs $(x_i, x_k)$,
\begin{equation}
    k_{\text{NTK}}(x_i, x_k) = J(x_i)^\top J(x_k).
\end{equation}
Therefore, the inner product of the task vectors is a weighted sum of NTK values between the data points of the two tasks,
\begin{equation}
    \langle \tau_t^*, \tau_j^* \rangle = \sum_{i=1}^{N_t} \sum_{k=1}^{N_j} \alpha_i \beta_k \, k_{\text{NTK}}(x_i, x_k).
\end{equation}

A central empirical finding of the TTA paper~\cite{tta2023} is that the NTK of large pre-trained models, such as CLIP, exhibits a strong localization property. This property means that the kernel function value is significant only when both inputs are from the same task domain and decays rapidly to near-zero when the inputs are from different, unrelated task domains. Formally, for distinct tasks $t \neq j$,
\begin{equation}
    k_{\text{NTK}}(x_i, x_k) \approx 0 \quad \text{for all } x_i \in \mathcal{D}_t \text{ and } x_k \in \mathcal{D}_j.
\end{equation}
Substituting this result into our expression for the inner product, we find that every term in the double summation is approximately zero. Consequently, the entire sum is approximately zero,
\begin{equation}
    \langle \tau_t^*, \tau_j^* \rangle \approx \sum_{i=1}^{N_t} \sum_{k=1}^{N_j} \alpha_i \beta_k \cdot 0 \approx 0.
\end{equation}

This derivation shows that TTA's effectiveness in promoting weight disentanglement stems from its ability to implicitly construct task vectors that are nearly orthogonal to each other. This orthogonality is not an explicit constraint but rather an emergent property arising from the localized structure of the pre-trained model's NTK.

Our analysis thus unifies our method and TTA under a common principle: inter-task vector orthogonality is a core mechanism for achieving weight disentanglement. Our OrthoReg method provides a more direct, explicit to enforce this geometric property, which explains its ability to further enhance the performance of TTA and other task arithmetic methods, as demonstrated in our experiments.

\begin{figure*}[htbp]
  \centering
  \begin{subfigure}{\linewidth}
    \centering
    \includegraphics[width=\linewidth]{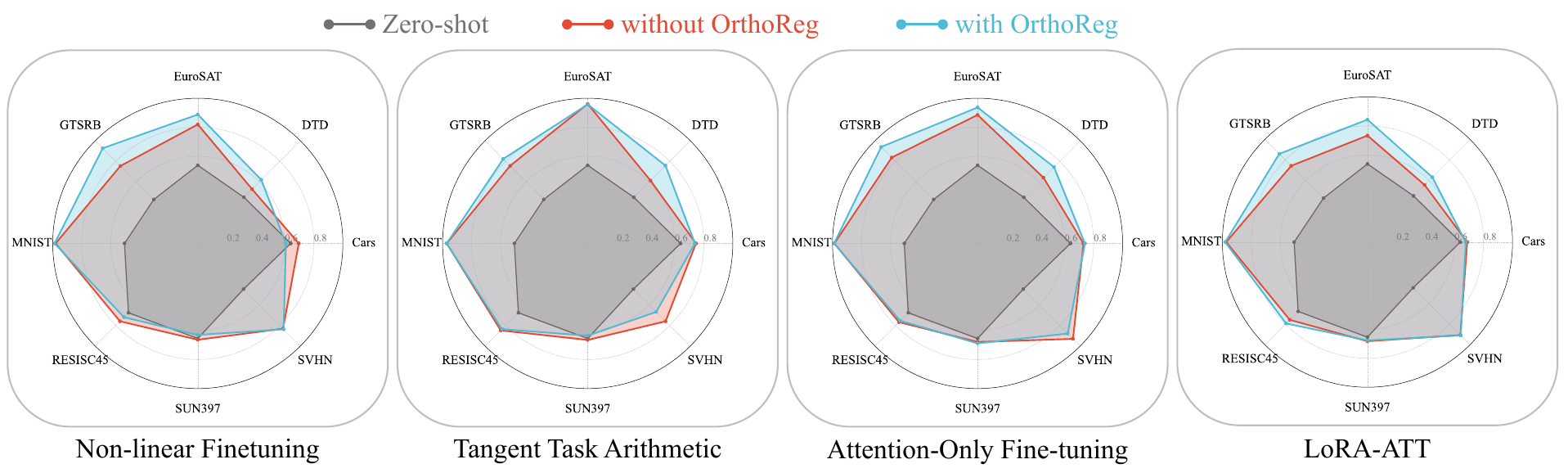}
    \caption{ViT-B-16}
    \label{subfig:radar-16}
  \end{subfigure}
  \vspace{1em} %
  \begin{subfigure}{\linewidth}
    \centering
    \includegraphics[width=\linewidth]{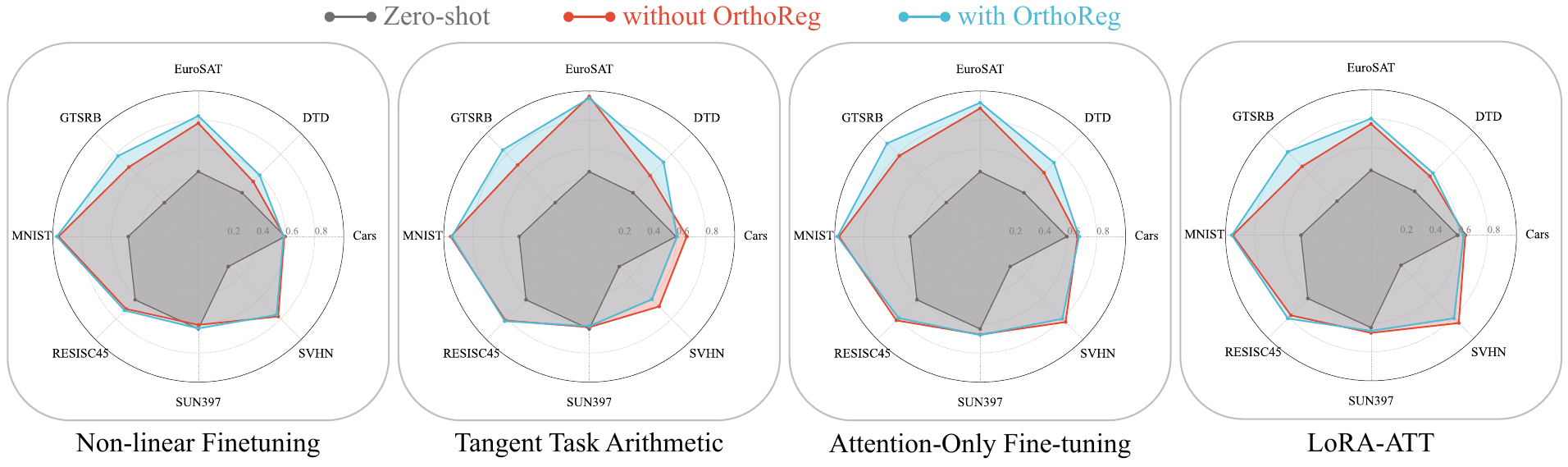}
    \caption{ViT-B-32}
    \label{subfig:radar-32}
  \end{subfigure}
  \vspace{1em} %
  \begin{subfigure}{\linewidth}
    \centering
    \includegraphics[width=\linewidth]{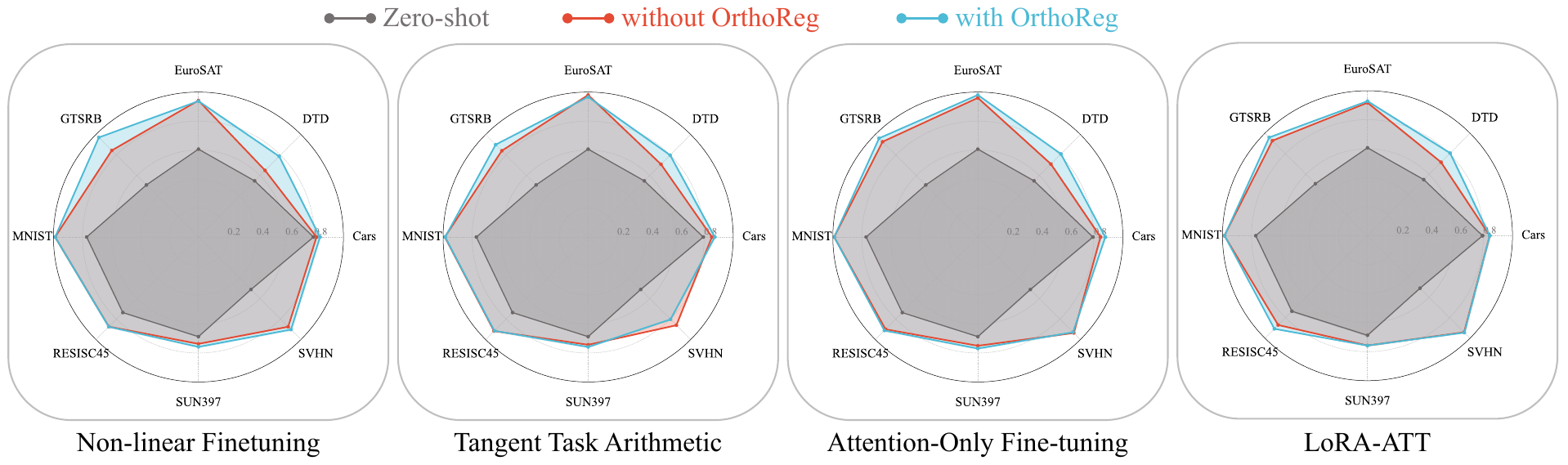}
    \caption{ViT-L-14}
    \label{subfig:radar-14}
  \end{subfigure}
  \caption{The accuracy of merged models across the eight benchmark tasks for different ViT architectures. Each subplot shows the performance for a specific baseline method: zero-shot (gray), the baseline's merged model (red), and the baseline enhanced with our orthogonal regularization (blue). The rows correspond to models: (a) ViT-B-16, (b) ViT-B-32, and (c) ViT-L-14.}
  \label{fig:radar-all-architectures}
\end{figure*}

\subsection{Experimental Performance Comparison and Analysis}
As established in Section~4.4 and Appendix~\ref{subsec:theoretical_tta}, both our OrthoReg method and Tangent Task Arithmetic (TTA)~\cite{tta2023} succeed by promoting inter-task vector orthogonality. However, we posited that OrthoReg offers a more direct, efficient, and scalable approach by avoiding the costly Jacobian computations inherent to TTA. This section provides an empirical analysis to validate this claim by comparing the computational costs, specifically training time and peak GPU memory usage of TTA against standard fine-tuning methods enhanced with our OrthoReg regularizer.

\noindent \textbf{Experimental Setup.} We conduct a controlled experiment on the Cars dataset~\cite{Krause2013CollectingAL} using the ViT-L-14 model architecture. We measure the wall-clock training time and peak GPU memory consumption for a single fine-tuning run.

\begin{table*}[t]
\small
\centering
\caption{The minimum average Target Accuracy (Tar.Acc.) achievable while maintaining at least 90\% of the zero-shot accuracy on the ImageNet control task (Con.Acc.). Our proposed orthogonal regularization (+OrthoReg) shows a consistent and significant improvement in forgetting the target task. An asterisk (*) denotes the best (lowest) target accuracy for each model architecture.}
\label{tab:clip_task_negation}
\begin{tabular}{ccccccc}
\toprule
\multirow{2}{*}{\textbf{Method}} &
\multicolumn{2}{c}{\textbf{ViT-B-32, 8 tasks}} &
\multicolumn{2}{c}{\textbf{ViT-B-16, 8 tasks}} &
\multicolumn{2}{c}{\textbf{ViT-L-14, 8 tasks}} \\
\cmidrule(lr){2-3} \cmidrule(lr){4-5} \cmidrule(lr){6-7}
& \textbf{Tar.Acc.}\textbf{(\textdownarrow)} & \textbf{Con.Acc.} \textbf{(\textuparrow)}
& \textbf{Tar.Acc.}\textbf{(\textdownarrow)} & \textbf{Con.Acc.} \textbf{(\textuparrow)}
& \textbf{Tar.Acc.}\textbf{(\textdownarrow)} & \textbf{Con.Acc.} \textbf{(\textuparrow)} \\
\midrule
zero-shot & 47.74 & 66.70 & 54.22 & 68.34 & 64.54 & 77.44 \\
\midrule
Non-linear Finetuning~\cite{ta2023}  & 17.34 & 60.80 & 14.92 & 63.63 & 13.51 & 72.51 \\
Non-lin. FT\textbf{+OrthoReg} (ours) & \textbf{14.14} & 60.84 & \textbf{13.78} & 65.69 & \textbf{12.69} & 74.17 \\
\rowcolor{rowblue}
\rowcolor{rowblue}
\quad  $\Delta$ & {\color{darkblue}-3.20} & {\color{darkblue}+0.04} & {\color{darkblue}-1.14} & {\color{darkblue}+2.06} & {\color{darkblue}-0.82} & {\color{darkblue}+1.66} \\
\midrule
Tangent Task Arithmetic~\cite{tta2023}  & 7.36 & 62.08 & 6.68 & 65.49 & 5.07 & 72.51 \\
TTA\textbf{+OrthoReg} (ours) & {\textbf{6.66}} & 62.19 & {\textbf{4.77}} & 65.13 & {\textbf{3.83}} & 72.87 \\
\rowcolor{rowblue}
\quad  $\Delta$ & {\color{darkblue}-0.70} & {\color{darkblue}+0.11} & {\color{darkblue}-1.91} & {\color{darkblue}-0.36} & {\color{darkblue}-1.24} & {\color{darkblue}+0.36} \\
\midrule
Attention-Only Fine-tuning~\cite{linearatt2025}  & 19.11 & 64.82 & 19.01 & 67.67 & 24.85 & 76.42 \\
ATT-FT\textbf{+OrthoReg} (ours) & \textbf{10.75} & 62.18 & \textbf{10.63} & 64.10 & \textbf{11.47} & 73.17 \\
\rowcolor{rowblue}
\quad  $\Delta$ & {\color{darkblue}-8.36} & {\color{darkblue}-2.64} & {\color{darkblue}-8.38} & {\color{darkblue}-3.57} & {\color{darkblue}-13.38} & {\color{darkblue}-3.25} \\
\midrule
LoRA-ATT & 16.85 & 63.23 & 19.44 & 67.28 & 21.23 & 75.41 \\
LoRA-ATT\textbf{+OrthoReg} (ours) & \textbf{14.59} & 61.68 & \textbf{17.25} & 67.08 & \textbf{10.10} & 72.19 \\
\rowcolor{rowblue}
\quad  $\Delta$ & {\color{darkblue}-2.26} & {\color{darkblue}-1.55} & {\color{darkblue}-2.19} & {\color{darkblue}-0.20} & {\color{darkblue}-11.13} & {\color{darkblue}-3.22} \\
\bottomrule
\end{tabular}\label{tab:task_negation2}
\end{table*}

\begin{table*}[t]
\small
\centering
\caption{The minimum average Target Accuracy (Tar.Acc.) achievable while maintaining at least 80\% of the zero-shot accuracy on the ImageNet control task (Con.Acc.). Our proposed orthogonal regularization (+OrthoReg) shows a consistent and significant improvement in forgetting the target task. An asterisk (*) denotes the best (lowest) target accuracy for each model architecture.}
\label{tab:clip_task_negation}
\begin{tabular}{ccccccc}
\toprule
\multirow{2}{*}{\textbf{Method}} &
\multicolumn{2}{c}{\textbf{ViT-B-32, 8 tasks}} &
\multicolumn{2}{c}{\textbf{ViT-B-16, 8 tasks}} &
\multicolumn{2}{c}{\textbf{ViT-L-14, 8 tasks}} \\
\cmidrule(lr){2-3} \cmidrule(lr){4-5} \cmidrule(lr){6-7}
& \textbf{Tar.Acc.}\textbf{(\textdownarrow)} & \textbf{Con.Acc.} \textbf{(\textuparrow)}
& \textbf{Tar.Acc.}\textbf{(\textdownarrow)} & \textbf{Con.Acc.} \textbf{(\textuparrow)}
& \textbf{Tar.Acc.}\textbf{(\textdownarrow)} & \textbf{Con.Acc.} \textbf{(\textuparrow)} \\
\midrule
zero-shot & 47.74 & 66.70 & 54.22 & 68.34 & 64.54 & 77.44 \\
\midrule
Non-linear Finetuning~\cite{ta2023}  & 11.97 & 54.43 & 11.65 & 59.20 & 12.67 & 70.59 \\
Non-lin. FT\textbf{+OrthoReg} (ours) & \textbf{10.24} & 57.06 & \textbf{10.40} & 61.39 & \textbf{9.30} & 72.33 \\
\rowcolor{rowblue}
\quad  $\Delta$ & {\color{darkblue}-1.73} & {\color{darkblue}+2.63} & {\color{darkblue}-1.25} & {\color{darkblue}+2.19} & {\color{darkblue}-3.37} & {\color{darkblue}+1.74} \\
\midrule
Tangent Task Arithmetic~\cite{tta2023}  & 5.70 & 60.76 & 5.61 & 64.53 & 2.84 & 70.81 \\
TTA\textbf{+OrthoReg} (ours) & {\textbf{3.26}} & 59.26 & {\textbf{2.10}} & 62.61 & {\textbf{1.86}} & 70.23 \\
\rowcolor{rowblue}
\quad  $\Delta$ & {\color{darkblue}-2.44} & {\color{darkblue}-1.50} & {\color{darkblue}-3.51} & {\color{darkblue}-1.92} & {\color{darkblue}-0.98} & {\color{darkblue}-0.58} \\
\midrule
Attention-Only Fine-tuning~\cite{linearatt2025}  & 19.11 & 64.82 & 19.01 & 67.67 & 24.85 & 76.42 \\
ATT-FT\textbf{+OrthoReg} (ours) & \textbf{7.23} & 58.38 & \textbf{8.08} & 61.21 & \textbf{8.12} & 68.46 \\
\rowcolor{rowblue}
\quad  $\Delta$ & {\color{darkblue}-11.88} & {\color{darkblue}-6.44} & {\color{darkblue}-10.93} & {\color{darkblue}-6.46} & {\color{darkblue}-16.73} & {\color{darkblue}-7.96} \\
\midrule
LoRA-ATT & 15.58 & 62.4 & 15.83 & 62.40 & 21.23 & 75.41 \\
LoRA-ATT\textbf{+OrthoReg} (ours) & \textbf{11.00} & 58.47 & \textbf{9.19} & 60.41 & \textbf{7.68} & 69.83 \\
\rowcolor{rowblue}
\quad  $\Delta$ & {\color{darkblue}-4.58} & {\color{darkblue}-3.93} & {\color{darkblue}-6.64} & {\color{darkblue}-1.99} & {\color{darkblue}-13.55} & {\color{darkblue}-5.58} \\
\bottomrule
\end{tabular}\label{tab:task_negation3}
\end{table*}

\noindent \textbf{Results and Analysis.} The results, summarized in \Cref{tab:computational_cost} are organized to highlight the efficiency trade-offs between different full-parameter fine-tuning strategies and their parameter-efficient counterparts.

The primary comparison focuses on the full fine-tuning methods. Standard Non-linear Fine-tuning (Non-lin.~FT) serves as our baseline, completing training in 158.21 minutes and consuming 42589.22 MB of peak GPU memory. In stark contrast, TTA~\cite{tta2023}, which operates on a linearized model, is substantially more resource-intensive. It requires 280.86 minutes (a 77.5\% increase in time) and 68031.34 MB of memory (a 59.7\% increase), confirming that its reliance on Jacobian computations imposes a significant computational burden.

Our proposed OrthoReg, when applied to Non-lin.~FT, introduces only a moderate overhead for its regularization calculations, resulting in a total cost of 177.04 minutes and 44500.27 MB of memory during the training phase. Crucially, this is significantly more efficient than TTA in both time and memory, while achieving superior or comparable task-addition performance as shown in the main text and the last column of \Cref{tab:computational_cost} (\textit{e.g.}, for ViT-L-14, Non-lin. FT + OrthoReg achieves 88.23\% Abs.Acc. vs. TTA's 86.19\%). This demonstrates that OrthoReg provides a more efficient path to enforcing the properties that benefit task arithmetic.

This efficiency advantage is also evident in the parameter-efficient setting. As shown in the lower section of \Cref{tab:computational_cost}, applying OrthoReg to ATT-FT baseline results in only a minimal increase in computational cost. The training time rises modestly from 126.28 to 132.96 minutes, and peak memory usage increases marginally from 36591.06 MB to 36976.50 MB. However, the performance increases considerably from 87.81\% to 90.41\%. This demonstrates that the substantial performance improvements gained from OrthoReg come at a very low computational price, further highlighting its practicality.

In conclusion, these experiments provide strong empirical evidence that OrthoReg achieves the goal of promoting task vector orthogonality more efficiently than TTA. This efficiency, combined with the superior performance demonstrated in our main results, establishes OrthoReg as a more effective and accessible tool for reliable task arithmetic.

\begin{figure*}[htbp] 
    \centering 
    \begin{subfigure}{0.24\textwidth}
        \centering
        \includegraphics[width=\linewidth]{img/hotmap_pdf/VIT-B-16/standard.pdf}
        \caption{Non-lin. FT}
    \end{subfigure}
    \hfill 
    \begin{subfigure}{0.24\textwidth}
        \centering
        \includegraphics[width=\linewidth]{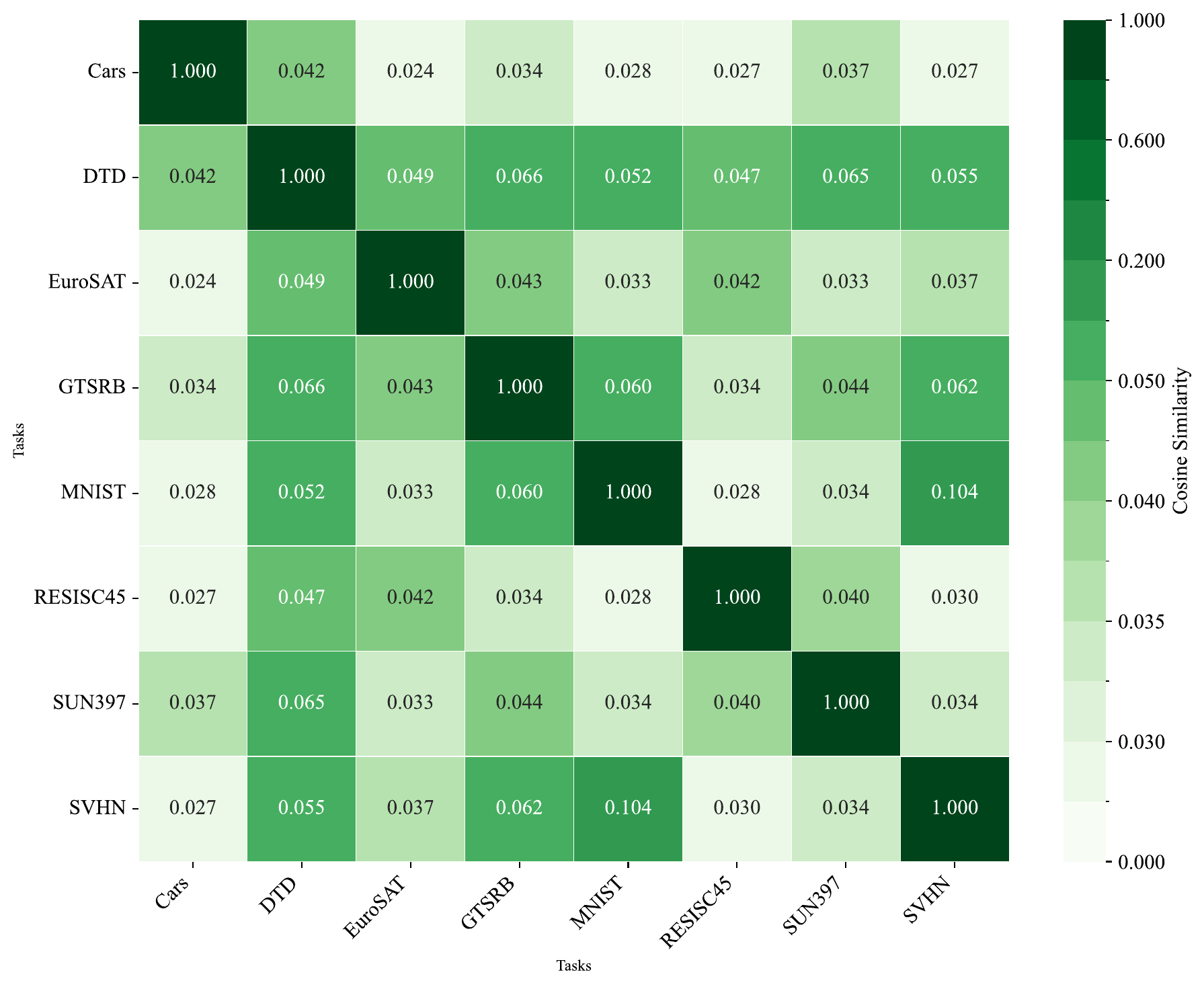}
        \caption{TTA}
    \end{subfigure}
    \hfill
    \begin{subfigure}{0.24\textwidth}
        \centering
        \includegraphics[width=\linewidth]{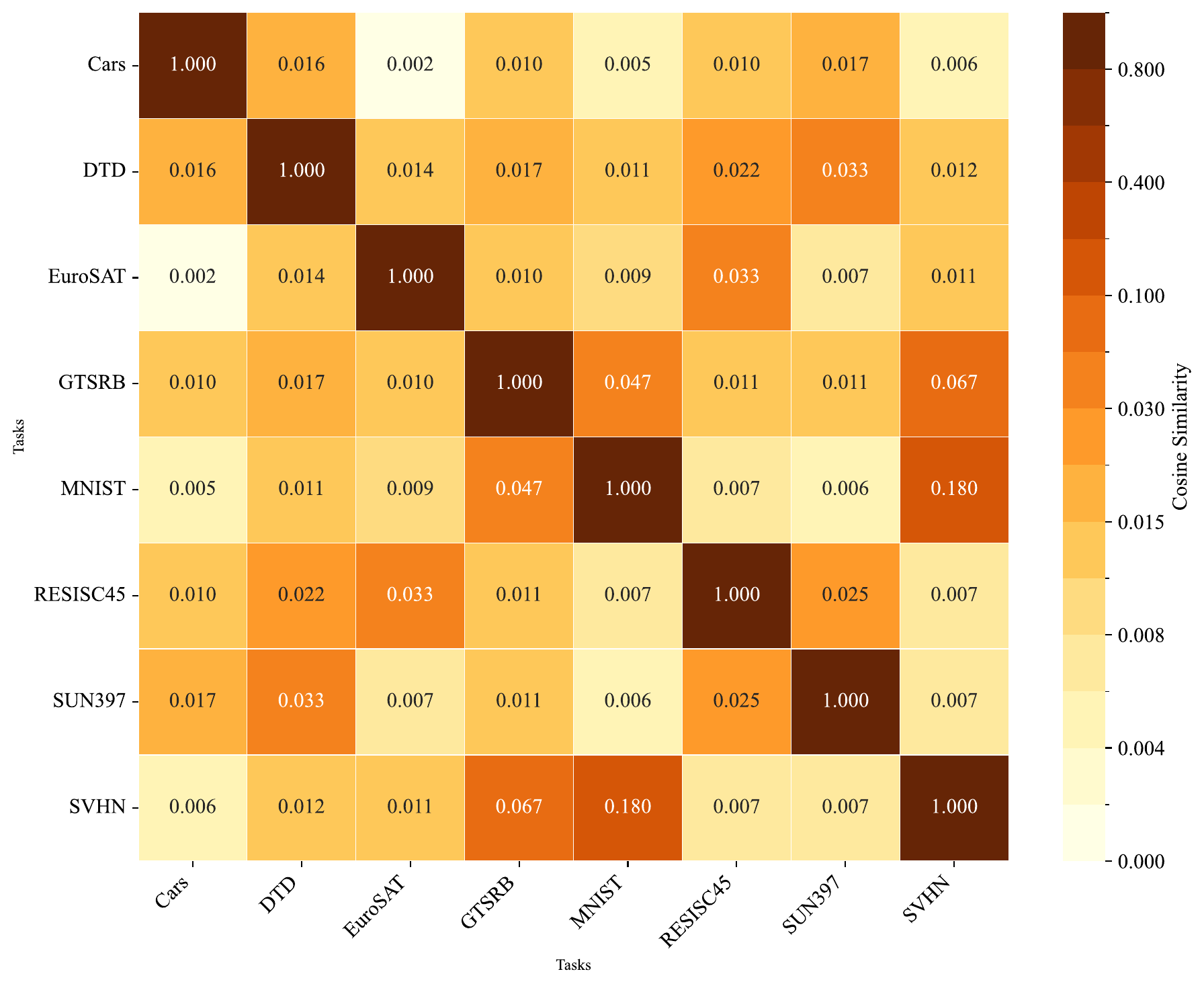} 
        \caption{ATT-FT}
    \end{subfigure}
    \hfill
    \begin{subfigure}{0.24\textwidth}
        \centering
        \includegraphics[width=\linewidth]{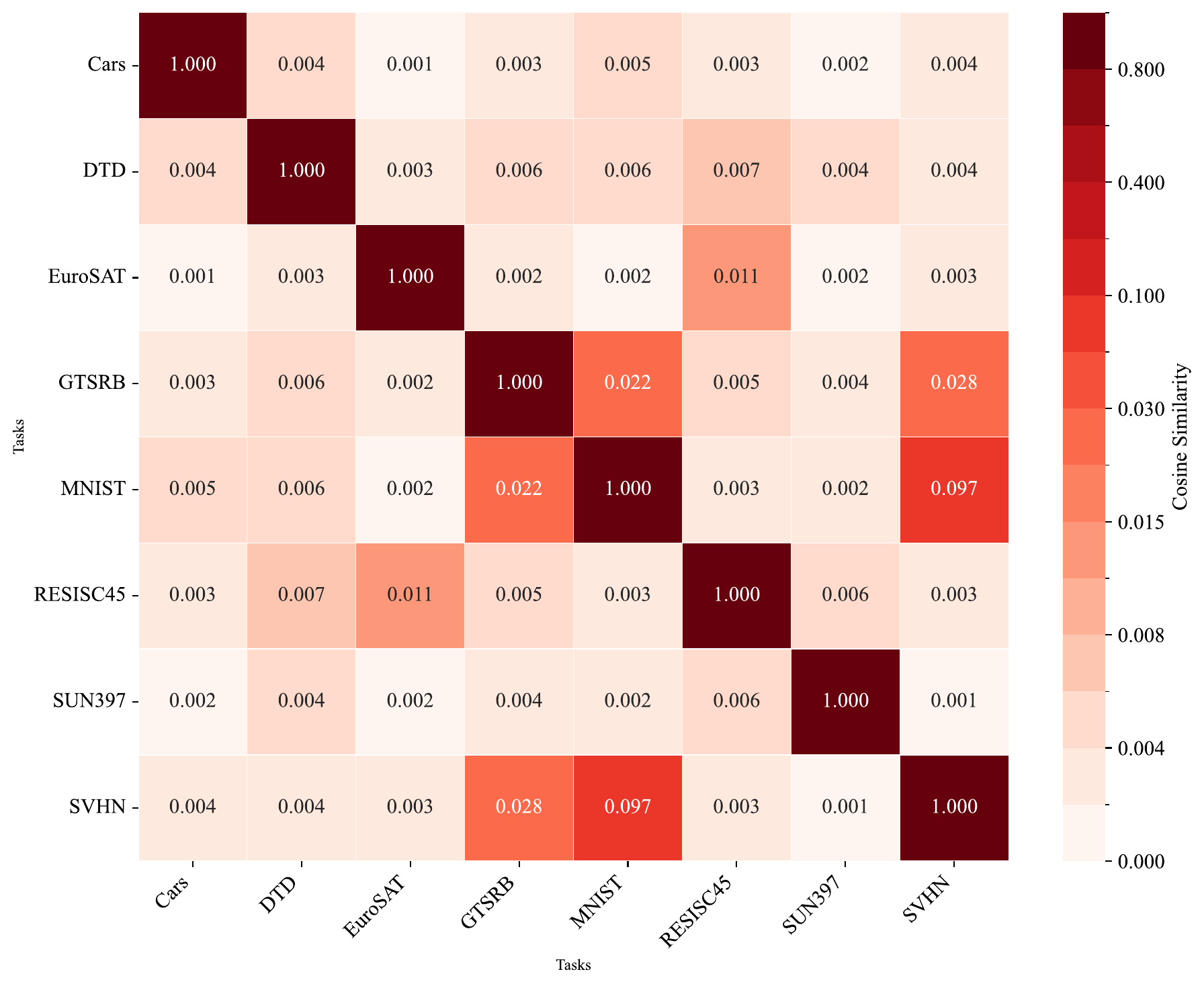} 
        \caption{LoRA-ATT}
    \end{subfigure}

    \vspace{0.3cm} 

    \begin{subfigure}{0.24\textwidth}
        \centering
        \includegraphics[width=\linewidth]{img/hotmap_pdf/VIT-B-16/standard_ortho.pdf} 
        \caption{Non-lin. FT+OrthoReg}
    \end{subfigure}
    \hfill
    \begin{subfigure}{0.24\textwidth}
        \centering
        \includegraphics[width=\linewidth]{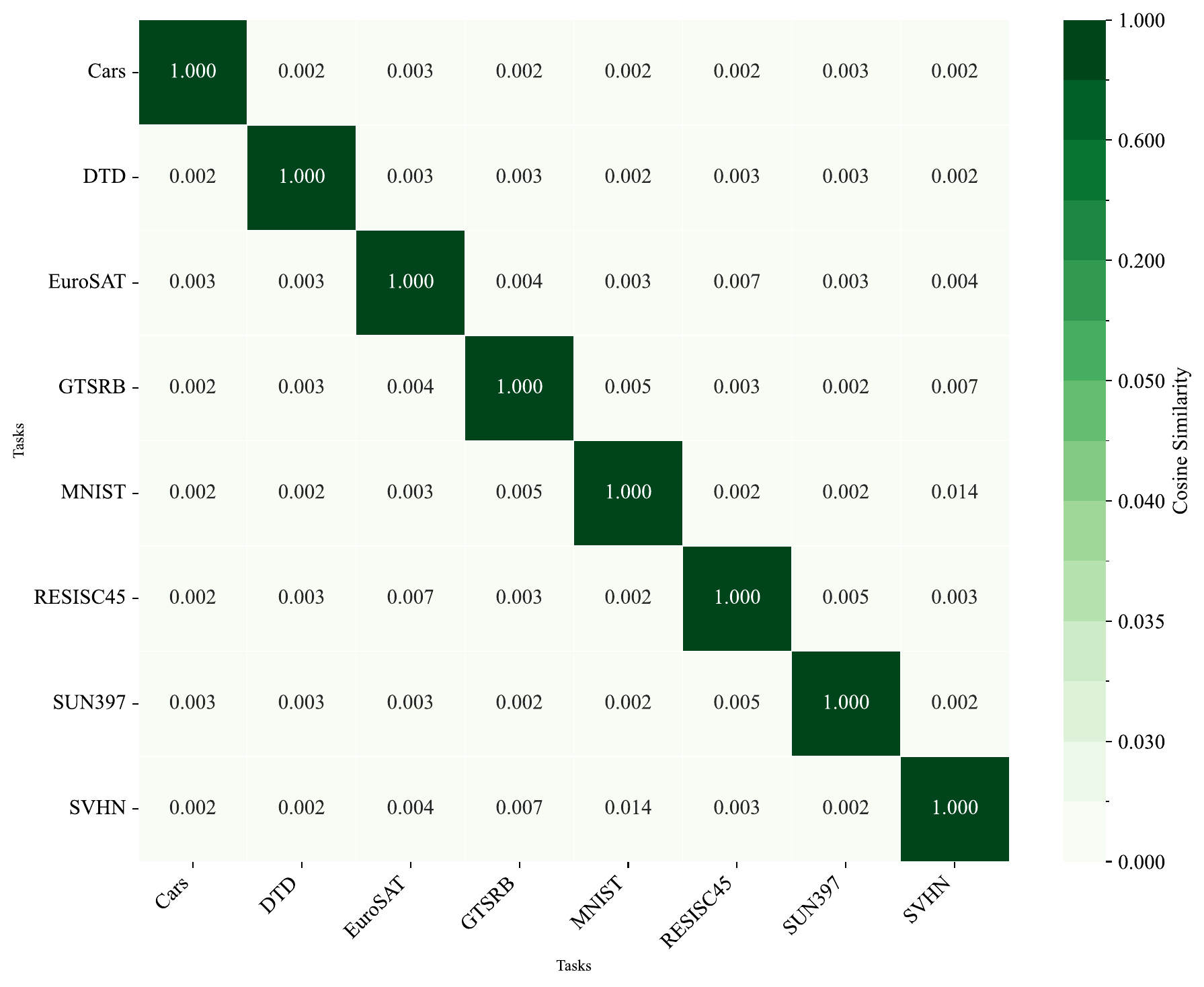} %
        \caption{TTA+OrthoReg}
    \end{subfigure}
    \hfill
    \begin{subfigure}{0.24\textwidth}
        \centering
        \includegraphics[width=\linewidth]{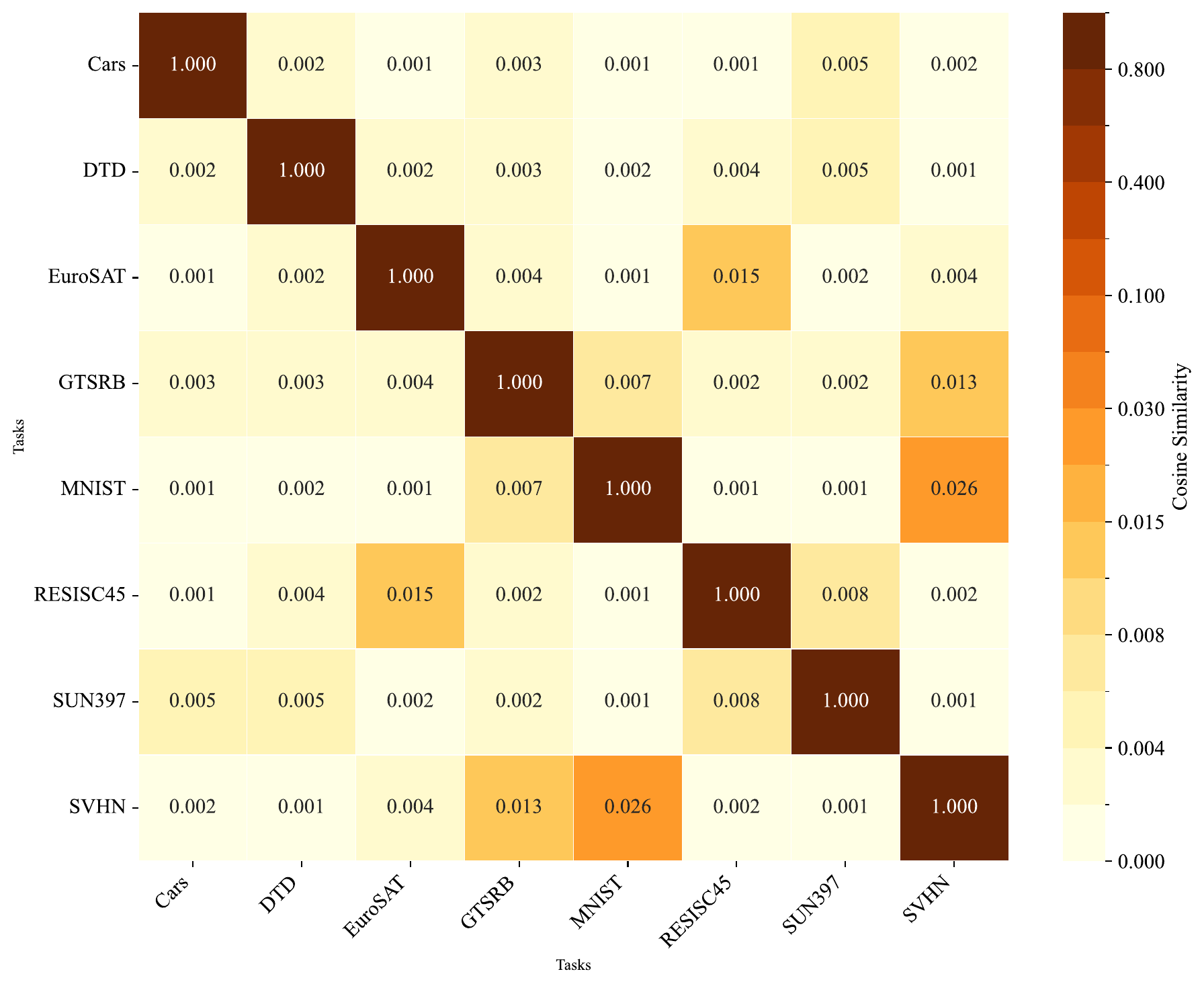}
        \caption{ATT-FT+OrthoReg}
    \end{subfigure}
    \hfill
    \begin{subfigure}{0.24\textwidth}
        \centering
        \includegraphics[width=\linewidth]{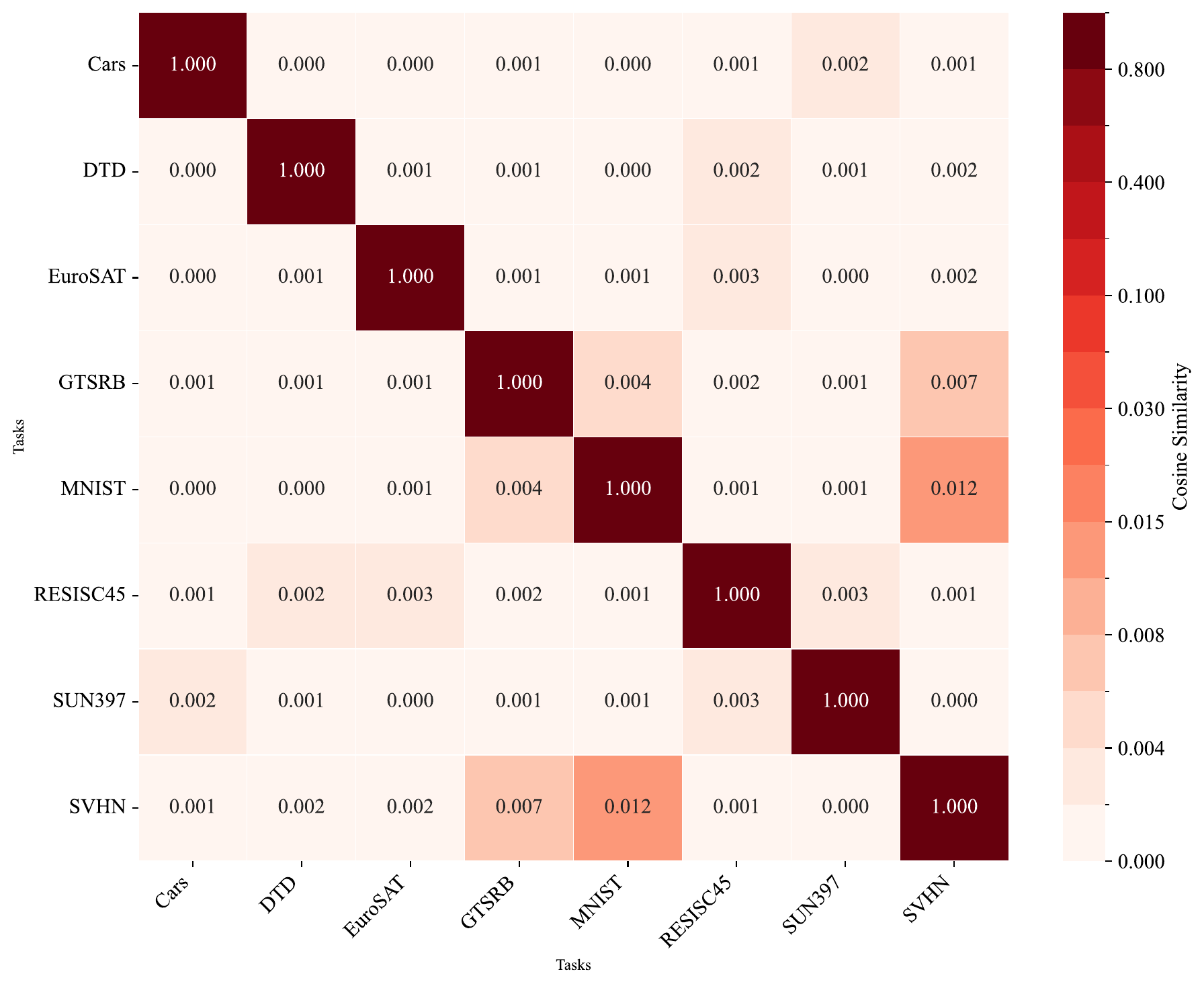}
        \caption{LoRA-ATT+OrthoReg}
    \end{subfigure}

    \caption{
        Pairwise cosine similarity heatmaps of task vectors for ViT-B-16 across different methods. The top row shows the baseline methods, where significant off-diagonal correlation (brighter colors) is visible.
        The bottom row shows the same methods with our OrthoReg regularizer.
        The consistently darker off-diagonal values in the bottom row provide strong empirical validation that OrthoReg successfully produces more orthogonal task vectors, mitigating a key source of task interference.
    }
    \label{fig:heatmaps_vitb16}
\end{figure*}

\begin{figure*}[htbp] 
    \centering 
    \begin{subfigure}{0.24\textwidth}
        \centering
        \includegraphics[width=\linewidth]{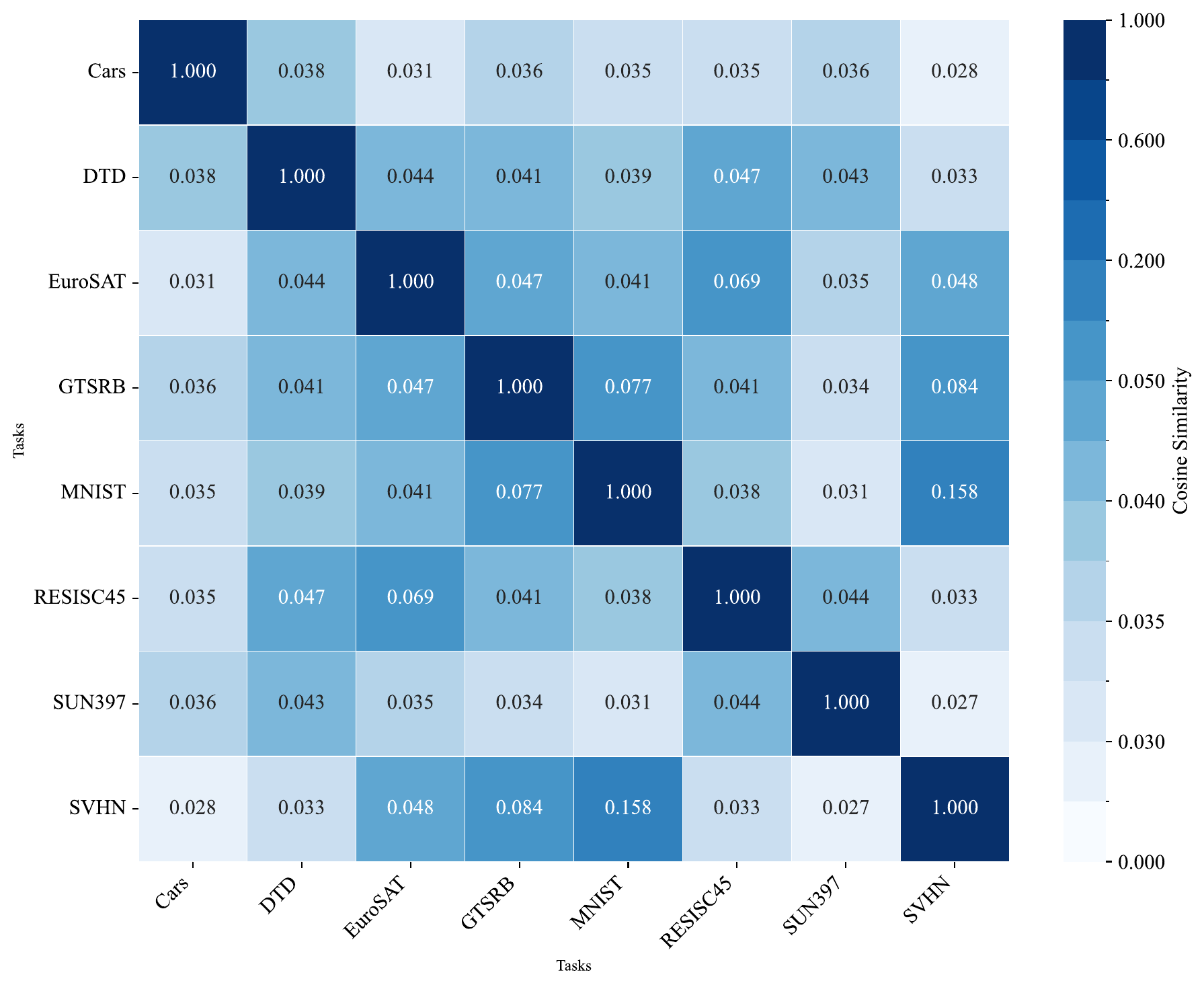}
        \caption{Non-lin. FT}
    \end{subfigure}
    \hfill 
    \begin{subfigure}{0.24\textwidth}
        \centering
        \includegraphics[width=\linewidth]{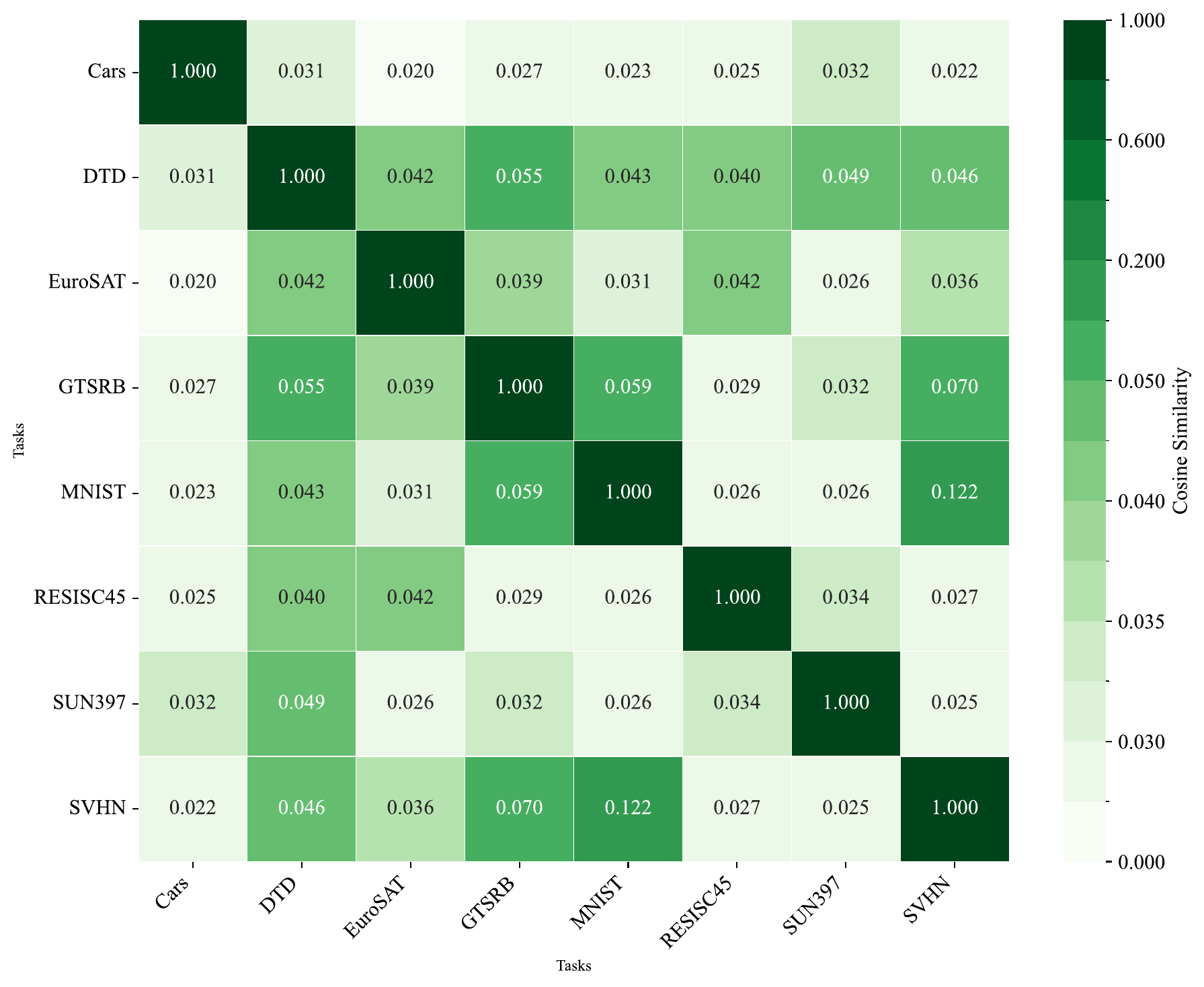}
        \caption{TTA}
    \end{subfigure}
    \hfill
    \begin{subfigure}{0.24\textwidth}
        \centering
        \includegraphics[width=\linewidth]{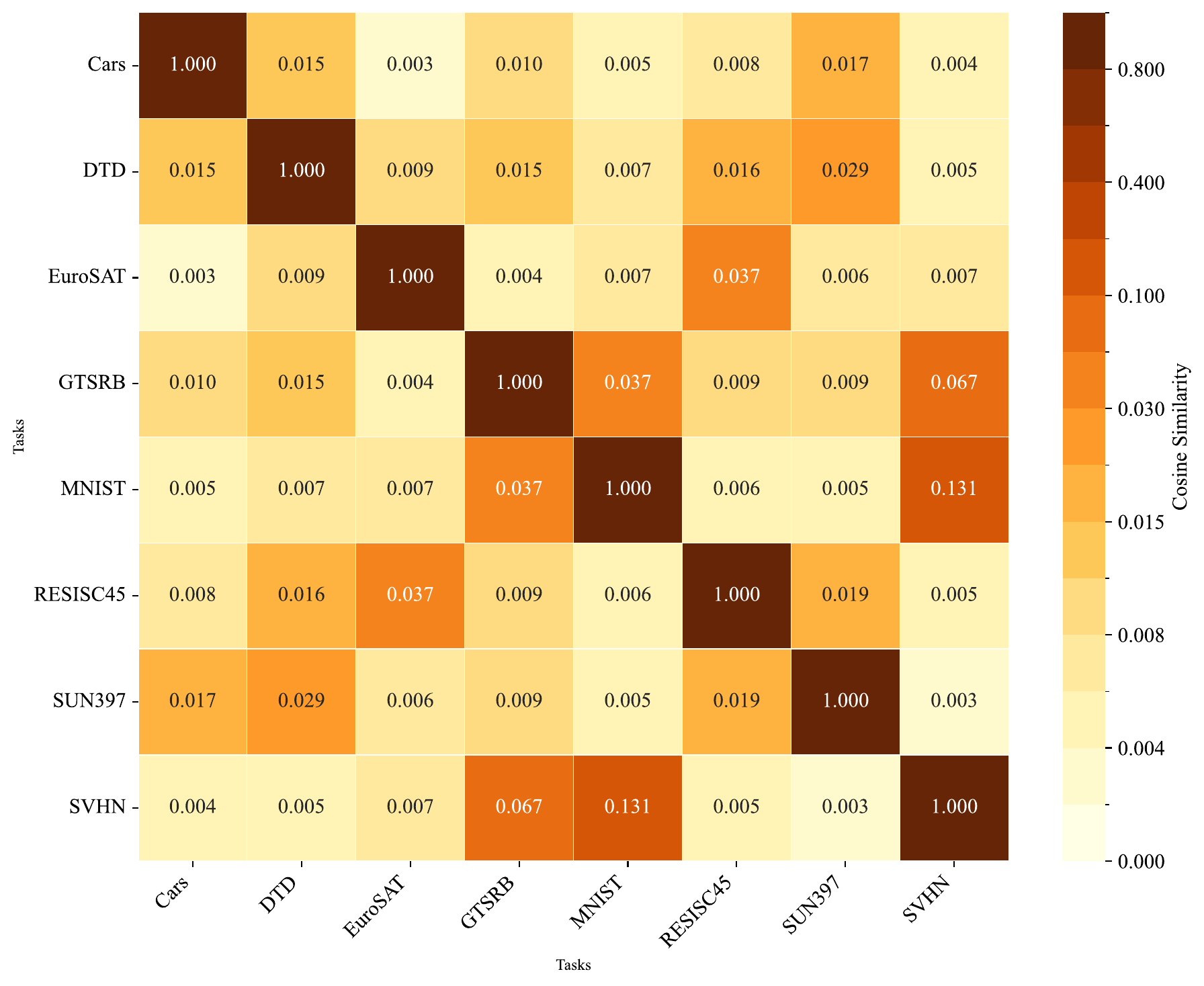} 
        \caption{ATT-FT}
    \end{subfigure}
    \hfill
    \begin{subfigure}{0.24\textwidth}
        \centering
        \includegraphics[width=\linewidth]{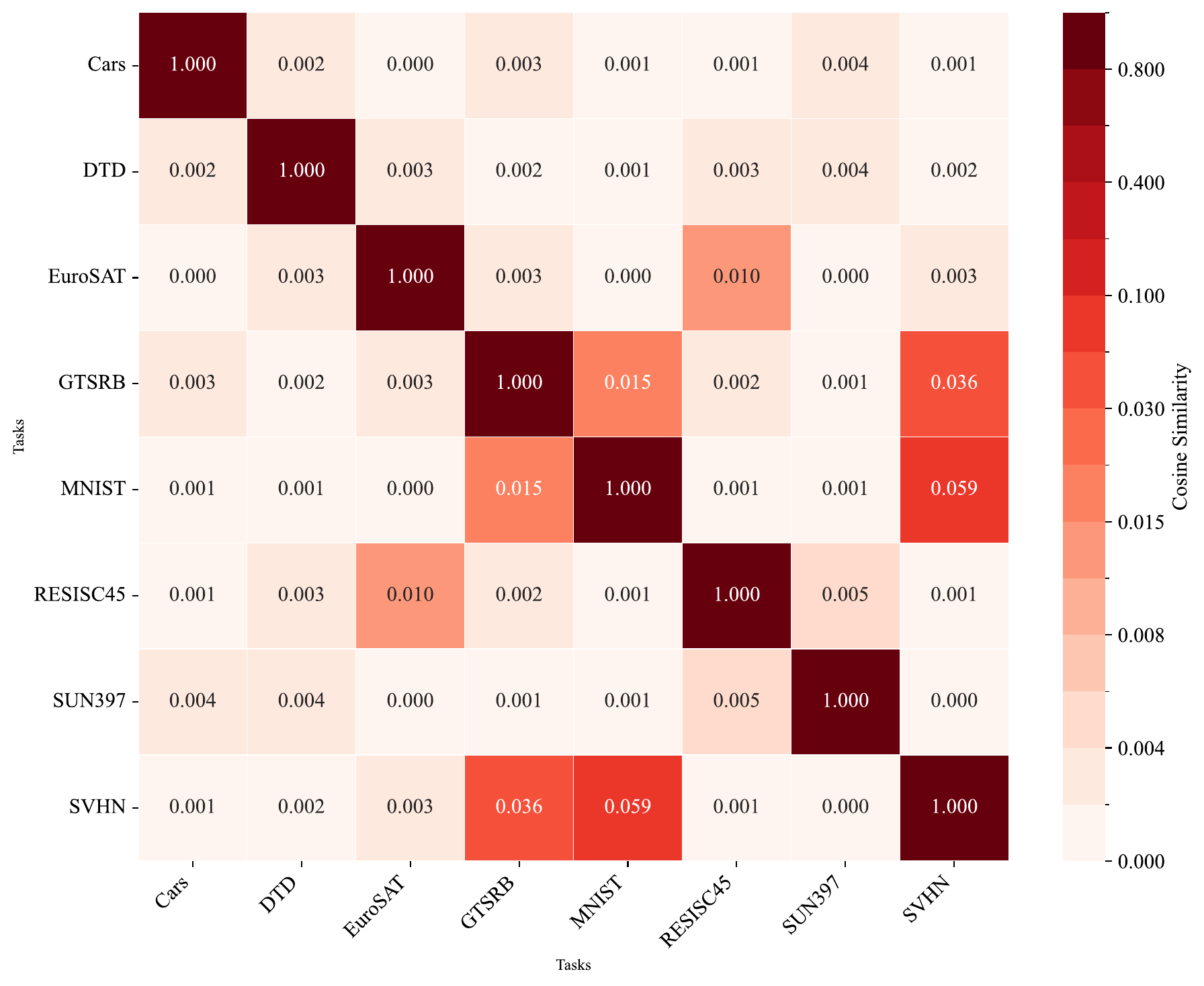} 
        \caption{LoRA-ATT}
    \end{subfigure}

    \vspace{0.3cm} 

    \begin{subfigure}{0.24\textwidth}
        \centering
        \includegraphics[width=\linewidth]{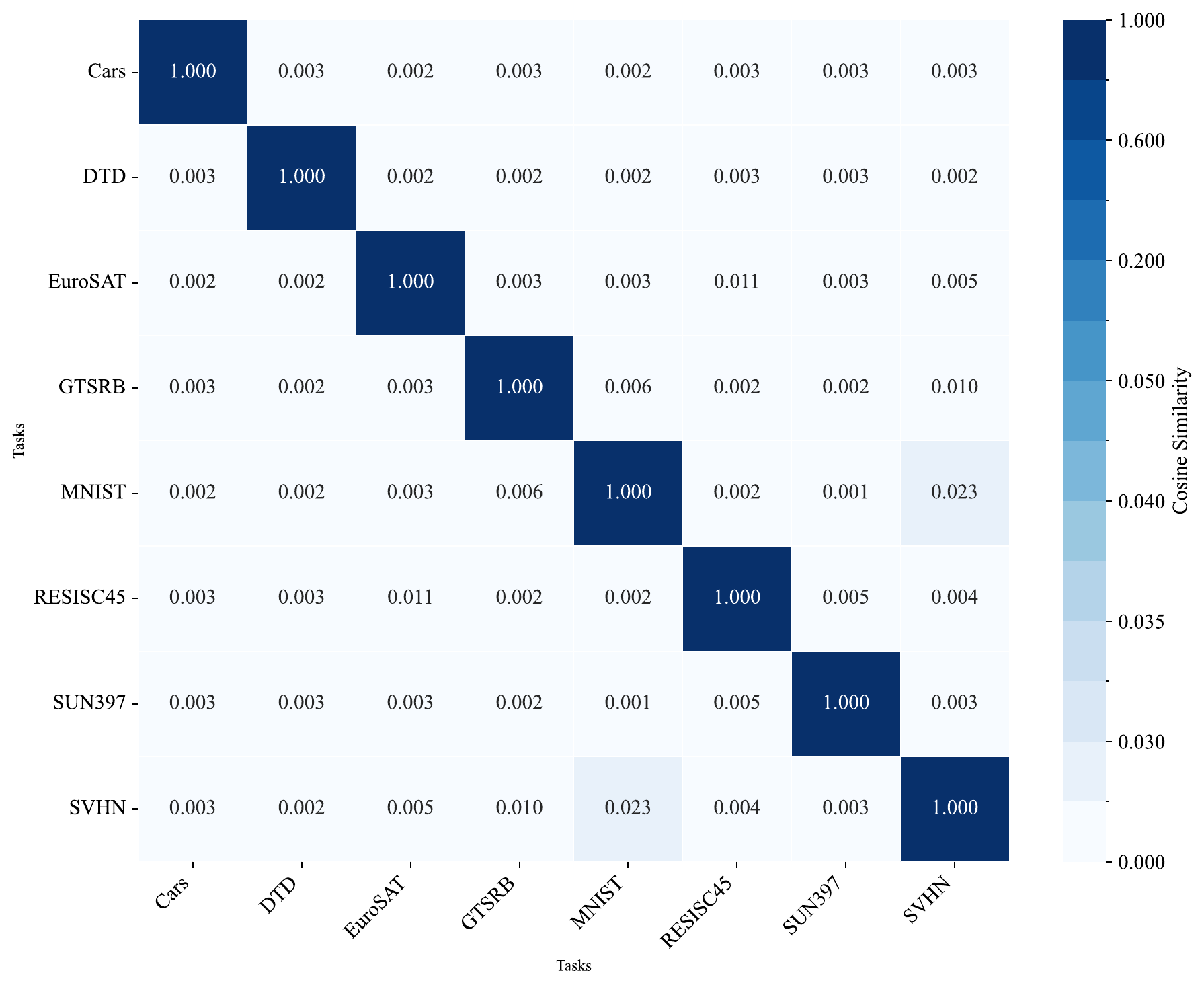} 
        \caption{Non-lin. FT+OrthoReg}
    \end{subfigure}
    \hfill
    \begin{subfigure}{0.24\textwidth}
        \centering
        \includegraphics[width=\linewidth]{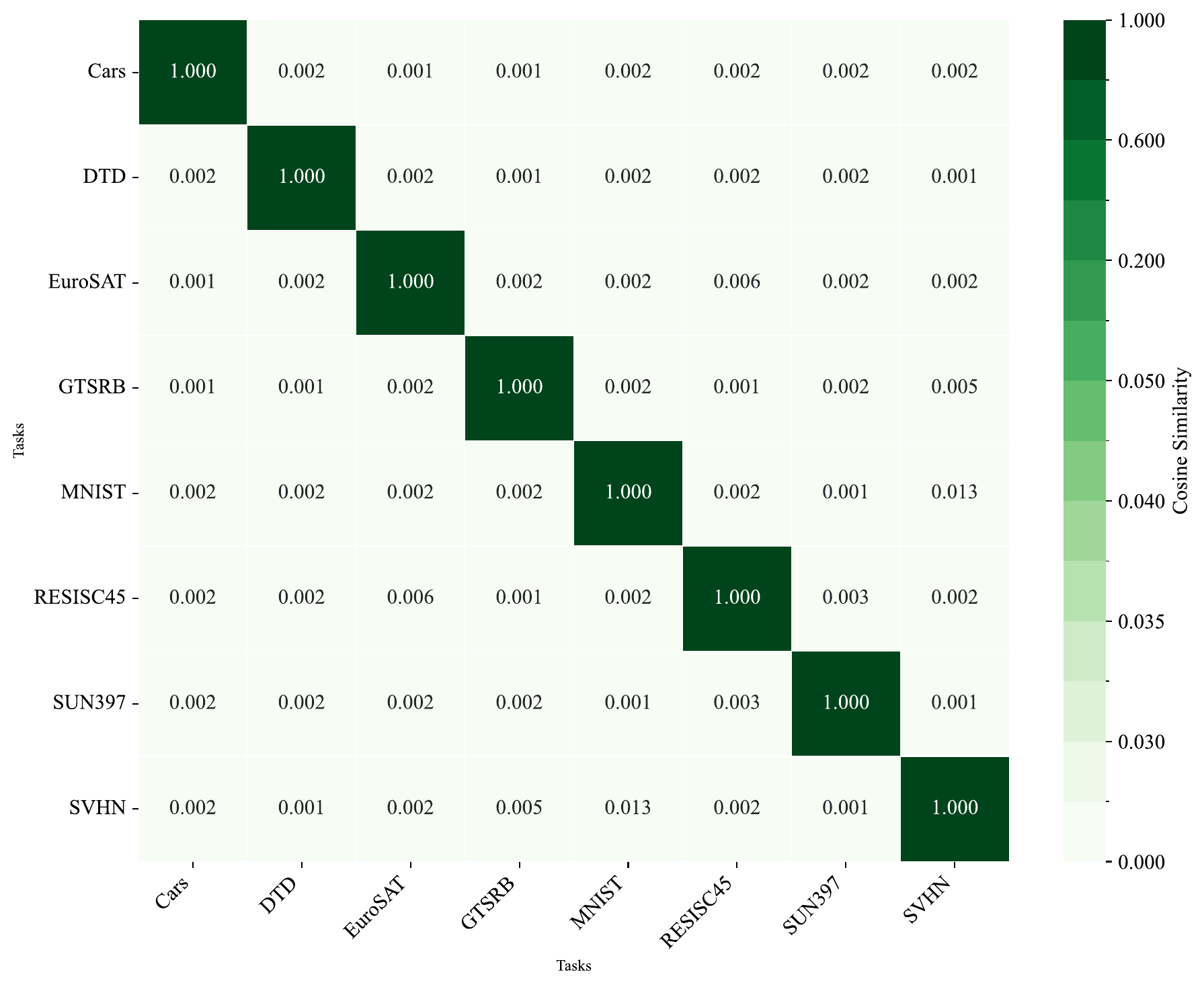} %
        \caption{TTA+OrthoReg}
    \end{subfigure}
    \hfill
    \begin{subfigure}{0.24\textwidth}
        \centering
        \includegraphics[width=\linewidth]{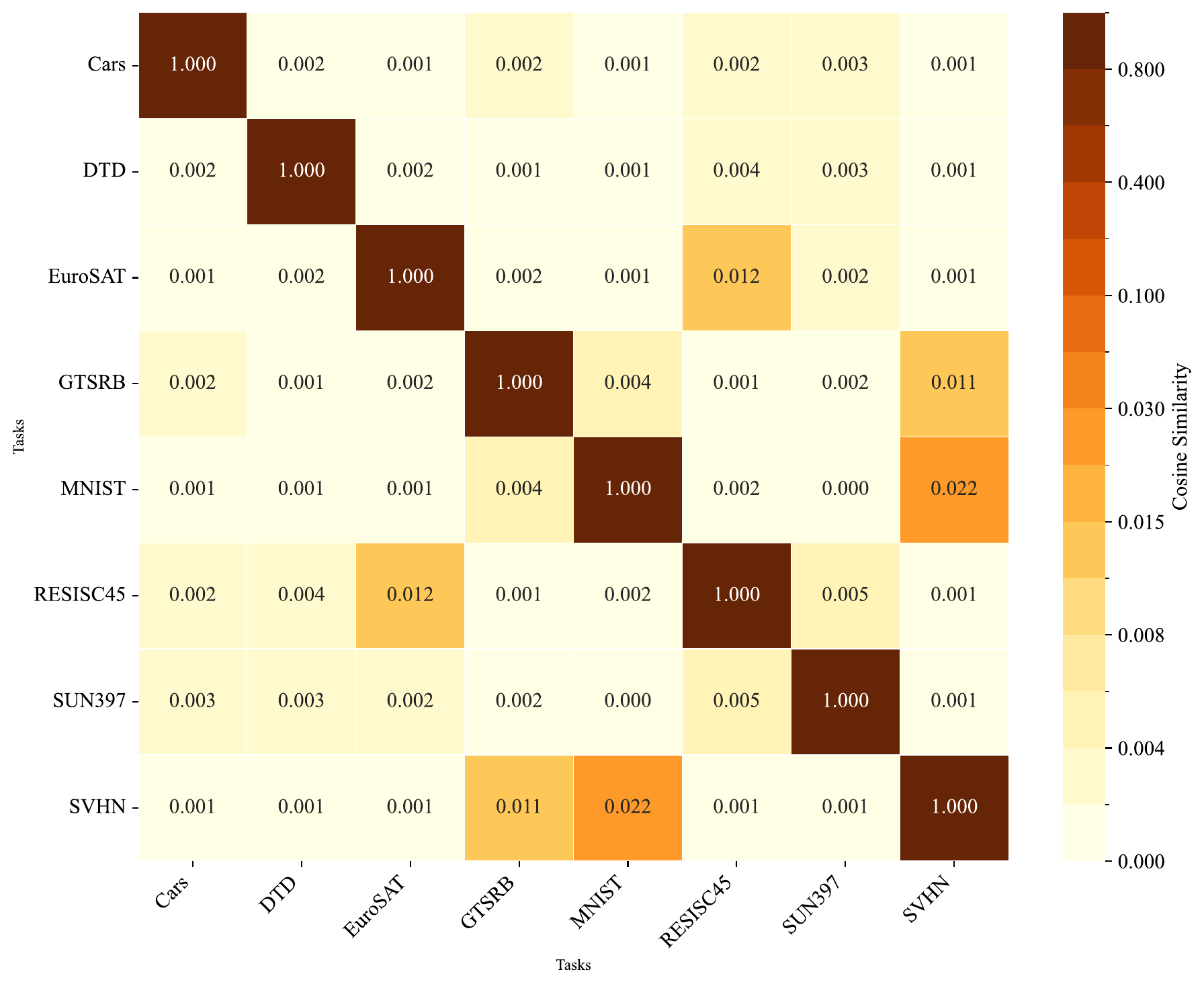}
        \caption{ATT-FT+OrthoReg}
    \end{subfigure}
    \hfill
    \begin{subfigure}{0.24\textwidth}
        \centering
        \includegraphics[width=\linewidth]{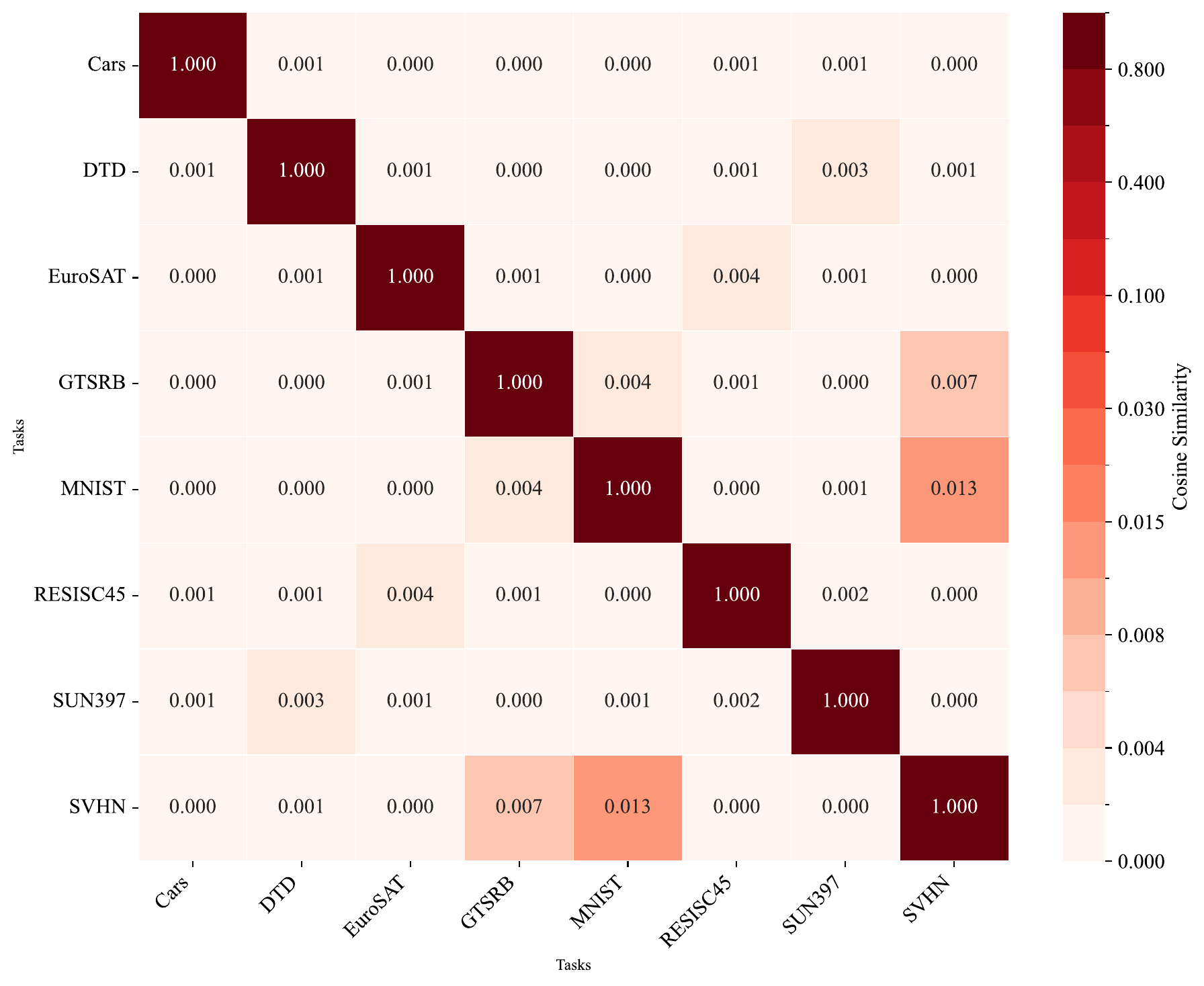}
        \caption{LoRA-ATT+OrthoReg}
    \end{subfigure}

    \caption{
        Pairwise cosine similarity heatmaps of task vectors for VIT-B-32 across different methods. The top row shows the baseline methods, where significant off-diagonal correlation (brighter colors) is visible.
        The bottom row shows the same methods with our OrthoReg regularizer.
        The consistently darker off-diagonal values in the bottom row provide strong empirical validation that OrthoReg successfully produces more orthogonal task vectors, mitigating a key source of task interference.
    }
    \label{fig:heatmaps_vitb32}
\end{figure*}

\begin{figure*}[htbp] 
    \centering 
    \begin{subfigure}{0.24\textwidth}
        \centering
        \includegraphics[width=\linewidth]{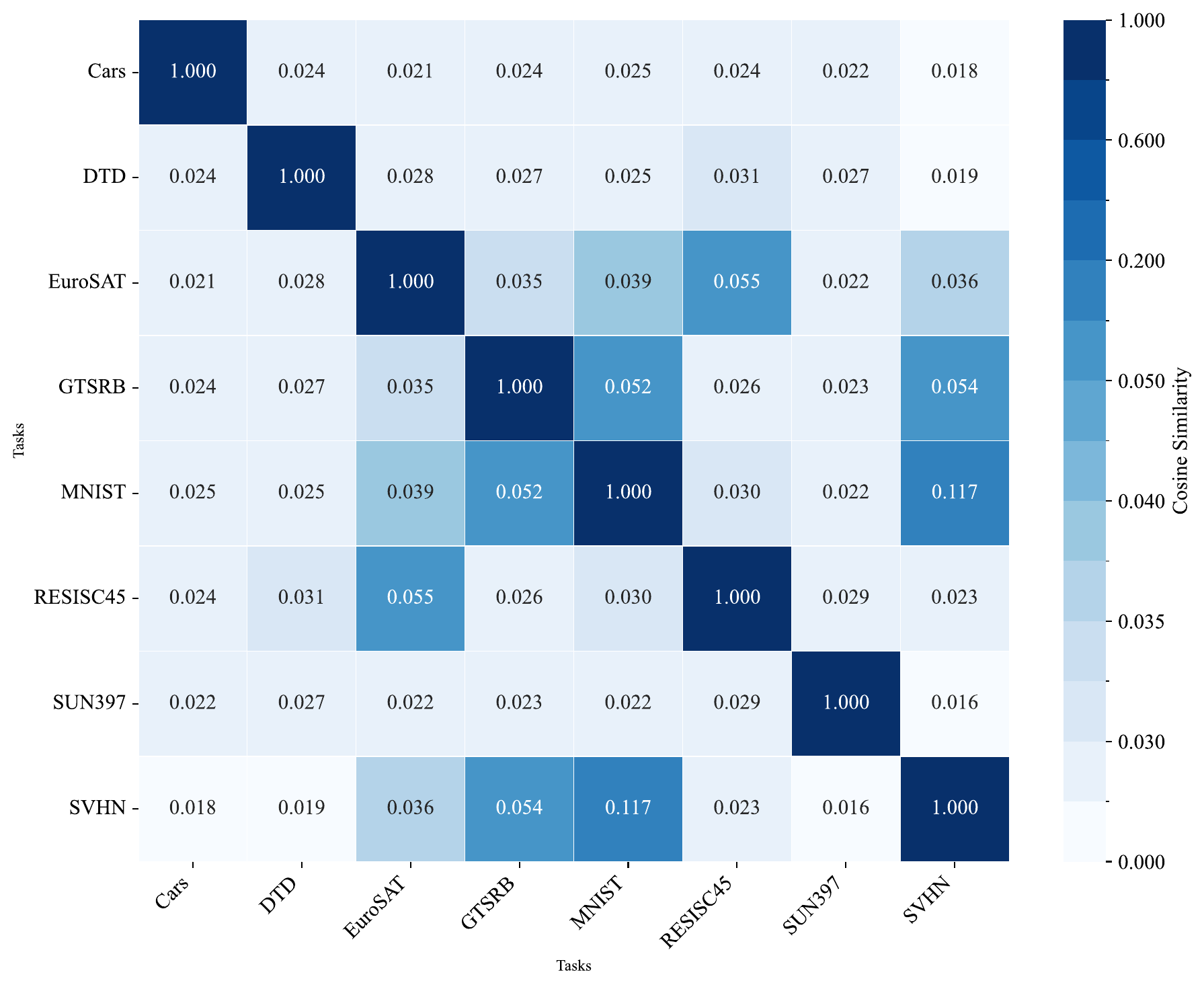}
        \caption{Non-lin. FT}
    \end{subfigure}
    \hfill 
    \begin{subfigure}{0.24\textwidth}
        \centering
        \includegraphics[width=\linewidth]{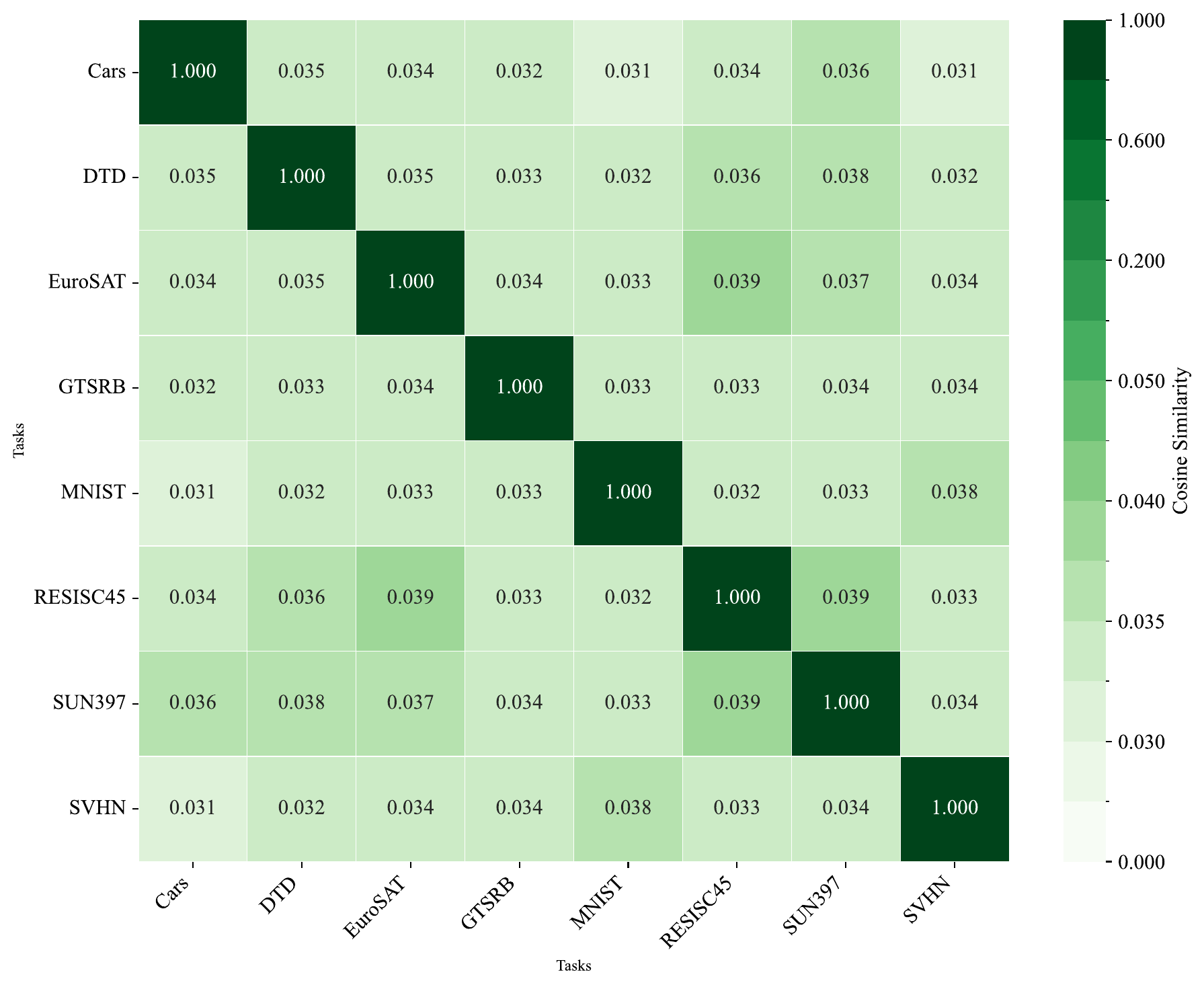}
        \caption{TTA}
    \end{subfigure}
    \hfill
    \begin{subfigure}{0.24\textwidth}
        \centering
        \includegraphics[width=\linewidth]{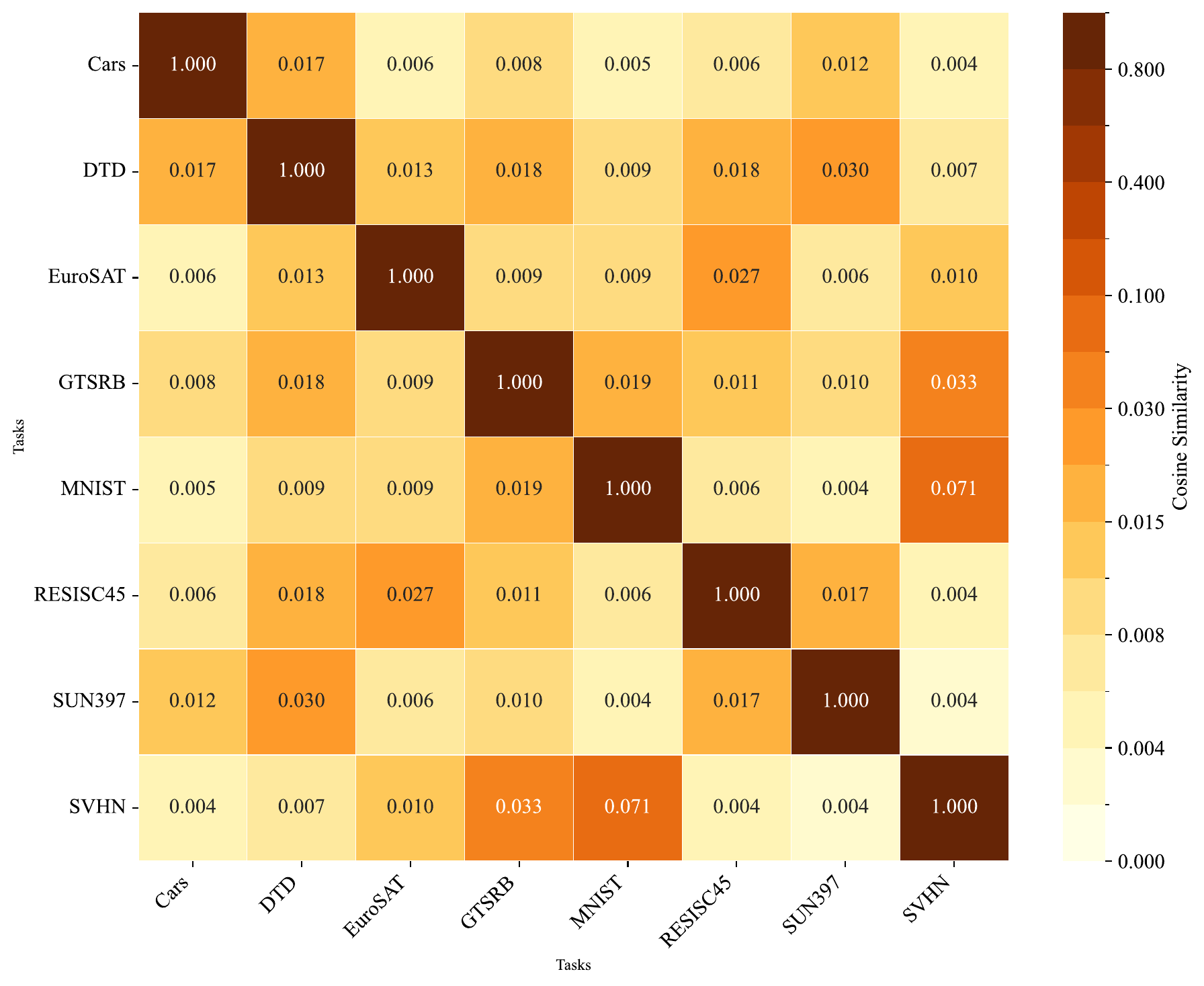} 
        \caption{ATT-FT}
    \end{subfigure}
    \hfill
    \begin{subfigure}{0.24\textwidth}
        \centering
        \includegraphics[width=\linewidth]{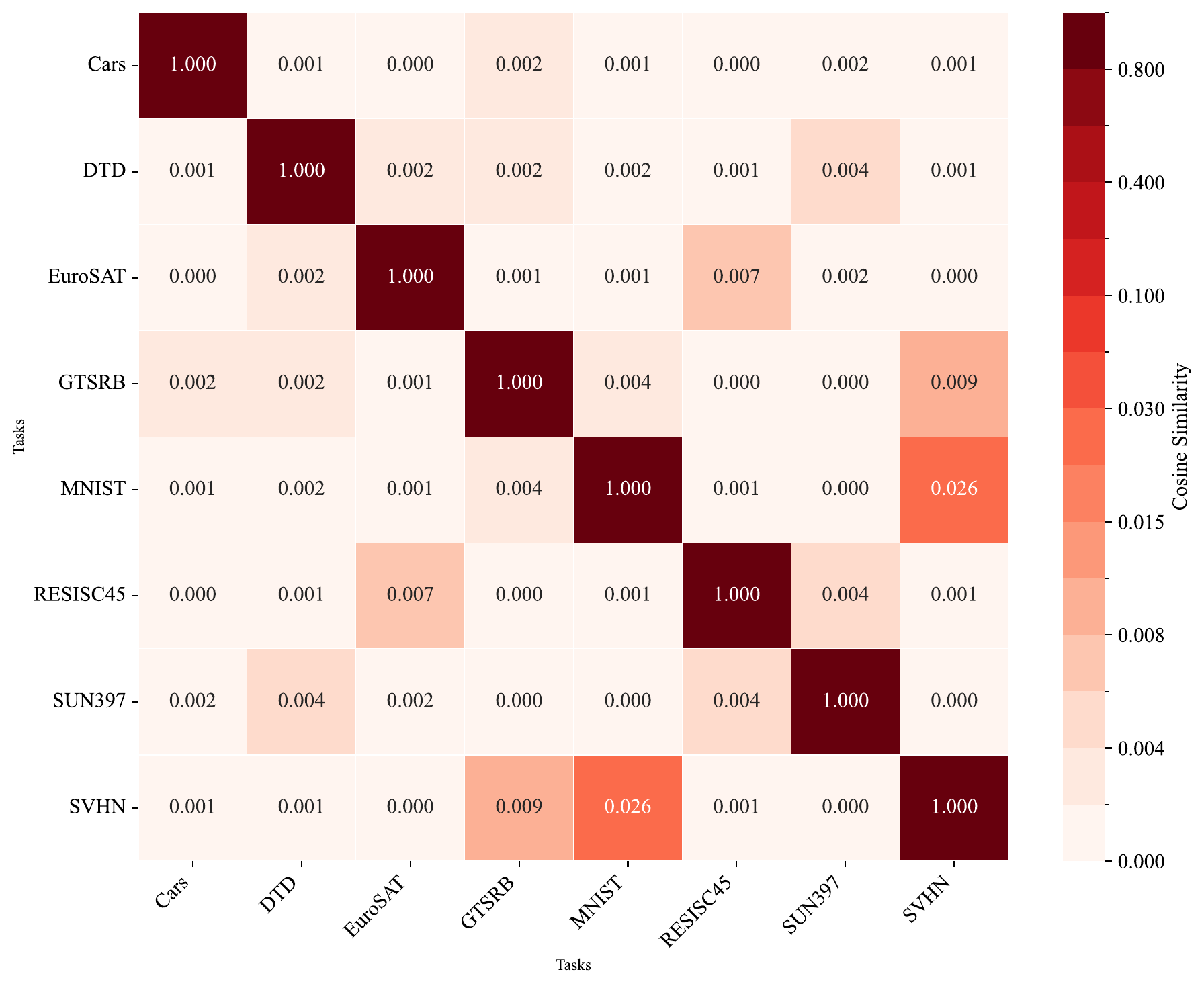} 
        \caption{LoRA-ATT}
    \end{subfigure}

    \vspace{0.3cm} 

    \begin{subfigure}{0.24\textwidth}
        \centering
        \includegraphics[width=\linewidth]{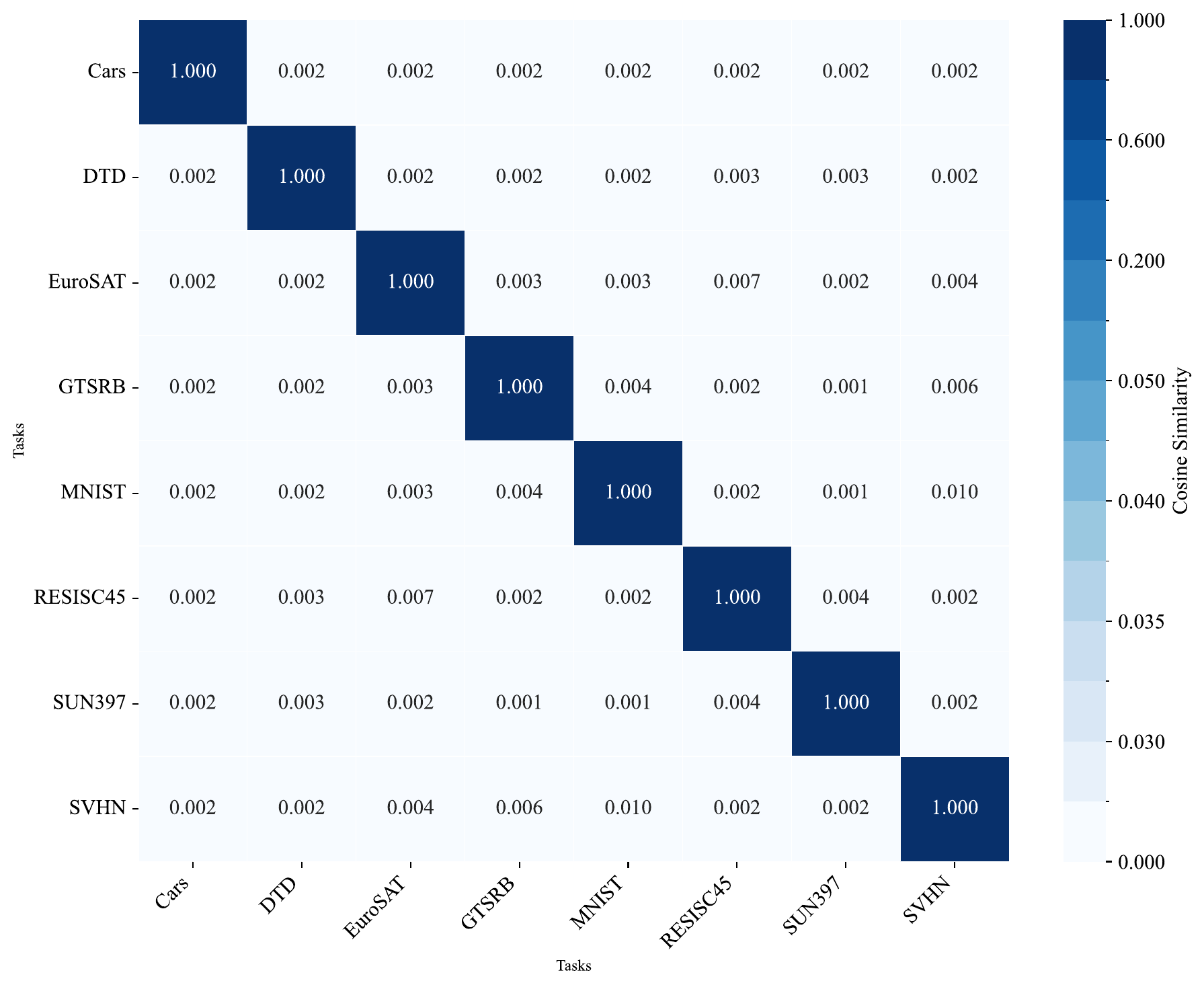} 
        \caption{Non-lin. FT+OrthoReg}
    \end{subfigure}
    \hfill
    \begin{subfigure}{0.24\textwidth}
        \centering
        \includegraphics[width=\linewidth]{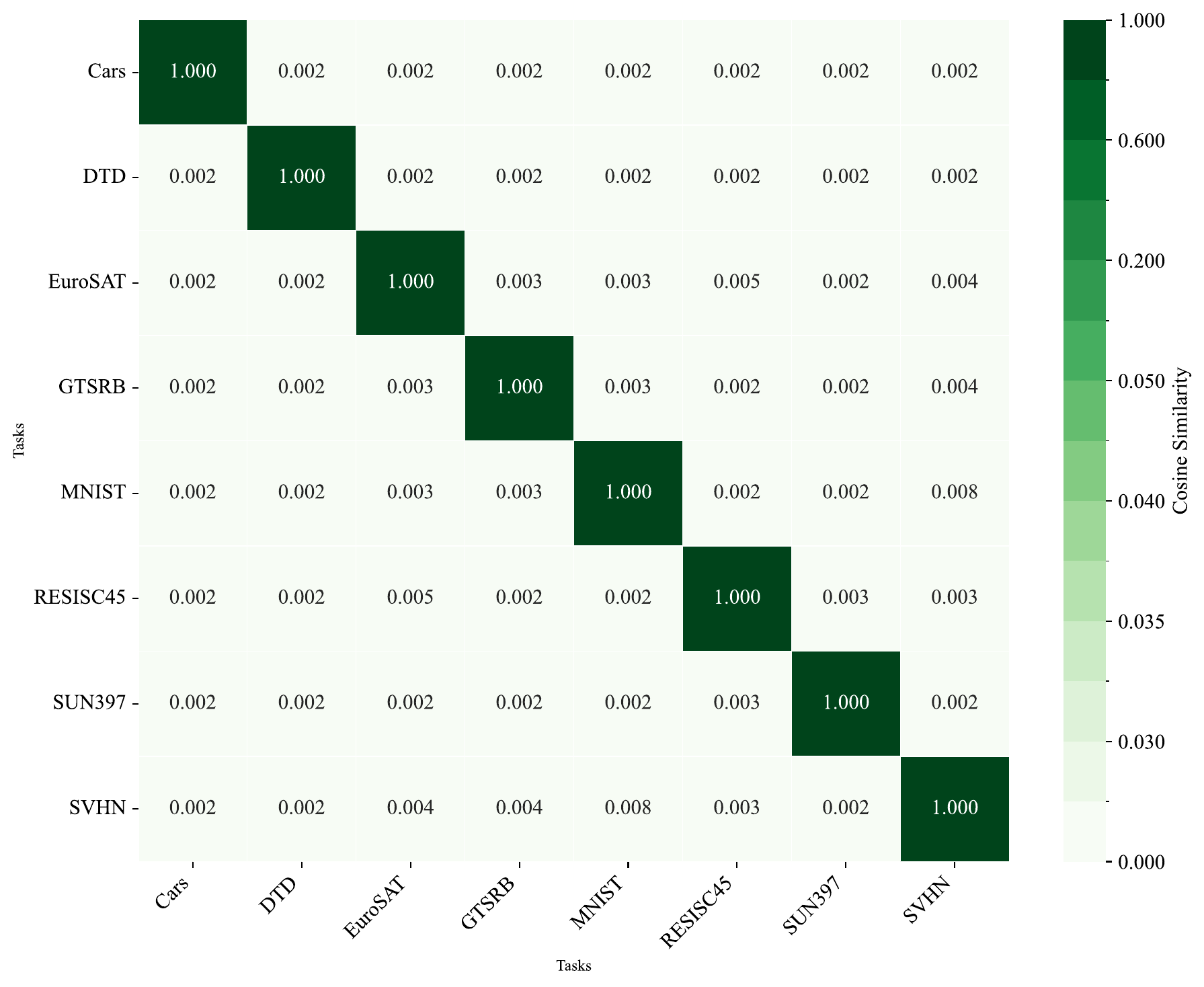} %
        \caption{TTA+OrthoReg}
    \end{subfigure}
    \hfill
    \begin{subfigure}{0.24\textwidth}
        \centering
        \includegraphics[width=\linewidth]{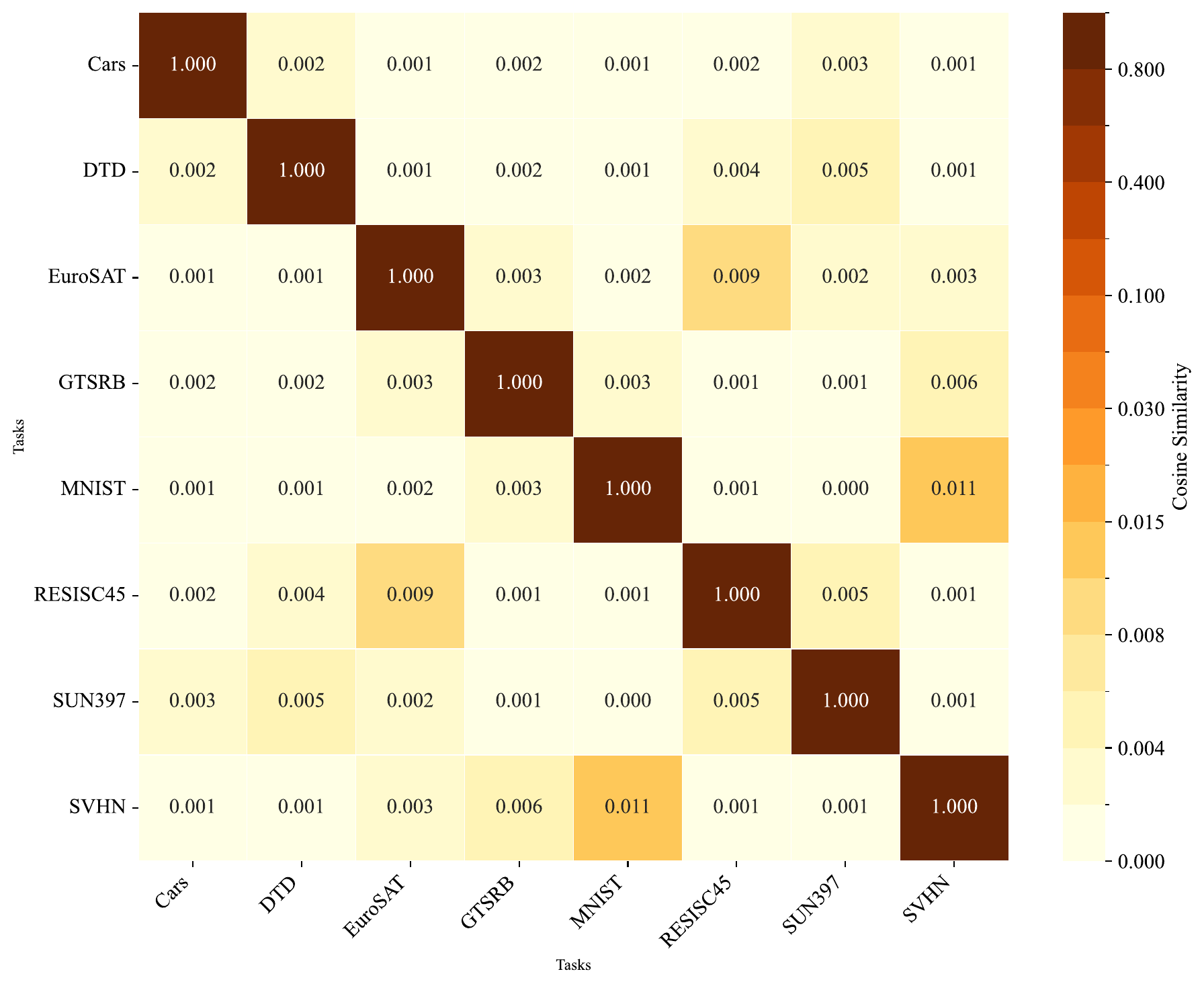}
        \caption{ATT-FT+OrthoReg}
    \end{subfigure}
    \hfill
    \begin{subfigure}{0.24\textwidth}
        \centering
        \includegraphics[width=\linewidth]{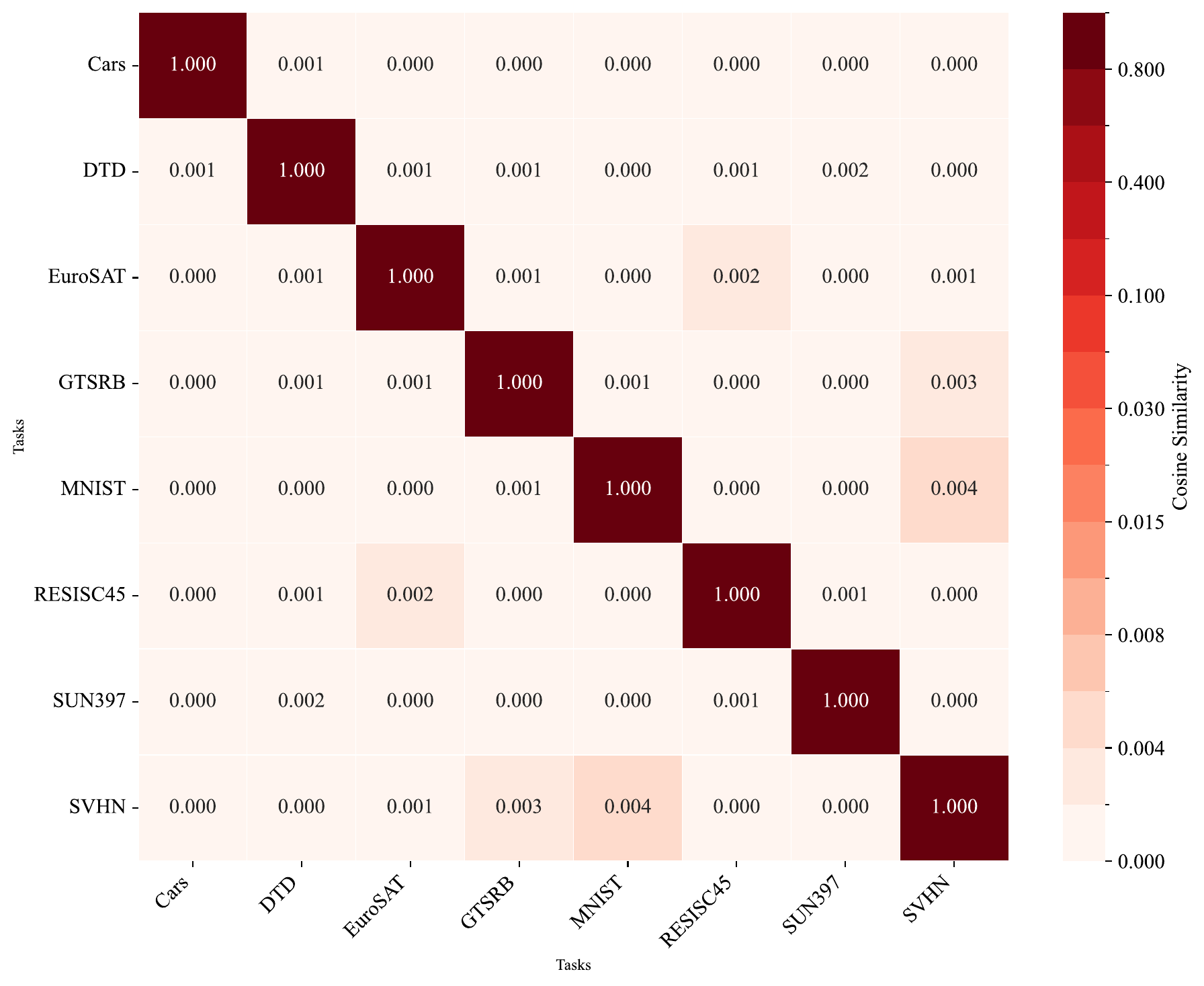}
        \caption{LoRA-ATT+OrthoReg}
    \end{subfigure}

    \caption{
        Pairwise cosine similarity heatmaps of task vectors for ViT-L-14 across different methods. The top row shows the baseline methods, where significant off-diagonal correlation (brighter colors) is visible.
        The bottom row shows the same methods with our OrthoReg regularizer.
        The consistently darker off-diagonal values in the bottom row provide strong empirical validation that OrthoReg successfully produces more orthogonal task vectors, mitigating a key source of task interference.
    }
    \label{fig:heatmaps_vitl14}
\end{figure*}
\section{Experiments Details}\label{sec:experiments_details}

The Normalized Accuracy (Norm.Acc.) metric evaluates the performance of the merged multi-task model ($\theta_{MT}$) relative to individually fine-tuned single-task models ($\theta_t^*$). It is defined as the average of the performance ratios across all $T$ tasks. A score of 100\% indicates that the merged model performs, on average, on par with the individual specialist models, suggesting a successful composition with minimal negative interference.

The formula is given by,
\begin{equation}
    \text{Norm.Acc.} = \left( \frac{1}{T} \sum_{t=1}^{T} \frac{\text{acc}(\theta_{MT}, \mathcal{D}_t)}{\text{acc}(\theta_t^*, \mathcal{D}_t)} \right) \times 100\%,
\end{equation}
where $T$ is the total number of tasks being merged, $\text{acc}(\theta_{MT}, \mathcal{D}_t)$ is the accuracy of the merged model on test set for task $t$ and $\text{acc}(\theta_t^*, \mathcal{D}_t)$ is the accuracy of the model fine-tuned only on task $t$, evaluated on its own test set.

This definition is consistent with the evaluation protocol established in prior work~\cite{ta2023,tta2023,linearatt2025}.

\section{More Experimental Results}
\subsection{Detailed Visualization of Orthogonality}\label{subsec:app_weight_angles}
To provide comprehensive empirical support for the claim made in Section~4.2.3, this part presents a detailed visualization of the weight vector angle distributions for all linear layers within the pre-trained CLIP ViT-B/16 model. \Cref{fig:angle_all} displays the histograms for each weight matrix.

\begin{table*}[t]
\centering
\small
\setlength{\tabcolsep}{8pt}
\caption{Performance comparison of different LoRA module configurations with and without orthogonality regularization. The last row under each module shows the improvement ($\Delta$) from OrthoReg.}
\label{tab:lora_module_ablation}
\begin{tabular}{@{} l l  cc  cc  cc @{}}
\toprule
\multirow{2}{*}{\textbf{LoRA Modules}} & {\textbf{Finetuning}} &
\multicolumn{2}{c}{\textbf{ViT-B-32, 8 tasks}} &
\multicolumn{2}{c}{\textbf{ViT-B-16, 8 tasks}} &
\multicolumn{2}{c}{\textbf{ViT-L-14, 8 tasks}} \\
\cmidrule(lr){3-4} \cmidrule(lr){5-6} \cmidrule(lr){7-8}
 & \textbf{Mode} & Abs.Acc.(↑) & Norm.Acc.(↑) & Abs.Acc.(↑) & Norm.Acc.(↑) & Abs.Acc.(↑) & Norm.Acc.(↑) \\
\midrule

\multirow{3}{*}{qkvofp (All)} 
& LoRA      & 73.03 & 81.89 & 75.18 & 81.83 & 85.44 & 90.98 \\
& +OrthoReg & 74.71 & 84.31 & 78.07 & 85.23 & 87.69 & 93.67 \\
& $\Delta$  & \textbf{+1.68} & \textbf{+2.42} & \textbf{+2.89} & \textbf{+3.40} & \textbf{+2.25} & \textbf{+2.69} \\
\midrule

\multirow{3}{*}{qkvo-- (Attn All)} 
& LoRA      & 73.95 & 84.19 & 76.31 & 84.04 & 87.13 & 93.49 \\
& +OrthoReg & 76.20 & 86.55 & 80.48 & 91.97 & 89.14 & 95.49 \\
& $\Delta$  & \textbf{+2.25} & \textbf{+2.36} & \textbf{+4.17} & \textbf{+7.93} & \textbf{+2.01} & \textbf{+2.00} \\
\midrule

\multirow{3}{*}{qkv--- (Q,K,V)} 
& LoRA      & 70.14 & 80.98 & 74.69 & 82.82 & 85.03 & 91.67 \\
& +OrthoReg & 73.68 & 84.40 & 78.10 & 86.27 & 87.56 & 93.97 \\
& $\Delta$  & \textbf{+3.54} & \textbf{+3.42} & \textbf{+3.41} & \textbf{+3.45} & \textbf{+2.53} & \textbf{+2.30} \\
\midrule

\multirow{3}{*}{q-v--- (Q,V only)} 
& LoRA      & 69.25 & 80.30 & 75.15 & 83.35 & 84.39 & 91.11 \\
& +OrthoReg & 72.71 & 83.77 & 77.03 & 85.37 & 86.58 & 93.29 \\
& $\Delta$  & \textbf{+3.46} & \textbf{+3.47} & \textbf{+1.88} & \textbf{+2.02} & \textbf{+2.19} & \textbf{+2.18} \\
\midrule

\multirow{3}{*}{----fp (MLP only)} 
& LoRA      & 69.19 & 78.01 & 71.24 & 78.02 & 81.98 & 87.78 \\
& +OrthoReg & 68.92 & 77.77 & 72.05 & 78.72 & 82.80 & 88.13 \\
& $\Delta$  & \textbf{-0.27} & \textbf{-0.24} & \textbf{+0.81} & \textbf{+0.70} & \textbf{+0.82} & \textbf{+0.35} \\
\bottomrule
\end{tabular}
\end{table*}

As illustrated in \Cref{fig:angle_all}, a clear and consistent pattern emerges across the model's layers. We observe two distinct behaviors. (1) Embedding Layers. The first two subplots correspond to the patch\_embedding and pos\_embedding layers. These layers show broader, more Gaussian-like distributions, which is understandable given their unique function of mapping raw inputs into the initial embedding space. As our analysis primarily concerns the transformation dynamics within the main model body, these layers are not the central focus of our study. (2) Transformer Blocks. In stark contrast, nearly all subsequent weight matrices, which constitute the core computational machinery of the model, including the query, key, value (QKV) projections, attention output projections (proj), and MLP layers within all 12 transformer blocks, exhibit angle distributions that are sharply and narrowly peaked at 90 degrees.

This detailed, per-layer visualization provides robust evidence that near-orthogonality is not an isolated occurrence but a pervasive geometric property of the pre-trained model's core processing blocks.

\subsection{Detailed Per-Task Performance Visualization}\label{subsec:details_per-task}

This section supplements the analysis in Section~5.2 by providing the comprehensive per-task performance radar charts for all evaluated architectures: ViT-L-14, ViT-B-16, and ViT-B-32. The results shown in \Cref{fig:radar-all-architectures} reinforce and expand upon the findings presented in the main body. We consistently observe that applying OrthoReg (the blue area) leads to a larger performance footprint compared to the baselines (the red area) across the vast majority of tasks, methods, and architectures. This further corroborates our claim that OrthoReg is a model-agnostic regularizer that effectively mitigates task interference, leading to broad performance gains in multi-task scenarios.

\subsection{Details About Task Negation}\label{subsec:details_task_negation}
In this section, we provide additional details for the task negation experiments discussed in Section ~5.3. When the accuracy requirement on the control task is further relaxed, such as to 90\% (see ~\Cref{tab:task_negation2}) or 80\% (see ~\Cref{tab:task_negation3}), the effect of task negation becomes progressively stronger, resulting in lower accuracy on the target task. Moreover, our OrthoReg regularizer can further enhance the negation effect while still meeting the control-task accuracy threshold. In some cases, it even improves control-task accuracy while reducing target-task accuracy. These results demonstrate that our method effectively disentangles task-specific feature information, substantially reducing undesired interference with non-target tasks during the task negation process.

\subsection{Visualization of Task Vector Similarity}\label{subsec:details_taskvector_sim}
To supplement the analysis in Section~5.4, this section provides additional task vector similarity heatmaps. These figures (\Cref{fig:heatmaps_vitb16}, \Cref{fig:heatmaps_vitb32}, \Cref{fig:heatmaps_vitl14}) illustrate the effect of OrthoReg across different baseline methods and model architectures, consistently demonstrating that our method produces more orthogonal task vectors.

\subsection{Detailed Ablation Study on LoRA Components}\label{subsec:details_lora}
This part provides additional details and results to supplement the LoRA ablation study presented in Section~5.1.

\subsubsection{Rationale for Module Selection}
The selection of different module subsets for our LoRA-based ablation study was designed to systematically probe the effect of OrthoReg on distinct functional components of the Vision Transformer. 
\begin{itemize}
    \item All Tunable Layers. qkvofp: This represents the most comprehensive PEFT approach, applying LoRA to all available linear layers (attention and MLP). It serves as a baseline to evaluate the effect of tuning the entire model in a parameter-efficient manner.
    \item MLP Layers Only. ---fp: This configuration isolates the FFN or MLP blocks. By tuning only these layers, we can assess their specific contribution to task adaptation and how OrthoReg influences them in isolation.
    \item Attention Subsets. qkvo--, qkv---, and q-v---: These configurations focus on the multi-head self-attention mechanism, which is widely considered crucial for capturing task-specific patterns.
    \begin{itemize}
        \item qkvo-- tunes all four projection matrices (query, key, value, and output), representing a full intervention within the attention block.
        \item qkv--- omits the output projection, allowing us to gauge its importance.
        \item q-v--- is a particularly important configuration. Prior work~\cite{Zanella2024} has identified that fine-tuning only the query and value matrices can be a highly effective and parameter-efficient strategy.
    \end{itemize}
\end{itemize}
By comparing these configurations, we can draw nuanced conclusions about where task-specific knowledge is stored and how promoting orthogonality in different components contributes to the final performance of task arithmetic.

\subsubsection{results}
Table \ref{tab:lora_module_ablation} summarizes the effect of applying OrthoReg across different LoRA module configurations. Overall, OrthoReg consistently improves performance in all settings except the MLP-only configuration. The largest gains appear in attention-related modules , such as qkvo-- , with improvements up to +4.17 points on ViT-B-16. This aligns with the common understanding that attention layers carry most of the task-specific information, and orthogonalizing their updates most effectively reduces feature entanglement.

Full-layer tuning (qkvofp) also benefits substantially from OrthoReg, indicating that larger tunable subspaces allow orthogonality constraints to better isolate task-relevant directions. The Q,V-only configuration (q-v---), previously identified as an efficient tuning strategy, also shows stable improvements when combined with OrthoReg.

The only exception is the MLP-only setup, where OrthoReg slightly reduces accuracy on smaller models. This suggests that MLP layers contribute less task-specific variation, and enforcing orthogonality may occasionally restrict useful shared representations.

Overall, the results confirm that OrthoReg most strongly enhances the components responsible for task-discriminative behavior, leading to more accurate task vectors and more reliable task arithmetic.

\end{document}